\documentclass{article}
\usepackage{natbib}
\usepackage[utf8]{inputenc}
\usepackage[T1]{fontenc}
\usepackage{amsmath, amssymb, amsthm}
\usepackage{mathtools}
\usepackage{geometry}
\usepackage{graphicx}
\usepackage{xcolor}
\usepackage{float}
\usepackage{enumitem}
\usepackage{hyperref}
\usepackage{cleveref}
\usepackage{tikz}
\usepackage{yhmath}
\usepackage{pgfplots}
\pgfplotsset{compat=1.18}
\theoremstyle{definition}
\newtheorem{definition}{Definition}[section]

\newtheorem{assum}{Assumption}

\theoremstyle{plain}
\newtheorem{theorem}{Theorem}[section]
\newtheorem{lemma}{Lemma}[section]
\newtheorem{corollary}{Corollary}[section]
\newtheorem{proposition}{Proposition}[section]

\usepackage{mathrsfs}

\newcommand{\totgraddr}[1]{\widehat{\nabla\mathcal F}^{\mathrm{OBi}}(#1)}

\newcommand{\dd}{\mathrm{d}}
\newcommand{\LL}{L^2(P_X)}

\newcommand{\hstar}{h^\star_{\omega,P}}
\newcommand{\jstar}{j^\star_{\omega,P}}

\DeclareMathOperator*{\argmin}{\arg\!\min}

\newcommand{\R}{\mathbb R}
\newcommand{\E}{\mathbb E}
\newcommand{\cH}{\mathcal H}
\newcommand{\cJ}{\mathcal J}
\newcommand{\cM}{\mathcal M}

\newcommand{\cV}{\mathcal V}
\newcommand{\Ttrain}{\mathscr T_1}
\newcommand{\ghat}{\widehat{\nabla\mathcal F}^{\mathrm{OBi}}}
\newcommand{\mstar}{m^\star_P}
\newcommand{\norm}[1]{\lVert #1\rVert}
\newcommand{\ip}[2]{\langle #1,#2\rangle}

\newcommand{\etr}{\mathcal E_{\mathrm{tr}}}

\newcommand{\step}[1]{\par\smallskip\noindent\textbf{#1}}

\newcommand{\cQ}{\mathcal Q}
\newcommand{\cF}{\mathcal F}
\newcommand{\cE}{\mathcal E}

\theoremstyle{remark}
\newtheorem{remark}[definition]{Remark}

\usepackage{booktabs}
\usepackage{subcaption}

\usepackage{url}

\title{Parametric Rates for Nonparametric Conditional-Moment Bilevel Optimization}

\crefname{assum}{Assumption}{Assumptions}
\Crefname{assum}{Assumption}{Assumptions}

\crefname{lemma}{Lemma}{Lemmas}
\Crefname{lemma}{Lemma}{Lemmas}
\crefname{theorem}{Theorem}{Theorems}
\Crefname{theorem}{Theorem}{Theorems}
\crefname{proposition}{Proposition}{Propositions}
\Crefname{proposition}{Proposition}{Propositions}
\crefname{corollary}{Corollary}{Corollaries}
\Crefname{corollary}{Corollary}{Corollaries}
\crefname{definition}{Definition}{Definitions}
\Crefname{definition}{Definition}{Definitions}
\Crefname{assumption}{Assumption}{Assumptions}

\usepackage{graphicx}
\usepackage{booktabs}
\usepackage{authblk}
\title{Parametric Rates for Nonparametric Conditional-Moment Bilevel Optimization}

\author[1]{Fares El Khoury$^{*}$}
\author[2]{Houssam Zenati$^{*}$}
\author[3,4]{Nathan Kallus}
\author[1]{Michael Arbel}
\author[4]{Aur\'elien Bibaut}

\affil[1]{Universit\'e Grenoble Alpes, Inria, CNRS, Grenoble INP, LJK}
\affil[2]{Gatsby Computational Neuroscience Unit, University College London}
\affil[3]{Cornell University}
\affil[4]{Netflix Research}

\usepackage{xcolor}

\begin{document}
\date{}
\addtocontents{toc}{\protect\setcounter{tocdepth}{0}}

\maketitle
\setcounter{footnote}{0}
\renewcommand{\thefootnote}{\fnsymbol{footnote}}
\footnotetext[1]{Equal contribution. Correspondence to: \href{mailto:fares.el-khoury@inria.fr}{\texttt{fares.el-khoury@inria.fr}} -- \href{mailto:h.zenati@ucl.ac.uk}{\texttt{h.zenati@ucl.ac.uk}}}

\begin{abstract}
We study functional bilevel optimization with conditional moments, where plug-in hypergradients can inherit first-order errors from nonparametrically learned conditional expectations. We introduce OBiGrad, an orthogonalized hypergradient estimator whose bias is a product of nuisance errors. Under our regression and complexity conditions, projected gradient descent with the direct plug-in hypergradient yields expected stationarity error $O(T^{-1/2}+n^{-1/2}+n^{-\nu/b})$ for the best iterate, whereas OBiGrad achieves $O(T^{-1/2}+n^{-1/2}+n^{-2\nu/b})$, where $b=\nu(2+p)+\gamma$ combines the approximation and complexity exponents. Consequently, when $\nu(2-p)\ge\gamma$, OBiGrad attains the parametric statistical rate $O(T^{-1/2}+n^{-1/2})$, improving the sufficient sample-size bound relative to the direct plug-in bound for a fixed iteration budget. We illustrate these gains on instrumental-variable and fitted-Q-regression problems, examining stationary-point estimation, returned population stationarity, and regularization sensitivity.
\end{abstract}

\section{Introduction}
\begingroup
{\color{black}Bilevel optimization captures a common hierarchical structure in machine learning, in which one learning task is constrained by the solution of another: an outer parameter is chosen to optimize a criterion that depends on the solution of a separate, inner learning problem.} Familiar examples include hyperparameter optimization and meta-learning \citep{pedregosa2016hyperparameter,franceschi2018bilevel}. In \emph{functional bilevel optimization}, the inner variable is a prediction function, and differentiation is performed in function space \citep{petrulionyte2024functional}. This viewpoint separates the learning problem from the particular kernel or neural-network model used to approximate its solution.

We focus on problems where the inner function is a conditional expectation: the average of a signal given observed covariates. Such functions arise in instrumental-variable regression \citep{newey2003instrumental,hartford2017deep,singh2019kernel} and the regression step underlying fitted value or Q-function methods \citep{ernst2005tree,farahmand2009regularized}. Specifically, we study a quadratic problem: the outer parameter determines a transformation of an observed input, and the inner least-squares regression predicts this signal from conditioning covariates. The population outer objective measures the prediction's mean squared error against the outcome on new observations. Its gradient with respect to the outer parameter, the \emph{hypergradient}, depends on both the inner function and how that function changes with the parameter.

{\color{black}The standard way to compute this hypergradient first estimates the inner solution from data, then differentiates that estimate to obtain the gradient update} \citep{elkhoury2025learning}. The inner function and its derivative are \emph{nuisance functions}: auxiliary functions learned to compute the outer gradient. Limited data, restricted model capacity, and regularization can leave estimation error even when the empirical inner problem is solved accurately. At a fixed outer parameter, averaging the plug-in formula over independent evaluation data is unbiased for its own mean, which can differ from the true population gradient. We call this difference the estimator's \emph{conditional bias}, holding the fitted nuisance functions fixed. Its expansion can contain \emph{first-order bias} terms \citep{chernozhukov2018double}, linear in either function's estimation error, unless additional alignment or cancellation removes them. The optimizer can therefore make small estimated updates while the population objective still admits improvement.

We construct \emph{OBiGrad} by adding a residual correction to the plug-in hypergradient. Besides the inner function and its derivative, this correction uses a regression of the outcome on the same conditioning covariates. We call the resulting hypergradient \emph{orthogonal} because its expectation is insensitive, to first order, to errors in these functions at their population values \citep{chernozhukov2022locally,foster2023orthogonal}. An exact mixed-bias identity bounds the corrected gradient's population bias by the derivative-estimation error multiplied by the sum of the two regression errors, all measured in root-mean-square norm. For errors of comparable size, generic plug-in sensitivity is linear and orthogonal sensitivity is quadratic in that size, separately from sampling fluctuation. This reduces the accuracy required of the individual regressions to obtain a useful outer update.

When observations are costly to collect, obtaining fresh data at every outer step may be impractical. We split the fixed dataset into two independent samples of size $n$, using the training sample to learn the regression functions and the other sample to evaluate hypergradients. Both samples are reused at every outer step. Each iterate depends on previous gradient evaluations. We bound the gradient error across all feasible parameters, as in Kernel Bilevel Optimization (KBO) \citep{elkhoury2025learning}, to cover every adaptive iterate. This turns second-order robustness into a population optimization guarantee without an error term growing linearly with the number of updates.

\noindent\textbf{Population optimization guarantees.}
Suppose the population outer objective is smooth and bounded below on a compact convex feasible set. Run projected gradient descent with a fixed step size $\eta$ no larger than the reciprocal of the gradient Lipschitz constant, using either OBiGrad or the direct plug-in hypergradient. The population gradient mapping $G_\eta$ is the difference between an iterate and its projected step using the exact population gradient, divided by $\eta$; its Euclidean norm measures stationarity. Under \Cref{assum:eif-integrability,ass:source,ass:approx,ass:erm}, use the tuning of \Cref{cor:source-only} and set $b=\nu(2+p)+\gamma$. For the respective trajectories $(\omega_t^{\mathrm{PI}})$ and $(\omega_t^{\mathrm{OBi}})$, the statistical bounds combined with \Cref{cor:gradient-descent-convergence} give
\begin{equation*}
\boxed{\begin{aligned}
\text{Direct plug-in:}\quad
\mathbb E\left[\min_{0\le t<T}\|G_\eta(\omega_t^{\mathrm{PI}})\|\right]
&\lesssim T^{-1/2}+n^{-1/2}+n^{-\nu/b},\\
\text{OBiGrad:}\quad
\mathbb E\left[\min_{0\le t<T}\|G_\eta(\omega_t^{\mathrm{OBi}})\|\right]
&\lesssim T^{-1/2}+n^{-1/2}+n^{-2\nu/b}.
\end{aligned}}
\end{equation*}
Expectations are over both samples and any fitting randomness; $\lesssim$ suppresses constants independent of $T$ and $n$ for fixed step size and initial objective gap. Each bound concerns the smallest population stationarity value among the first $T$ iterates. Here $\nu$ is the approximation exponent, while $p$ and $\gamma$ control class complexity. The OBiGrad bound uses \Cref{cor:source-only}; the direct plug-in bound follows by applying the same nuisance-moment and localization controls to its first-order bias decomposition (Appendices~\ref{sec:source-only} and~\ref{sec:appendix_plugin_uniform}). Thus OBiGrad attains $O(T^{-1/2}+n^{-1/2})$ stationarity when $\nu(2-p)\ge\gamma$.

Existing bilevel analyses control the effect of approximate hypergradients on optimization \citep{ghadimi2018approximation,ji2021bilevel}. {\color{black}Kernel bilevel optimization {(KBO)} instantiates this program in {a reproducing kernel Hilbert space (RKHS)}, using the representer theorem to obtain finite-dimensional empirical solutions \citep{elkhoury2025learning}.} KBO provides uniform gradient and population-stationarity guarantees for a regularized problem. Here the target is the unregularized quadratic conditional-moment objective, and the question is how to cancel first-order sensitivity to regression error before it enters the optimizer. The correction also has a statistical debiasing interpretation \citep{chernozhukov2018double,foster2023orthogonal}, developed in the appendix.

\noindent\textbf{Contributions.}
First, we give an explicit orthogonal hypergradient with an exact product-of-errors bias identity. Second, we establish uniform control for fixed-sample population optimization, yielding projected-stationarity guarantees along adaptive trajectories and a qualified comparison with plug-in bounds. Third, matched instrumental-variable and fitted-Q-regression experiments examine stationary-point estimation, population stationarity along optimization trajectories, and robustness to nonzero conditional residuals.

\noindent\textbf{Organization of the paper.}
\Cref{sec:problem} defines the quadratic problem, projected optimization, and plug-in bias; \Cref{sec:debias} develops the orthogonal hypergradient and its uniform error bound; and \Cref{sec:experiments} reports the main experiments. Appendix~\ref{sec:appendix_uniform} contains technical proofs and the hypergradient lower bound. Appendix~\ref{sec:appendix_nuisance_plugin} gives nuisance-learning conditions, the uniform direct plug-in comparison, and sample-complexity bounds; Appendices~\ref{sec:appendix_EIF}--\ref{sec:appendix_CLT} give additional statistical derivations and fixed-parameter inference; and Appendix~\ref{app:experiment-details} gives experimental details and robustness analyses.
\endgroup

\section{Conditional Moments as Functional Bilevel Optimization}
\label{sec:problem}
\subsection{Problem statement}\label{sec:problem_statement}
We consider the problem of finding a transformation of a covariate $Z$ that best matches an outcome $Y$ conditional on an auxiliary variable $X$. 
More precisely, let $(X,Y,Z)$ be drawn from a probability distribution $P$ supported on a product space $\mathcal X\times\mathcal Y\times\mathcal Z$ of Euclidean subsets $\mathcal{X}\subset\mathbb{R}^{d_X}$, $\mathcal{Y}\subset\mathbb{R}^{q}$ and $\mathcal{Z}\subset\mathbb{R}^{r}$.
We model the transformation using a parametric family $g_\omega:\mathcal Z\to\mathbb R^q$, indexed by $\omega\in\Omega\subset\mathbb R^d$, where $\Omega$ is compact and convex. Our goal is to choose $\omega$ so that the conditional means of $g_\omega(Z)$ and $Y$ given $X$ are as close as possible. 
Such a conditional moment problem can be naturally formulated as Functional Bilevel Optimization:
\begin{equation}
\label{eq:FBO}
    \min_{\omega\in\Omega}\mathcal{F}(\omega)\coloneqq \frac{1}{2}\mathbb{E}_P\bigl[\|Y - h^\star_{\omega,P}(X)\|^2\bigr]\quad\text{s.t.}\quad h^\star_{\omega,P}=\argmin_{h\in\LL}\frac{1}{2}\mathbb{E}_P\bigl[\|h(X) - g_\omega(Z)\|^2\bigr], \tag{FBO}
\end{equation}
where $\LL$ is the space of $\mathbb{R}^q$-valued functions on $\mathcal{X}$ that are square-integrable under the law of $X$, whereas $\|\cdot\|$ denotes the Euclidean norm with associated Euclidean inner product  $\langle\cdot,\cdot\rangle$ on $\mathbb{R}^q$. The inner problem (rightmost problem) regresses $g_{\omega}(Z)$ from $X$ for a given $\omega$ over all square integrable functions of $X$ and results in the conditional mean $h^\star_{\omega,P}(X)=\mathbb{E}_P[g_\omega(Z)\mid X]$ of $g_\omega(Z)$ given $X$. Writing $m_P^\star(X):=\mathbb E_P[Y\mid X]$, orthogonality gives $\mathcal F(\omega)=\tfrac12\mathbb E_P\|Y-m_P^\star(X)\|^2+\tfrac12\|m_P^\star-h^\star_{\omega,P}\|_{\LL}^2$. The first term is independent of $\omega$, so minimizing $\mathcal F$ matches the two conditional means. This problem formulation covers many cases of interest. We discuss two of them below.

\noindent\textbf{Instrumental variable regression.} A canonical instance is nonparametric instrumental variable (IV) regression, where $X$ is an instrument, $Z$ is the treatment, and $Y$ is the outcome. Under the quadratic losses, the inner minimizer projects the treatment-driven signal $g_\omega(Z)$ onto the instrument $X$. The bilevel problem thus learns the structural signal $g_\omega(Z)$ and evaluates it through a downstream prediction criterion on the outcome $Y$, connecting to nonparametric IV and modern deep IV procedures \citep{newey2003instrumental,hartford2017deep,bennett2019deep,shen2026instrumentalvariableanalysisstructural}.

\noindent\textbf{Fitted Bellman regression in reinforcement learning.} 
Quadratic inner losses also arise in reinforcement learning for fitted value iteration (FVI) and fitted Q-function iteration (FQI) methods. For a state--action pair $X=(S,A)$ and transition tuple $Z=(S,A,R,S')$, the Bellman target is $g_\omega(Z)=R+\gamma V_\omega(S')$ for FVI or $g_\omega(Z)=R+\gamma\max_{a'}Q_\omega(S',a')$ for FQI, with discount factor $\gamma\in(0,1)$. The inner minimizer averages over transition randomness in $S'$ given $(S,A)$, yielding the projected Bellman backup $h^\star_{\omega,P}(S,A)=\mathbb{E}_P[g_\omega(Z)\mid S,A]$, which the outer objective fits to returns or policy-performance targets. This covers the regression step in FQI and least-squares Bellman methods \citep{ernst2005tree,antos2007fitted,farahmand2009regularized}. For FQI, the hard maximum can fail to be differentiable at action ties; the theory below applies to instances satisfying \Cref{assum:eif-integrability}.

Throughout the rest of the paper we impose the following smoothness and boundedness conditions, which keep all of the population objects above well defined and square-integrable.
\begin{assum}[Smoothness and boundedness]
\label{assum:eif-integrability}
The map $\omega\mapsto g_\omega(z)$ is twice continuously differentiable for every $z$. Moreover, there exist finite constants $A, B, D, L > 0$ such that $\|Y\| \leq A$ almost surely and
\begin{align}
\sup_{\omega\in\Omega}\|g_\omega(z)\|\leq B,\quad
\sup_{\omega\in\Omega}\|D_\omega g_\omega(z)\|_F\leq D,\quad
\sup_{\omega\in\Omega}\|D_\omega^2g_\omega(z)\|_{\mathrm{op}}\leq L,\qquad \forall z\in \mathcal{Z}.
\end{align}
\end{assum}
Here $D_\omega g_\omega(z)$ and $D_\omega^2g_\omega(z)$ are the Jacobian and second derivative, $\|\cdot\|_F$ is the Frobenius norm, and $\|D_\omega^2g_\omega(z)\|_{\mathrm{op}}:=\sup_{\|u\|=1}(\sum_{k=1}^d\|\sum_{j=1}^du_j\partial_{\omega_j}\partial_{\omega_k}g_\omega(z)\|^2)^{1/2}$.

\subsection{First-order optimization for solving \ref{eq:FBO}}
\label{sec:projected-optimization}
When solving \eqref{eq:FBO} in practice, one only has access to 
i.i.d.\ samples from the distribution $P$. Throughout, expectations are taken with respect to these samples and any auxiliary
randomness used by the estimator. When $\mathcal{F}(\omega)$ is differentiable, a natural approach to solving \eqref{eq:FBO} is to optimize $\omega$ using a first-order method based on an estimator $\widehat{\nabla\mathcal{F}}(\omega)$ of the hypergradient $\nabla\mathcal{F}(\omega)$ of the value function $\mathcal{F}$. We view $\widehat{\nabla\mathcal F}:\Omega\to\mathbb R^d$ as a random gradient field constructed from these data and reused throughout the optimization. To keep the iterates in $\Omega$, we can use the projected update: 
\begin{align}\label{eq:projected_grad}
    \omega_{t+1}=\mathrm{Proj}_\Omega(\omega_t-\eta\widehat{\nabla\mathcal F}(\omega_t)),
\end{align}
where $\eta>0$ is a step size and $\mathrm{Proj}_\Omega$ denotes Euclidean projection onto the compact convex set $\Omega$. We hope to approach a stationary point $\omega^\star$ of the value function $\mathcal{F}$ on $\Omega$. Since the gradient need not vanish at a constrained stationary point, stationarity is typically measured by the norm of the \emph{gradient mapping}, defined for $\eta>0$ by
\begin{equation*}
    G_\eta(\omega;g)\coloneqq\frac{\omega-\mathrm{Proj}_\Omega(\omega-\eta g)}{\eta},\qquad G_\eta(\omega)\coloneqq G_\eta(\omega;\nabla\mathcal{F}(\omega)).
\end{equation*}
This mapping vanishes exactly at stationary points of $\mathcal{F}$ on $\Omega$ and reduces to the ordinary gradient when $\Omega=\mathbb{R}^d$. The following convergence result bounds the stationarity measure by the sum of an iteration-dependent term and a gradient estimation error term.
\begin{proposition}[Projected gradient descent convergence]
\label{cor:gradient-descent-convergence}
Suppose \Cref{assum:eif-integrability} holds, $\widehat{\nabla\mathcal F}:\Omega\to\mathbb{R}^d$ is measurable and satisfies $\mathbb{E}\left[\sup_{\omega\in\Omega}\|\widehat{\nabla\mathcal F}(\omega)-\nabla\mathcal F(\omega)\|\right]<\infty$, and the iterates are given by \cref{eq:projected_grad} with step-size $\eta$ satisfying $0<\eta\leq\frac{1}{L_{\mathcal F}}$, where $L_\mathcal{F}\coloneqq D^2+(A+B)L$ is the smoothness constant of $\mathcal{F}$ on $\Omega$ provided in \Cref{lem:F-smooth}. Then, for every $T\geq1$:
\begin{equation*}
    \mathbb{E}\left[\min_{0\leq t<T}\|G_\eta(\omega_t)\|\right]\lesssim\left(\frac{\mathcal{F}(\omega_0)-F_{\inf}}{\eta T}\right)^{1/2}+\overbrace{\mathbb{E}\left[\sup_{\omega\in\Omega}\|\widehat{\nabla\mathcal F}(\omega)-\nabla\mathcal{F}(\omega)\|\right]}^{\mathfrak{s}}.
\end{equation*}
\end{proposition}
\Cref{cor:gradient-descent-convergence} quantifies the effect of gradient estimation error on convergence and is proven in Appendix~\ref{sec:appendix_optimization}. The same proof controls the population mapping norm of the implementable iterate minimizing the estimated mapping norm; Appendix~\ref{app:optimization-details} records the corresponding inequality. With exact gradients ($\mathfrak{s}=0$), the result recovers the standard $O(T^{-1/2})$ stationarity rate for gradient descent on smooth nonconvex objectives. We next show that \emph{plug-in estimators}, a natural choice for gradient estimation, can yield unfavorable bounds on $\mathfrak{s}$, whereas a class of debiased estimators admits substantially sharper bounds.

\subsection{Plug-In Hypergradient Estimation and its First-Order Bias}
\label{sec:plug-in}

Plug-in estimators of the hypergradient $\nabla\mathcal{F}(\omega)$ can be constructed by leveraging its expression    \citep{petrulionyte2024functional,elkhoury2025learning}. Below, we present two of them. 

\noindent\textbf{Direct plug-in estimator.}
The chain rule allows expressing the hypergradient through $h^\star_{\omega,P}$ and its derivatives $j^\star_{k,\omega,P}\coloneqq\partial_{\omega_k}h^\star_{\omega,P}$. Specifically, defining $\varphi^{\mathrm{PI}}(h,j)\coloneqq\langle h(X)-Y,j(X)\rangle$ for $h,j\in\LL$, we obtain, under suitable differentiability and integrability conditions, the \emph{direct form}
\begin{equation}\label{eq:PI}
    \partial_{\omega_k}\mathcal{F}(\omega)=\mathbb{E}_P\bigl[\varphi^{\mathrm{PI}}(h^\star_{\omega,P},j^\star_{k,\omega,P})\bigr].
\end{equation}
Furthermore, the partial derivative  $j^\star_{k,\omega,P}(X)$ can be expressed as a conditional expectation of the model's partial derivatives $\partial_{\omega_k}g_{\omega}$, which suggests estimating it by regression:  
\begin{equation}\label{eq:jacobian}
    j^\star_{k,\omega,P}(X)\coloneqq \mathbb{E}_P[\partial_{\omega_k}g_\omega(Z)\mid X],\qquad k=1,\ldots,d.
\end{equation}
Substituting regression estimates $h$ and $j_k$ of $h^\star_{\omega,P}$ and $j^\star_{k,\omega,P}$ into \cref{eq:PI} suggests the plug-in estimator $[\widehat{\nabla\mathcal F}^{\mathrm{PI}}(\omega)]_k\coloneqq\mathbb{E}_{\hat{P}}[\varphi^{\mathrm{PI}}(h,j_k)]$, where $\hat{P}$ is the empirical distribution of an i.i.d.\ sample from $P$, independent of the training data used for computing the estimates  $h$ and $j_k$. The hypergradient estimator $\widehat{\nabla\mathcal F}^{\mathrm{PI}}(\omega) = ([\widehat{\nabla\mathcal F}^{\mathrm{PI}}(\omega)]_1,\dots, [\widehat{\nabla\mathcal F}^{\mathrm{PI}}(\omega)]_d)$ thus requires $d+1$ regression targets.

\noindent\textbf{Adjoint plug-in estimator.} A second expression for $\partial_{\omega_k}\mathcal{F}(\omega)$ follows from \cref{eq:PI,eq:jacobian} by application of the tower property for conditional expectations and allows constructing a second estimator requiring only two separate regression functions $h^\star_{\omega,P}$ and the conditional mean $m_P^\star$ defined in \Cref{sec:problem_statement}. Specifically, defining $\varphi^\mathrm{API}_{k,\omega}(h,m)\coloneqq\langle\partial_{\omega_k}g_\omega(Z),\,h(X)-m(X)\rangle$ for arbitrary $m,h\in \LL$, the \emph{adjoint form} is given by
\begin{equation}\label{eq:adj_PI}
    \partial_{\omega_k}\mathcal{F}(\omega)=\mathbb{E}_P\bigl[ \varphi_{k,\omega}^\mathrm{API}(h^\star_{\omega,P},m_P^{\star})\bigr].
\end{equation}
Substituting regression estimates $h$ and $m$ of $h^\star_{\omega,P}$ and $m_P^{\star}(X)$ into the above equation suggests the plug-in estimator $[\widehat{\nabla\mathcal F}^{\mathrm{API}}(\omega)]_k\coloneqq\mathbb{E}_{\hat{P}}[\varphi_{k,\omega}^{\mathrm{API}}(h,m)]$, where $\hat{P}$ is the empirical distribution of an i.i.d.\ sample from $P$, independent of the training data used to compute the estimates  $h$ and $m$.

\noindent\textbf{First-order bias of the direct plug-in estimator.} The functions $h^\star_{\omega,P}$,  $j^\star_{\omega,P}$ and $m^\star_P$ will be called \emph{nuisances} as they are unknown functions that must themselves be estimated before estimates of the hypergradient can be computed. Such two-stage estimation induces a bias that can be quantified. The next proposition provides a decomposition of the bias of $[\widehat{\nabla\mathcal F}^{\mathrm{PI}}(\omega)]_k$ into first order errors and a higher order term. A similar decomposition for $[\widehat{\nabla\mathcal F}^{\mathrm{API}}(\omega)]_k$ can be obtained but is omitted here for simplicity.
\begingroup
\setlength{\abovedisplayskip}{5pt plus 2pt minus 2pt}
\setlength{\belowdisplayskip}{4pt plus 2pt minus 2pt}
\begin{proposition}
\label{prop:plugin-bias}
For any $h,j_k\in\LL$, the estimator $[\widehat{\nabla\mathcal F}^{\mathrm{PI}}(\omega)]_k = \mathbb{E}_{\hat{P}}[\varphi^{\mathrm{PI}}(h,j_k)]$ satisfies:
\begin{align}\label{eq:plugin-bias}
    \mathbb{E}_P\bigl[[\widehat{\nabla\mathcal F}^{\mathrm{PI}}(\omega)]_k\bigr]-\partial_{\omega_k}\mathcal{F}(\omega) =&\underbrace{\langle h-h^\star_{\omega,P},j^\star_{k,\omega,P}\rangle_{\LL}}_{\mathrm{(I)}}+\underbrace{\langle h^\star_{\omega,P}-m^\star_P,j_k-j^\star_{k,\omega,P}\rangle_{\LL}}_{\mathrm{(II)}}\nonumber\\
    &+\underbrace{\langle h-h^\star_{\omega,P},j_k-j^\star_{k,\omega,P}\rangle_{\LL}}_{\mathrm{(III)}}.
\end{align}
\end{proposition}
\endgroup
Terms (I) and (II) are each \emph{linear} in one nuisance error, multiplied against the true, generically nonzero, value of the \emph{other} nuisance, so each has an upper bound linear in that error; only the product term (III) is higher order and can vanish faster than each of the errors under nonparametric rates \citep{bickel1993efficient,vanderVaart1998asymptotic}. Concretely, with $j^\star_{k,\omega,P}$ and $h^\star_{\omega,P}-m^\star_P$ fixed in $\LL$, estimators based on $n$ samples satisfying $\|h-h^\star_{\omega,P}\|_{\LL}=O(n^{-\alpha})$ and $\|j_k-j^\star_{k,\omega,P}\|_{\LL}=O(n^{-\beta})$, for $\alpha,\beta>0$, yield (I)$=O(n^{-\alpha})$, (II)$=O(n^{-\beta})$, and (III)$=O(n^{-(\alpha+\beta)})$ by Cauchy--Schwarz. Thus, the linear terms determine the leading rate bound, and improving one regression alone does not generally remove the plug-in bias. Appendices~\ref{sec:plugin-appendix} and~\ref{sec:appendix_plugin_uniform} give pointwise and uniform plug-in bounds, respectively. Under \Cref{assum:eif-integrability,ass:source,ass:approx,ass:erm} and the tuning of \Cref{cor:source-only}, with $b=\nu(2+p)+\gamma$, the expected uniform plug-in gradient error satisfies $\mathfrak{s}=O(n^{-1/2}+n^{-\nu/b})$ (Appendix~\ref{sec:appendix_plugin_uniform}). For fixed step size, these bounds give sufficient choices $T=O(\epsilon^{-2})$ and $n=O(\epsilon^{-b/\nu})$ for expected best-iterate stationarity error at most $\epsilon$. The familiar nuisance exponent $1/(2+p)$ is the special case $\gamma=0$; \Cref{cor:plugin-polynomial-rate} also records the corresponding conditional bound under the milder appendix assumptions. For standard classes, such as Sobolev balls, $p$ is typically of the form $p=d_X/s$, where $d_X$ is the dimension of $X$ and $s$ measures the smoothness of the nuisance functions.
Thus, high dimensionality or low smoothness can make the gradient method particularly sample-intensive (Appendix~\ref{sec:appendix_sample_complexity}). 
The next section constructs estimators with second-order bias and derives the conditions under which this improves the optimization bound in \cref{cor:gradient-descent-convergence}.

\section{Debiased Hypergradient Estimation}
\label{sec:debias}

\Cref{sec:problem} showed that the plug-in hypergradient estimator suffers a first-order bias from the nonparametric estimation of its nuisance components: its bias contains terms that are \emph{linear} in a single nuisance's estimation error, and such terms can dominate the higher-order terms under nonparametric rates. We now construct a corrected score that removes these first-order terms at the population level, and show that the resulting estimator's hypergradient error admits a finite-sample bound that holds uniformly over $\omega\in\Omega$. The key device is a \emph{mixed bias identity}: we show that the bias of the corrected score, evaluated at arbitrary candidate nuisances, collapses to a single term that is a \emph{product} of two nuisance errors, with no term linear in either error alone. This reduces the nuisance accuracy needed to control the hypergradient bias.

\subsection{OBiGrad: One-fold doubly robust hypergradient estimator}
\label{sec:eif-expansion}

The plug-in bias in Proposition~\ref{prop:plugin-bias} arises because substituting candidates $h$ and $j_k$ directly into $\varphi^{\mathrm{PI}}(h,j_k)$ leaves each candidate's error multiplied against the \emph{true} value of the other nuisance. We remove this first order bias by constructing an alternative estimator built from all nuisances and chosen so that linear terms in the bias decomposition cancel exactly.
To this end, we consider the following \emph{debiased score} defined for candidate nuisances $h,j_k,m\in\LL$
\begin{equation}
\label{eq:debiased-score}
    \varphi_{k,\omega}(h,j_k,m)\coloneqq\bigl\langle g_\omega(Z)-Y,\,j_k(X)\bigr\rangle+\bigl\langle\partial_{\omega_k}g_\omega(Z)-j_k(X),\,h(X)-m(X)\bigr\rangle.
\end{equation}
When evaluating $\varphi_{k,\omega}$ at the triplet $(h_{\omega,P}^{\star}, j^\star_{k,\omega,P}, m_P^{\star})$ and taking expectation under $P$, the first term in \cref{eq:debiased-score} recovers $\partial_{\omega_k}\mathcal{F}(\omega)$, while the second vanishes:
\begin{align}
\mathbb{E}_{P}\left[ \varphi_{k,\omega}(h_{\omega,P}^{\star}, j_{k,\omega,P}^{\star},m_{P}^{\star}) \right] &= \partial_{\omega_k}\mathcal{F}(\omega)\nonumber\\
&\quad+\underbrace{\mathbb{E}_{P}\left[ \bigl\langle\partial_{\omega_k}g_\omega(Z)-j^\star_{k,\omega,P}(X),\,h^\star_{\omega,P}(X)-m^\star_P(X)\bigr\rangle \right]}_{=0}.
\end{align}
The second term is corrective as it has zero mean when using the exact nuisance, but will act as a first-order debiasing term when using nuisance estimators. We use this score to construct OBiGrad, a hypergradient estimator whose \emph{double-robustness} property follows from the mixed-bias identity below. Specifically, given nuisance estimators $h$, $j_1,\dots, j_d$ and $m$ of $h_{\omega,P}^{\star}$, $j^\star_{1,\omega,P},\dots j^\star_{d,\omega,P}$ and $m_P^{\star}$ at some parameter value $\omega$, we define OBiGrad as follows:
\begin{align}\label{eq:obigrad}
    [\widehat{\nabla\mathcal F}^{\mathrm{OBi}}(\omega)]_k = \mathbb{E}_{\hat{P}}\left[\varphi_{k,\omega}(h,j_k,m)  \right],
\end{align}
where, similarly as for plug-in estimators, $\hat{P}$ is an empirical distribution obtained from i.i.d. samples from $P$ and independent of the training data used to fit the nuisances. \Cref{eq:debiased-score} augments the plug-in moment by a regression-residual correction chosen so that, at the true nuisance functions, the derivative of its expectation with respect to errors in either $j_k$ or the difference $h-m$ (which estimates $\mathbb{E}_P[g_\omega(Z)-Y\mid X]$) is zero. Thus small nuisance-estimation errors affect the expected hypergradient only at second order. This zero first-order sensitivity is the orthogonality property. The next theorem parallels \cref{prop:plugin-bias} and shows that $[\widehat{\nabla\mathcal F}^{\mathrm{OBi}}(\omega)]_k$ admits only second order bias, unlike its plug-in counterparts.
\begin{theorem}[Mixed bias]
\label{thm:mixed-bias}
Assume \ref{assum:eif-integrability}, for $\omega\in\Omega$, $1\leq k\leq d$ and $h,j_k,m\in\LL$, the hypergradient  estimator OBiGrad defined in \cref{eq:obigrad} satisfies:
\begin{equation}
\label{eq:mixed-bias}
    \mathbb{E}_P\bigl[[\widehat{\nabla\mathcal F}^{\mathrm{OBi}}(\omega)]_k\bigr]-\partial_{\omega_k}\mathcal{F}(\omega)=\bigl\langle j_k-j^\star_{k,\omega,P},\ (m-m^\star_P)-(h-h^\star_{\omega,P})\bigr\rangle_{\LL}.
\end{equation}
\end{theorem}
\Cref{eq:mixed-bias} holds exactly, for every candidate $h,j_k,m\in\LL$, and vanishes whenever \emph{either} $j_k=j^\star_{k,\omega,P}$ or $h-m=h^\star_{\omega,P}-m^\star_P$, regardless of the error in the other nuisance block. This is the double-robustness property. If $\|j_{k,\omega}-j^\star_{k,\omega,P}\|_{\LL}=O(n^{-\beta})$ and $\|h_\omega-h^\star_{\omega,P}\|_{\LL}+\|m-m^\star_P\|_{\LL}=O(n^{-\alpha})$, the decomposition in \cref{eq:mixed-bias} implies a bound of order $O(n^{-\alpha-\beta})$, which is $o(n^{-1/2})$ under the product-rate condition $\alpha+\beta>1/2$, a condition that can be satisfied even if both nuisances are estimated at nonparametric rates. This is in stark contrast with requirements when using a plug-in estimator, as discussed in \cref{sec:plug-in}.

\subsection{Fast uniform rates for OBiGrad}
\label{sec:debiased-clt}
We establish a control on the uniform gradient estimation error $\mathfrak{s}$ when using OBiGrad as a gradient estimator under a general entropy--approximation scale condition on the hypothesis classes when the nuisances are  fitted by least-squares up to some vanishing sub-optimality error. To allow a unified treatment of the nuisances, we introduce the index $\ell\in\{h,j,m\}$ which indicates which regression problem is being considered. Thus, we define $f^\star_{\ell,\omega}$ to be the nuisance function of index $\ell$ so that  $(f^\star_{h,\omega},f^\star_{j,\omega},f^\star_{m,\omega})= (h^\star_\omega,j^\star_\omega,m^\star)$, while $T^{\ell}_\omega$ refers to its corresponding nuisance's target, i.e. $(T^h_\omega(O),T^j_\omega(O),T^m_0(O))=(g_\omega(Z),\dot g_\omega(Z),Y)$ where $O=(X,Y,Z)\sim P$. Here $\dot g_\omega=(\partial_{\omega_k}g_\omega)_{k=1}^d$, $\omega_h=\omega_j=\omega$, and $\omega_m=0$; quantities indexed by $m$ do not depend on $\omega$. Hence, each nuisance function $f^\star_{\ell,\omega}$ minimizes the regression objective
\begin{align}
 R^\ell_\omega(f)=\mathbb{E}_{P}\left[\ell^\ell_\omega(O;f)\right],\qquad\ell^\ell_\omega(O;f)=\frac12\|f(X)-T^\ell_{\omega_\ell}(O)\|^2.
 \label{eq:risks-source}
\end{align}
Each nuisance $f_{\ell,\omega}^{\star}$ is fitted in a class $\mathcal F_\ell(t)$ of functions defined on $\mathcal X$ and indexed by \emph{approximation scale} $t\in(0,1]$ that balances approximation error and statistical complexity and may depend on sample-size $n$\footnote{$\mathcal F_\ell(t)$ could be a ball of radius $\frac{R}{t}$ in some suitable function space. The approximation scale $t$ will typically be indexed by training sample-size and decrease to $0$ as the sample-size increases.}. Specifically, given a dataset $S_1$ of $n$, i.i.d. observations, denote by $P_{n,1}$ the empirical measure on $S_1$ and define the empirical risk of nuisance $\ell$ as $\widehat R^\ell_\omega(f)=\mathbb{E}_{P_{n,1}}\left[\ell^\ell_\omega(\cdot;f)\right]$. The nuisance estimate $\widehat f_{\ell,\omega}$ is then obtained by approximately minimizing the empirical risk $\widehat R^\ell_\omega$ in the class $\mathcal{F}_\ell(t_n)$ for a suitable choice of approximation scale $t_n\in (0,1]$ to be defined later. The estimates $\widehat f_{\ell,\omega}\in\cF_\ell(t_n)$ are measurable functions of $S_1$ and, optionally, of an additional random source $U$ independent of all data. \cref{thm:mixed-bias} suggests controlling the estimation error of each nuisance function in order to bound the overall gradient estimation error. To this end, we impose complexity and approximation assumptions on the corresponding function classes, which allow us to derive explicit error bounds in terms of the sample size. The assumption is stated in terms of the uniform covering number $N(u,\mathcal{C},\|\cdot\|_\infty)$ of a function class $\mathcal{C}$ \citep{vanderVaart1996weak}, defined as the smallest number of functions in $\mathcal{C}$ such that every $f\in\mathcal{C}$ lies within distance $u$ of at least one of them under the supremum norm $\|\cdot\|_\infty$.
\begin{assum}[Nuisance-class complexity]\label{ass:source}
For every $\ell\in\{h,j,m\}$ and $t\in(0,1]$, the class $\cF_\ell(t)$ is nonempty, deterministic, pointwise measurable, and uniformly bounded:
\begin{equation}
 \sup_{f\in\cF_\ell(t)}\sup_x\norm{f(x)}\le  b_{\ell},\qquad  \text{ where }\quad (b_h,b_j,b_m)=(B,D,A).
 \label{eq:source-envelope}
\end{equation}
Moreover, there exist fixed $p\in(0,2)$, $\gamma\ge0$, and finite constants $K_\ell\ge0$ such that for all $\ell$
\begin{equation}
\log N\bigl(\epsilon b_\ell,\cF_\ell(t),\Vert\cdot\Vert_{\infty} \bigr)\le K_\ell t^{-\gamma}\epsilon^{-p},\qquad \forall t,\epsilon \in (0,1].   \label{eq:source-entropy}
\end{equation}
\end{assum}
\begin{assum}[Approximation]\label{ass:approx}
There exist $\nu>0$ and finite $ L_\ell\ge 0$ such that, for all $\ell$,
\begin{equation}
 \sup_{\omega\in\Omega}\inf_{f\in\cF_\ell(t)}
       \norm{f-f^\star_{\ell,\omega}}_{L^2(P_X)}^2
 \le L_\ell^2t^{2\nu},\qquad \forall t\in(0,1].
 \label{eq:source-approx}
\end{equation}
\end{assum}
\Cref{ass:source} governs the empirical-process terms and \Cref{ass:approx} the approximation error. The polynomial entropy condition in \Cref{eq:source-entropy} allows nonparametric nuisance classes and also covers simpler classes. \Cref{eq:source-approx} is an approximation condition and does not require the true nuisances to belong to the search classes for any value of the approximation scale $t$. The parameter $\nu$ determines the rate at which the approximation error decreases as $t\downarrow0$. The complexity of the classes is characterized by the exponents $\gamma$ and $p$: $\gamma$ determines how the entropy grows as the approximation scale $t$ decreases, while $p$ governs the dependence of the entropy on the covering radius. The next assumption controls the fitting errors of the nuisance estimates in terms of a prescribed threshold.
\begin{assum}[Inexact least-squares fitting]\label{ass:erm}
The nuisance estimates satisfy, almost surely and simultaneously for every $\omega\in\Omega$,
\begin{equation}
 \widehat R^\ell_\omega(\widehat f_{\ell,\omega})\le \inf_{f\in\cF_\ell(t_n)}\widehat R^\ell_\omega(f)+\varepsilon_{\ell,n},\qquad \ell\in\{h,j,m\},
 \label{eq:erm-source}
\end{equation}
where $\varepsilon_{\ell,n}\ge0$ is a prescribed fitting error on the nuisance $\ell$ at sample-size $n$.
\end{assum}
The following result provides a bound on the expected uniform gradient estimation error $\mathfrak{s}$ arising in \cref{cor:gradient-descent-convergence}, in terms of the sample size and the parameters $\nu$, $\gamma$, and $p$.

\begin{theorem}[Second-order uniform hypergradient error]
\label{cor:source-only}
Write $\varphi_\omega(h,j,m):=(\varphi_{k,\omega}(h,j_k,m))_{k=1}^d$ for $j=(j_k)_{k=1}^d$. Assume \ref{assum:eif-integrability} and \ref{ass:source}, and consider the gradient estimator defined in \cref{eq:obigrad} using the nuisance estimates $(\widehat h_\omega,\widehat j_\omega,\widehat m)=(\widehat f_{h,\omega},\widehat f_{j,\omega},\widehat f_{m,\omega})$. Define
\[
R_{h,n}\coloneqq\sup_{\omega\in\Omega}\|\widehat h_\omega-h^\star_{\omega,P}\|_{\LL},\; R_{j,n}\coloneqq \sup_{\omega\in\Omega} \|\widehat j_\omega-j^\star_{\omega,P}\|_{\LL^d},\;R_{m,n}\coloneqq\|\widehat m-m^\star_P\|_{\LL},
\]
and let $\rho_{\ell,n}\coloneqq(\mathbb E[R_{\ell,n}^2])^{1/2}$ for $\ell\in\{h,j,m\}$ and $\delta_n\coloneqq(nt_n^\gamma)^{-1/(2+p)}$. Then there exist finite constants $C_0$, depending only on $A,B,D,L,d,q$ and $\operatorname{diam}(\Omega)$, and $C$, depending in addition on $p,\gamma$ and $(K_\ell)_{\ell\in\{h,j,m\}}$, such that, for every $n\ge1$ and every deterministic $t_n\in(0,1]$,
\begin{equation}
\mathbb E\Bigl[\sup_{\omega\in\Omega}\bigl\|\widehat{\nabla\mathcal F}^{\mathrm{OBi}}(\omega)-\nabla\mathcal F(\omega)\bigr\|\Bigr]\le C_0n^{-1/2}+C\bigl\{\delta_n\textstyle\sum_{\ell}\rho_{\ell,n}+\delta_n^2+\rho_{j,n}(\rho_{h,n}+\rho_{m,n})\bigr\}.
\label{eq:generic-source}
\end{equation}
If, in addition, \ref{ass:approx} and \ref{ass:erm} hold, then $\rho_{\ell,n}^2\le C_\ell\{t_n^{2\nu}+(nt_n^\gamma)^{-2/(2+p)}+\varepsilon_{\ell,n}\}$ for each $\ell\in\{h,j,m\}$, with finite $C_\ell$ independent of $n$, $t_n$ and $\varepsilon_{\ell,n}$. In particular, let $b:=\nu(2+p)+\gamma$ and, for any training sample-size $n\geq 1$, set the approximation scale to $t_n=n^{-1/b}$ with sub-optimality error such that $\varepsilon_{\ell,n}\le c_\ell n^{-2\nu/b}$ for fixed $c_\ell<\infty$. Then $\rho_{\ell,n}=O(n^{-\nu/b})$ and $\delta_n=n^{-\nu/b}$, and there exists a finite constant $C_1$, depending only on $A,B,D,L,p,\nu,\gamma,d,q$, $\operatorname{diam}(\Omega)$, and $(c_\ell,K_\ell,L_\ell)_{\ell\in\{h,j,m\}}$, such that, for every $n\ge1$,
\begin{equation}
\mathfrak{s}\le C_0 n^{-1/2}+C_1 n^{-2\nu/b}.
\label{eq:tuned-source}
\end{equation}
\end{theorem}
\cref{cor:source-only}, which is proven in Appendix~\ref{sec:source-only}, shows that the expected uniform gradient estimation error consists of three contributions: an oracle sampling term $O(n^{-1/2})$; a nuisance-dependent empirical-process term $\delta_n\sum_\ell\rho_{\ell,n}+\delta_n^2$, controlled by localization under \Cref{ass:source}; and the population bias $\rho_{j,n}(\rho_{h,n}+\rho_{m,n})$, of product order by orthogonality (\cref{thm:mixed-bias}). Under \Cref{ass:approx,ass:erm} and the above tuning, the last two are $O(n^{-2\nu/b})$, and are dominated by the parametric term whenever $\nu(2-p)\geq\gamma$. In particular, if the target nuisance functions belong to $\mathcal{F}_{\ell}(t_0)$ for some $t_0>0$, then the classes can be redefined as $\mathcal{F}_{\ell}(t)=\mathcal{F}_{\ell}(t_0)$ for all $t\leq t_0$. The entropy bound is then uniform in $t$, so that $\gamma=0$ and the condition $\nu(2-p)\geq\gamma$ is automatically satisfied. Consequently, the non-parametric nuisance estimation does not affect the overall $O(n^{-1/2})$ gradient rate. In this fixed-class regime, the corresponding plug-in upper bound has a first-order nuisance contribution $O(n^{-1/(2+p)})$.  Moreover, Appendix~\ref{sec:appendix_minimax} establishes a matching worst-case $n^{-1/2}$ lower bound for hypergradient estimation over the distribution class defined there. This matches the OBiGrad rate whenever $\nu(2-p)\geq\gamma$. Appendix~\ref{sec:appendix_nuisance_plugin} gives the uniform direct plug-in and sample-complexity comparisons, while Appendix~\ref{sec:appendix_CLT} gives fixed-parameter inference for the cross-fitted estimator.

\section{Numerical Experiments}
\label{sec:experiments}
\begingroup
\setlength{\textfloatsep}{5pt plus 1pt minus 2pt}
\setlength{\floatsep}{5pt plus 1pt minus 2pt}
\setlength{\intextsep}{5pt plus 1pt minus 2pt}
\setlength{\abovecaptionskip}{2pt}
\setlength{\belowcaptionskip}{0pt}
\captionsetup{skip=2pt}

We evaluate OBiGrad on two synthetic problems with known unregularized population objectives: a triangular instrumental-variable (IV) model with a quadratic first stage, and a two-feature model of a single projected Bellman backup. These designs instantiate conditional-moment problems arising in IV estimation \citep{newey2003instrumental} and fitted $Q$ methods \citep{ernst2005tree}. We use two-fold cross-fitting \citep{chernozhukov2018double}: nuisances are fitted on one half of the data and the score in \eqref{eq:debiased-score} is evaluated on the other; exchanging the folds and averaging the held-out gradient fields gives the estimator used below.

To isolate the effect of orthogonalization, we compare plug-in and corrected gradients for Gaussian Nadaraya--Watson regression (NW; \citealp{nadaraya1964estimating,watson1964smooth}) and the kernel-ridge specialization of KBO \citep{elkhoury2025learning}. Each pair shares observations, folds, tuning, fitted inner functions, and sensitivities; OBiGrad additionally estimates $m^\star(X)=\mathbb E[Y\mid X]$ with the same regression operator. Tuning uses observed inner targets within training folds, while population quantities serve only for evaluation and oracle references. After exploratory pilots, we fixed the protocol and ran two independent suites of $400$ replications per setting and total sample size $N$. Appendix~\ref{app:experiment-details} provides the full specification and supplementary analyses.

We first solve the estimated gradient equations directly to isolate statistical error from iterative optimization error. Figure~\ref{fig:main-roots} shows that OBiGrad reduces root RMSE relative to the population optimizer for both learners throughout the sample-size grid. At $N=800$, IV RMSE falls from $0.0251$ to $0.0124$ for NW and from $0.0153$ to $0.0123$ for KBO; in fitted $Q$-regression, it falls from $0.0283$ to $0.0094$ and from $0.0117$ to $0.0093$, respectively. The corrected estimates approach the finite-sample oracle benchmark computed with true nuisances.

\begin{figure}[t]
    \centering
    \includegraphics[width=.97\linewidth]{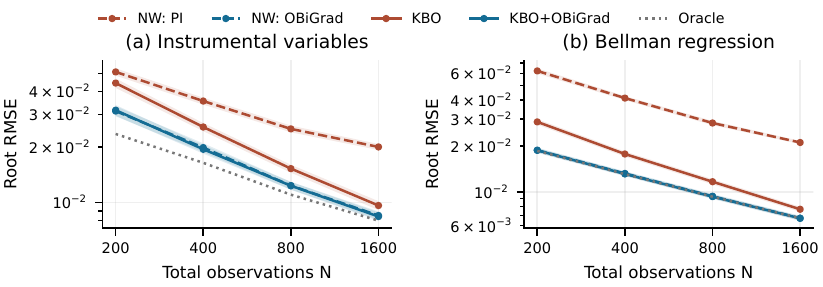}
    \caption{Root RMSE is computed over $800$ replications per $N$ in the zero-residual designs. Dashed and solid curves correspond to NW and KBO, respectively; bands show pointwise $95\%$ Monte Carlo intervals.}
    \label{fig:main-roots}
\end{figure}

To assess whether these gains translate into improved population stationarity, we run projected gradient descent from zero with step size $0.1$ on $[-2,2]^d$, holding the estimated gradient field fixed. At budget $T$, we return the iterate with the smallest estimated projected-update norm among the first $T$ iterates, without using population information for selection. Figure~\ref{fig:main-optimization} evaluates the population gradient-mapping norm at this iterate. Orthogonalization lowers the eventual statistical error for both learners, although improvements need not occur at every intermediate iteration. Descent with the exact population gradient provides a reference for the remaining optimization error.

\begin{figure}[t]
    \centering
    \includegraphics[width=.97\linewidth]{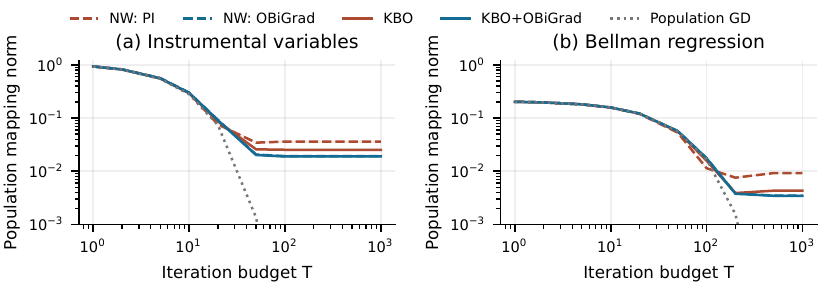}
    \caption{Population gradient-mapping norms at the returned iterates are averaged over $800$ replications with $N=400$. Bands show pointwise $95\%$ Monte Carlo intervals. Selection uses estimated updates; population norms are evaluated afterward. Population GD eventually falls below the displayed range.}
    \label{fig:main-optimization}
\end{figure}

The preceding designs satisfy $h^\star_{\omega^\star}=m^\star$, so a shared linear smoother can cancel regularization bias at the corrected population root. We therefore introduce a conditional residual of amplitude $0.25$ orthogonal to the true sensitivity, preserving the population optimizer while breaking this equality. Table~\ref{tab:main-residual} shows that correction continues to improve root estimation and returned stationarity, although the KBO gains are smaller in fitted $Q$-regression. The IV extension now represents an approximate conditional-moment problem. Across both residual settings, all sample sizes, and both learners, the $64$ prespecified paired comparisons of root squared error and returned stationarity favor correction under approximate simultaneous $95\%$ Monte Carlo intervals. The penalty sweeps and further checks in Appendix~\ref{app:optimization-details} qualify this finding: orthogonalization retains product-order bias and need not improve performance for every penalty or data-generating process.

\begin{table}[ht]
\centering
\small
\renewcommand{\arraystretch}{0.9}
\setlength{\tabcolsep}{4pt}
\caption{With a nonzero conditional residual and $N=800$, we report root RMSE and the mean returned population mapping norm after $1000$ updates. ``OBi'' denotes correction of the indicated learner.}
\label{tab:main-residual}
\begin{tabular}{llrrrr}
\toprule
Design & Learner & PI root & OBi root & PI mapping & OBi mapping \\
\midrule
IV & NW & 0.0354 & 0.0168 & 0.0375 & 0.0168 \\
IV & KBO & 0.0205 & 0.0169 & 0.0202 & 0.0168 \\
Bellman & NW & 0.0416 & 0.0265 & 0.0095 & 0.0067 \\
Bellman & KBO & 0.0284 & 0.0264 & 0.0070 & 0.0066 \\
\bottomrule
\end{tabular}

\end{table}

\section{Conclusion}
\label{sec:discussion}
We studied population bilevel optimization with learned conditional moments and constructed OBiGrad, an orthogonal hypergradient estimator whose bias is a product of nuisance errors. Uniform gradient control translates this correction into sharper population-stationarity bounds. In the studied IV and fitted-Q-regression regimes, orthogonalizing a fixed nuisance learner improves stationary-point estimation and the population stationarity of returned optimization iterates, including with matched KBO fits. These gains persist beyond exact conditional fit, while regularization can still induce product-order bias. Limitations remain: our guarantees concern gradient estimation and approximate stationarity, not a full end-to-end analysis of arbitrary bilevel algorithms, and they rely on nuisance learners satisfying the complexity, approximation, and fitting conditions of \Cref{ass:source,ass:approx,ass:erm}. Verifying these conditions for highly adaptive neural learners and extending the method to infinite-horizon reinforcement learning, meta-learning, and other structured bilevel problems remain natural next steps.
\endgroup

\subsubsection*{Acknowledgements}
Fares El Khoury and Michael Arbel are supported by the ANR project BONSAI (grant ANR-23-CE23-0012-01). Houssam Zenati is supported by the Gatsby Charitable Foundation.

\bibliography{references}
\newpage
\appendix
\section*{Appendix}

\addtocontents{toc}{\protect\setcounter{tocdepth}{2}}
{
    \hypersetup{linkcolor=black}
    \tableofcontents
}

\section{Uniform Control and Optimization}
\label{sec:appendix_uniform}

\subsection{Smoothness of \texorpdfstring{$\mathcal{F}$}{F}}
\begin{lemma}[Smoothness]
\label{lem:F-smooth}
Under \Cref{assum:eif-integrability}, $\mathcal{F}$ is $L_{\mathcal F}$-smooth on $\Omega$ with $L_{\mathcal F}\coloneqq D^{2}+(A+B)L$, and $0\le\mathcal{F}\le\tfrac12(A+B)^2$, so $F_{\inf}>-\infty$.
\end{lemma}
 
\begin{proof}
Collect the derivative nuisances into the matrix-valued map
\begin{equation*}
  j^\star_{\omega,P}(x)\coloneqq\bigl(j^\star_{1,\omega,P}(x),\ldots,j^\star_{d,\omega,P}(x)\bigr)\in\mathbb{R}^{q\times d},
\end{equation*}
whose $k$-th column is $j^\star_{k,\omega,P}(x)$ and whose block norm $\|\cdot\|_{2,d}$ is the Frobenius norm of the stack. Since the $k$-th entry of $j^\star_{\omega,P}(x)^\top v$ is $\langle j^\star_{k,\omega,P}(x),v\rangle$, the direct form of the hypergradient stacks to $\nabla\mathcal{F}(\omega)=\mathbb{E}_P[j^\star_{\omega,P}(X)^\top(h^\star_{\omega,P}(X)-Y)]$. Let $\omega_1,\omega_2\in\Omega$ and abbreviate $j_i\coloneqq j^\star_{\omega_i,P}(X)$, $h_i\coloneqq h^\star_{\omega_i,P}(X)$ and $R_i\coloneqq h_i-Y$ for $i=1,2$. We have
\begin{align}
\label{eq:smooth-split}
  \bigl\|\nabla\mathcal{F}(\omega_1)-\nabla\mathcal{F}(\omega_2)\bigr\|
  &=\Bigl\|\mathbb{E}_P\bigl[\,j_1^\top R_1-j_2^\top R_2\,\bigr]\Bigr\|\nonumber\\
  &=\Bigl\|\mathbb{E}_P\bigl[\,j_1^\top(R_1-R_2)+(j_1-j_2)^\top R_2\,\bigr]\Bigr\|\nonumber\\
  &=\Bigl\|\mathbb{E}_P\bigl[\,j_1^\top(h_1-h_2)+(j_1-j_2)^\top R_2\,\bigr]\Bigr\|\nonumber\\
  &\le\mathbb{E}_P\bigl[\,\bigl\|j_1^\top(h_1-h_2)\bigr\|
       +\bigl\|(j_1-j_2)^\top R_2\bigr\|\,\bigr]\nonumber\\
  &\le\mathbb{E}_P\bigl[\,\|j_1\|_{2,d}\,\|h_1-h_2\|
       +\|j_1-j_2\|_{2,d}\,\|R_2\|\,\bigr]\nonumber\\
  &\le D\;\mathbb{E}_P\bigl\|h^\star_{\omega_1,P}(X)-h^\star_{\omega_2,P}(X)\bigr\|\nonumber\\
  &\qquad+(A+B)\;\mathbb{E}_P\bigl\|j^\star_{\omega_1,P}(X)-j^\star_{\omega_2,P}(X)\bigr\|_{2,d},
\end{align}
where the second line adds and subtracts the mixed term $j_1^\top R_2$; the third uses $R_1-R_2=(h_1-Y)-(h_2-Y)=h_1-h_2$, the outcome $Y$ being independent of $\omega$; the fourth uses Jensen's inequality and the triangle inequality; the fifth uses $\|M^\top v\|\le\|M\|_{\mathrm{op}}\|v\|\le\|M\|_{2,d}\|v\|$; and the last uses the almost sure bounds supplied by \Cref{assum:eif-integrability} through conditional Jensen, namely $\|j^\star_{\omega,P}(X)\|_{2,d}\le D$ and $\|R_2\|\le\|h^\star_{\omega_2,P}(X)\|+\|Y\|\le A+B$.
 
It remains to show that both nuisances are Lipschitz in $\omega$. By conditional Jensen and the mean value theorem applied to $\omega\mapsto g_\omega(z)$ along the segment
$[\omega_2,\omega_1]\subset\Omega$,
\begin{equation}
\label{eq:h-lip}
\begin{aligned}
  \mathbb{E}_P\left[\bigl\|h^\star_{\omega_1,P}(X)-h^\star_{\omega_2,P}(X)\bigr\|\right]
  &\le\mathbb{E}_P\left[\bigl\|g_{\omega_1}(Z)-g_{\omega_2}(Z)\bigr\|\right]\\
  &\le\sup_{\omega\in\Omega,\,z\in\mathcal{Z}}\|D_\omega g_\omega(z)\|_{F}\;\|\omega_1-\omega_2\|
  \;\le\;D\|\omega_1-\omega_2\| .
\end{aligned}
\end{equation}
Similarly, for any unit $u\in\mathbb{R}^d$ the directional derivative of $\omega\mapsto D_\omega g_\omega(z)$ along $u$ is the matrix with columns $\sum_{j}u_j\partial_{\omega_j}\partial_{\omega_k}g_\omega(z)$, $k=1,\dots,d$, whose Frobenius norm is precisely the quantity bounded by $L$ in \Cref{assum:eif-integrability}. Hence $\omega\mapsto D_\omega g_\omega(z)$ is $L$-Lipschitz in $\|\cdot\|_F$, uniformly in $z$, and conditional Jensen gives
\begin{equation}
\label{eq:j-lip}
  \mathbb{E}_P\left[\bigl\|j^\star_{\omega_1,P}(X)-j^\star_{\omega_2,P}(X)\bigr\|_{2,d}\right]
  \le\mathbb{E}_P\left[\bigl\|D_\omega g_{\omega_1}(Z)-D_\omega g_{\omega_2}(Z)\bigr\|_{F}\right]
  \le L\|\omega_1-\omega_2\| .
\end{equation}
Substituting \eqref{eq:h-lip} and \eqref{eq:j-lip} into \eqref{eq:smooth-split} yields
\begin{equation*}
  \bigl\|\nabla\mathcal{F}(\omega_1)-\nabla\mathcal{F}(\omega_2)\bigr\|
  \le\underbrace{\bigl(D^{2}+(A+B)L\bigr)}_{L_{\mathcal F}}\|\omega_1-\omega_2\| .
\end{equation*}
Finally $0\le\mathcal{F}(\omega)=\tfrac12\mathbb{E}_P\left[\|Y-h^\star_{\omega,P}(X)\|^{2}\right]\le\tfrac12(A+B)^2$, so $F_{\inf}$ is finite.
\end{proof}

\subsection{Projected gradient descent with an inexact hypergradient}
\label{sec:appendix_optimization}
Since $\Omega$ is compact and convex (\Cref{sec:problem}), the Euclidean projection $\mathrm{Proj}_\Omega$ onto $\Omega$ is well-defined and non-expansive: $\|\mathrm{Proj}_\Omega(y)-\mathrm{Proj}_\Omega(y')\|\leq\|y-y'\|$ for all $y,y'\in\mathbb{R}^d$. With projection, the correct notion of gradient descent restricted to $\Omega$ replaces the unconstrained update by
\begin{equation*}
    \omega_{t+1}=\mathrm{Proj}_\Omega\bigl(\omega_t-\eta\widehat{\nabla\mathcal F}(\omega_t)\bigr),
\end{equation*}
and empirical stationarity of this projected iteration is no longer measured by the raw gradient norm, but by the \emph{gradient mapping}. For a point $\omega\in\Omega$, a vector $g\in\mathbb{R}^d$, and step size $\eta>0$, define
\begin{equation*}
    G_\eta(\omega;g)\coloneqq\frac{1}{\eta}\Bigl(\omega-\mathrm{Proj}_\Omega(\omega-\eta g)\Bigr),
\end{equation*}
and write $G_\eta(\omega)\coloneqq G_\eta(\omega;\nabla\mathcal{F}(\omega))$ for the gradient mapping evaluated at the population gradient, which is the quantity Proposition~\ref{cor:gradient-descent-convergence} controls; $G_\eta(\omega)=0$ if and only if $\omega$ is a stationary point of $\mathcal{F}$ on $\Omega$ in the usual projected sense. We record two standard facts about $G_\eta(\omega;\cdot)$. First, non-expansiveness of $\mathrm{Proj}_\Omega$ gives, for any $g,g'\in\mathbb{R}^d$,
\begin{equation}
\label{eq:gradmap-lipschitz}
    \|G_\eta(\omega;g)-G_\eta(\omega;g')\|\leq\|g-g'\|.
\end{equation}
Second, writing $z=\mathrm{Proj}_\Omega(\omega-\eta g)$ and $G=G_\eta(\omega;g)=(\omega-z)/\eta$, the variational characterization of the Euclidean projection onto the convex set $\Omega$, $\langle(\omega-\eta g)-z,u-z\rangle\leq0$ for all $u\in\Omega$, applied at $u=\omega\in\Omega$ gives $\langle\eta G-\eta g,\eta G\rangle\leq0$, i.e.
\begin{equation}
\label{eq:gradmap-inner-product}
    \langle g,G\rangle\geq\|G\|^2.
\end{equation}
When $\Omega=\mathbb{R}^d$, $\mathrm{Proj}_\Omega$ is the identity and $G_\eta(\omega;g)=g$, so both facts reduce to their unconstrained counterparts.

\begin{proof}[Proof of Proposition~\ref{cor:gradient-descent-convergence}]
Fix a realization of $\widehat{\nabla\mathcal F}$ and set $\Delta\coloneqq\sup_{\omega\in\Omega}\|\widehat{\nabla\mathcal F}(\omega)-\nabla\mathcal F(\omega)\|$. Write $e_t\coloneqq\widehat{\nabla\mathcal F}(\omega_t)-\nabla\mathcal{F}(\omega_t)$, so that $\|e_t\|_2\leq\Delta$, and let $\widehat G_t\coloneqq G_\eta(\omega_t;\widehat{\nabla\mathcal F}(\omega_t))=(\omega_t-\omega_{t+1})/\eta$ denote the \emph{empirical} gradient mapping actually used by the update, as opposed to the population gradient mapping $G_t\coloneqq G_\eta(\omega_t)=G_\eta(\omega_t;\nabla\mathcal{F}(\omega_t))$ that the proposition bounds. By the Lipschitz property \eqref{eq:gradmap-lipschitz} applied with $g=\widehat{\nabla\mathcal F}(\omega_t)$ and $g'=\nabla\mathcal{F}(\omega_t)$,
\begin{equation}
\label{eq:gradmap-hat-vs-pop}
    \|G_t\|\leq\|\widehat G_t\|+\|e_t\|\leq\|\widehat G_t\|+\Delta.
\end{equation}
It therefore suffices to bound $\min_{0\leq t<T}\|\widehat G_t\|$. By $L_{\mathcal F}$-smoothness and $\omega_{t+1}-\omega_t=-\eta\widehat G_t$,
\begin{equation*}
    \mathcal{F}(\omega_{t+1})\leq\mathcal{F}(\omega_t)-\eta\langle\nabla\mathcal{F}(\omega_t),\widehat G_t\rangle+\frac{L_{\mathcal F}\eta^2}{2}\|\widehat G_t\|^2.
\end{equation*}
Writing $\langle\nabla\mathcal{F}(\omega_t),\widehat G_t\rangle=\langle\widehat{\nabla\mathcal F}(\omega_t),\widehat G_t\rangle-\langle e_t,\widehat G_t\rangle$ and applying the inner-product identity \eqref{eq:gradmap-inner-product} with $g=\widehat{\nabla\mathcal F}(\omega_t)$, $G=\widehat G_t$, together with Cauchy--Schwarz on the second term,
\begin{equation*}
    \langle\nabla\mathcal{F}(\omega_t),\widehat G_t\rangle\geq\|\widehat G_t\|^2-\|e_t\|\,\|\widehat G_t\|.
\end{equation*}
Combining the two displays and using $\eta\leq1/L_{\mathcal F}$ so that $1-L_{\mathcal F}\eta/2\geq1/2$,
\begin{equation*}
    \mathcal{F}(\omega_{t+1})\leq\mathcal{F}(\omega_t)-\eta\|\widehat G_t\|^2\Bigl(1-\frac{L_{\mathcal F}\eta}{2}\Bigr)+\eta\|e_t\|\,\|\widehat G_t\|\leq\mathcal{F}(\omega_t)-\frac{\eta}{2}\|\widehat G_t\|^2+\eta\|e_t\|\,\|\widehat G_t\|.
\end{equation*}
By Young's inequality, $\eta\|e_t\|\,\|\widehat G_t\|\leq\frac{\eta}{4}\|\widehat G_t\|^2+\eta\|e_t\|^2$, so
\begin{equation*}
    \mathcal{F}(\omega_{t+1})\leq\mathcal{F}(\omega_t)-\frac{\eta}{4}\|\widehat G_t\|^2+\eta\|e_t\|^2\leq\mathcal{F}(\omega_t)-\frac{\eta}{4}\|\widehat G_t\|^2+\eta\Delta^2.
\end{equation*}
Summing over $t=0,\dots,T-1$, dividing by $\eta T/4$, and taking the minimum over $t$,
\begin{equation*}
    \min_{0\leq t<T}\|\widehat G_t\|^2\leq\frac{4\{\mathcal{F}(\omega_0)-F_{\inf}\}}{\eta T}+4\Delta^2,
\end{equation*}
so that, using $\sqrt{a+b}\leq\sqrt a+\sqrt b$, $\min_{0\leq t<T}\|\widehat G_t\|\leq2\bigl(\{\mathcal{F}(\omega_0)-F_{\inf}\}/(\eta T)\bigr)^{1/2}+2\Delta$. Combining with \eqref{eq:gradmap-hat-vs-pop} at the minimizing step,
\begin{equation*}
    \min_{0\leq t<T}\|G_t\|\leq2\left(\frac{\mathcal{F}(\omega_0)-F_{\inf}}{\eta T}\right)^{1/2}+3\Delta.
\end{equation*}
Taking expectations over both samples and any fitting randomness, with $\mathbb E[\Delta]=\mathfrak{s}<\infty$, gives the stated bound with universal implicit constant $\max(2,3)$. Only the first moment of $\Delta$ is needed, since the square root is taken before expectation.
\end{proof}

\subsection{Hypergradient bias identities}
\begin{proof}[Proof of \Cref{prop:plugin-bias}]
Fix $\omega\in\Omega$, $k\in\{1,\dots,d\}$ and $h,j_k\in L^2(P_X)$. Under \Cref{assum:eif-integrability}, $\|Y\|\le A$ almost surely, so by Cauchy--Schwarz
\begin{equation*}
  \mathbb{E}_P\bigl|\bigl\langle h(X)-Y,\,j_k(X)\bigr\rangle\bigr|\;\le\;\bigl(\|h\|_{L^2(P_X)}+A\bigr)\,\|j_k\|_{L^2(P_X)}\;<\;\infty ;
\end{equation*}
all expectations below are therefore absolutely convergent and the tower property applies. Since $h$ and $j_k$ are $\sigma(X)$-measurable and $\mathbb{E}_P[Y\mid X]=m^\star_P(X)$,
\begin{equation}
\label{eq:pi-mean}
  \mathbb{E}_P\Bigl[\bigl[\widehat{\nabla\mathcal{F}}^{\mathrm{PI}}(\omega)\bigr]_k\Bigr]=\mathbb{E}_P\bigl[\varphi^{\mathrm{PI}}(h,j_k)\bigr]=\bigl\langle h-m^\star_P,\;j_k\bigr\rangle_{L^2(P_X)} .
\end{equation}
Evaluating \Cref{eq:pi-mean} at $h=h^\star_{\omega,P}$ and $j_k=j^\star_{k,\omega,P}$ and comparing with \Cref{eq:PI},
\begin{equation}
\label{eq:ref-form}
  \partial_{\omega_k}\mathcal{F}(\omega)=\bigl\langle h^\star_{\omega,P}-m^\star_P,\;j^\star_{k,\omega,P}\bigr\rangle_{L^2(P_X)} .
\end{equation}
Adding and subtracting $h^\star_{\omega,P}$ in the first argument of \Cref{eq:pi-mean} and $j^\star_{k,\omega,P}$ in the second, and expanding by bilinearity of $\langle\cdot,\cdot\rangle_{L^2(P_X)}$,
\begin{align*}
  \bigl\langle h-m^\star_P,\;j_k\bigr\rangle_{L^2(P_X)}
  &=\bigl\langle \bigl(h-h^\star_{\omega,P}\bigr)+\bigl(h^\star_{\omega,P}-m^\star_P\bigr),\;
     j^\star_{k,\omega,P}+\bigl(j_k-j^\star_{k,\omega,P}\bigr)\bigr\rangle_{L^2(P_X)}\\[2pt]
  &=\bigl\langle h^\star_{\omega,P}-m^\star_P,\;j^\star_{k,\omega,P}\bigr\rangle_{L^2(P_X)}
   +\bigl\langle h-h^\star_{\omega,P},\;j^\star_{k,\omega,P}\bigr\rangle_{L^2(P_X)}\\[2pt]
  &\qquad+\bigl\langle h^\star_{\omega,P}-m^\star_P,\;j_k-j^\star_{k,\omega,P}\bigr\rangle_{L^2(P_X)}\\[2pt]
  &\qquad+\bigl\langle h-h^\star_{\omega,P},\;j_k-j^\star_{k,\omega,P}\bigr\rangle_{L^2(P_X)} .
\end{align*}
The first of the four terms is $\partial_{\omega_k}\mathcal{F}(\omega)$ by \Cref{eq:ref-form}. Subtracting it from \Cref{eq:pi-mean},
\begin{align*}
  \mathbb{E}_P\Bigl[\bigl[\widehat{\nabla\mathcal{F}}^{\mathrm{PI}}(\omega)\bigr]_k\Bigr]
  -\partial_{\omega_k}\mathcal{F}(\omega)
  &=\underbrace{\bigl\langle h-h^\star_{\omega,P},\;j^\star_{k,\omega,P}\bigr\rangle_{L^2(P_X)}}_{\mathrm{(I)}}\\[2pt]
  &\qquad+\underbrace{\bigl\langle h^\star_{\omega,P}-m^\star_P,\;j_k-j^\star_{k,\omega,P}\bigr\rangle_{L^2(P_X)}}_{\mathrm{(II)}}\\[2pt]
  &\qquad+\underbrace{\bigl\langle h-h^\star_{\omega,P},\;j_k-j^\star_{k,\omega,P}\bigr\rangle_{L^2(P_X)}}_{\mathrm{(III)}},
\end{align*}
which is \Cref{eq:plugin-bias}.
\end{proof}

\begin{proof}[Proof of \Cref{thm:mixed-bias}]
Fix $\omega\in\Omega$, $k\in\{1,\dots,d\}$ and $h,j_k,m\in L^2(P_X)$. Expanding the second inner product in the debiased score of \Cref{eq:debiased-score},
\begin{align*}
  \varphi_{k,\omega}(h,j_k,m)
  &=\bigl\langle g_\omega(Z)-Y,\;j_k(X)\bigr\rangle\\
  &\qquad+\bigl\langle \partial_{\omega_k}g_\omega(Z),\;h(X)-m(X)\bigr\rangle
  -\bigl\langle j_k(X),\;h(X)-m(X)\bigr\rangle .
\end{align*}
Under \Cref{assum:eif-integrability}, $\|g_\omega(Z)-Y\|\le A+B$ and $\|\partial_{\omega_k}g_\omega(Z)\|\le\|D_\omega g_\omega(Z)\|_F\le D$ almost surely, so in absolute value the three terms have expectations bounded by $(A+B)\|j_k\|_{L^2(P_X)}$, $D\,\|h-m\|_{L^2(P_X)}$ and $\|j_k\|_{L^2(P_X)}\|h-m\|_{L^2(P_X)}$ respectively; all expectations below are therefore absolutely convergent and the tower property applies. Since $h$, $j_k$ and $m$ are $\sigma(X)$-measurable, and $\mathbb{E}_P[g_\omega(Z)\mid X]=h^\star_{\omega,P}(X)$, $\mathbb{E}_P[Y\mid X]=m^\star_P(X)$ and $\mathbb{E}_P[\partial_{\omega_k}g_\omega(Z)\mid X]=j^\star_{k,\omega,P}(X)$ by \Cref{eq:jacobian}, taking $\mathbb{E}_P$ term by term gives
\begin{align}
  \mathbb{E}_P\bigl[\varphi_{k,\omega}(h,j_k,m)\bigr]
  &=\bigl\langle h^\star_{\omega,P}-m^\star_P,\;j_k\bigr\rangle_{L^2(P_X)}
   +\bigl\langle j^\star_{k,\omega,P},\;h-m\bigr\rangle_{L^2(P_X)}\nonumber\\
  &\qquad-\bigl\langle j_k,\;h-m\bigr\rangle_{L^2(P_X)}\nonumber\\[2pt]
  &=\bigl\langle h^\star_{\omega,P}-m^\star_P,\;j_k\bigr\rangle_{L^2(P_X)}
   -\bigl\langle j_k-j^\star_{k,\omega,P},\;h-m\bigr\rangle_{L^2(P_X)} .
  \label{eq:obi-mean}
\end{align}
Evaluating \Cref{eq:obi-mean} at $h=h^\star_{\omega,P}$, $j_k=j^\star_{k,\omega,P}$ and $m=m^\star_P$, the second term vanishes and comparison with \Cref{eq:adj_PI} gives
\begin{equation}
\label{eq:ref-form-obi}
  \partial_{\omega_k}\mathcal{F}(\omega)=\bigl\langle h^\star_{\omega,P}-m^\star_P,\;j^\star_{k,\omega,P}\bigr\rangle_{L^2(P_X)} .
\end{equation}
Subtracting \Cref{eq:ref-form-obi} from \Cref{eq:obi-mean} and using bilinearity in the first term,
\begin{align*}
  \mathbb{E}_P\Bigl[\bigl[\widehat{\nabla\mathcal{F}}^{\mathrm{OBi}}&(\omega)\bigr]_k\Bigr]
  -\partial_{\omega_k}\mathcal{F}(\omega)\\[2pt]
  &=\bigl\langle h^\star_{\omega,P}-m^\star_P,\;j_k-j^\star_{k,\omega,P}\bigr\rangle_{L^2(P_X)}
   -\bigl\langle j_k-j^\star_{k,\omega,P},\;h-m\bigr\rangle_{L^2(P_X)}\\[2pt]
  &=\bigl\langle j_k-j^\star_{k,\omega,P},\;
      \bigl(h^\star_{\omega,P}-m^\star_P\bigr)-\bigl(h-m\bigr)\bigr\rangle_{L^2(P_X)} .
\end{align*}
Finally $\bigl(h^\star_{\omega,P}-m^\star_P\bigr)-\bigl(h-m\bigr)=\bigl(m-m^\star_P\bigr)-\bigl(h-h^\star_{\omega,P}\bigr)$, so
\begin{equation*}
  \mathbb{E}_P\Bigl[\bigl[\widehat{\nabla\mathcal{F}}^{\mathrm{OBi}}(\omega)\bigr]_k\Bigr]-\partial_{\omega_k}\mathcal{F}(\omega)=\bigl\langle j_k-j^\star_{k,\omega,P},\;\bigl(m-m^\star_P\bigr)-\bigl(h-h^\star_{\omega,P}\bigr)\bigr\rangle_{L^2(P_X)},
\end{equation*}
which is \Cref{eq:mixed-bias}.
\end{proof}

\subsection{Auxiliary empirical-process lemmas}
\label{sec:appendix_auxiliary_lemmas}
Throughout this subsection, write $\|f\|_{Q,2}\coloneqq(\int\|f\|^2\,dQ)^{1/2}$ for the $L^2(Q)$ norm of a function $f$, and $N(\delta,\mathcal{C},L^2(Q))$ for the $\delta$-covering number of $\mathcal{C}$ with respect to the $L^2(Q)$ metric, \emph{i.e.}, the minimum number of $L^2(Q)$-balls of radius $\delta$ needed to cover $\mathcal{C}$. Its logarithm $\log N(\delta,\mathcal{C},L^2(Q))$ is called the metric entropy of $\mathcal{C}$. For a class $\mathcal{C}$ with envelope $C$ satisfying $\sup_{f\in\mathcal{C}}\|f\|\leq C$ pointwise, define
\begin{equation*}
    \bar H(\varepsilon;\mathcal{C},C)\coloneqq\sup_Q\log N(\varepsilon\|C\|_{Q,2},\mathcal{C},L^2(Q)),\qquad 0<\varepsilon\leq1,
\end{equation*}
where the supremum is over all finitely supported probability measures $Q$.

\begin{lemma}[Entropy of sums and differences]
\label{lem:entropy-sums}
Let $\mathcal C_1$ and $\mathcal C_2$ be classes of measurable functions taking values in the same finite-dimensional normed space, with nonnegative envelopes $C_1$ and $C_2$. Then $C_1+C_2$ is an envelope for $\mathcal C_1\pm\mathcal C_2=\{f_1\pm f_2:f_1\in\mathcal C_1,\ f_2\in\mathcal C_2\}$, and
\begin{equation*}
    \bar H(\varepsilon;\mathcal C_1\pm\mathcal C_2,C_1+C_2)\le\bar H(\varepsilon/2;\mathcal C_1,C_1)+\bar H(\varepsilon/2;\mathcal C_2,C_2).
\end{equation*}
\end{lemma}

\begin{proof}
Fix $Q$. Let $\{f_{1,i}\}_{i=1}^{N_1}$ and $\{f_{2,j}\}_{j=1}^{N_2}$ be minimal $L^2(Q)$-covers of $\mathcal{C}_1$ and $\mathcal{C}_2$ at radii $(\varepsilon/2)\|C_1\|_{Q,2}$ and $(\varepsilon/2)\|C_2\|_{Q,2}$, respectively, so that
\begin{equation*}
    N_1=N\!\left(\tfrac{\varepsilon}{2}\|C_1\|_{Q,2},\mathcal{C}_1,L^2(Q)\right),\qquad N_2=N\!\left(\tfrac{\varepsilon}{2}\|C_2\|_{Q,2},\mathcal{C}_2,L^2(Q)\right).
\end{equation*}
For any $f=f_1\pm f_2\in\mathcal{C}_1\pm\mathcal{C}_2$, find the closest cover elements $f_{1,i}$ and $f_{2,j}$. By the triangle inequality,
\begin{equation*}
    \|f_1\pm f_2-(f_{1,i}\pm f_{2,j})\|_{Q,2}\leq\|f_1-f_{1,i}\|_{Q,2}+\|f_2-f_{2,j}\|_{Q,2}\leq\frac{\varepsilon}{2}\|C_1\|_{Q,2}+\frac{\varepsilon}{2}\|C_2\|_{Q,2}.
\end{equation*}
Since $C_1$ and $C_2$ are nonnegative, $C_1+C_2\geq C_1$ and $C_1+C_2\geq C_2$ pointwise, so $\|C_1+C_2\|_{Q,2}\geq\|C_1\|_{Q,2}$ and $\|C_1+C_2\|_{Q,2}\geq\|C_2\|_{Q,2}$. Adding these two inequalities gives
\begin{equation*}
    \frac{\varepsilon}{2}\|C_1\|_{Q,2}+\frac{\varepsilon}{2}\|C_2\|_{Q,2}\leq\varepsilon\|C_1+C_2\|_{Q,2},
\end{equation*}
so $\{f_{1,i}\pm f_{2,j}\}$ covers $\mathcal{C}_1\pm\mathcal{C}_2$ at radius $\varepsilon\|C_1+C_2\|_{Q,2}$, confirming that $C_1+C_2$ is an envelope for $\mathcal{C}_1\pm\mathcal{C}_2$. Since this cover has at most $N_1\cdot N_2$ elements,
\begin{equation*}
    \log N\!\left(\varepsilon\|C_1+C_2\|_{Q,2},\mathcal{C}_1\pm\mathcal{C}_2,L^2(Q)\right)\leq\log N_1+\log N_2.
\end{equation*}
Taking the supremum over all finitely supported $Q$ on both sides gives
\begin{equation*}
    \bar H(\varepsilon;\mathcal{C}_1\pm\mathcal{C}_2,C_1+C_2)\leq\bar H(\varepsilon/2;\mathcal{C}_1,C_1)+\bar H(\varepsilon/2;\mathcal{C}_2,C_2).
\end{equation*}
\end{proof}

\begin{lemma}[Entropy of bounded products]
\label{lem:bounded-product-entropy}
Let $\mathcal{U}$ and $\mathcal{V}$ be vector-valued function classes such that, for finite constants $\bar U,\bar V$, $\sup_{u\in\mathcal{U}}\|u(O)\|\leq\bar U$ and $\sup_{v\in\mathcal{V}}\|v(O)\|\leq\bar V$ for all $O$. Define
\begin{equation*}
    \mathcal{U}\odot\mathcal{V}\coloneqq\{O\mapsto\langle u(O),v(O)\rangle:u\in\mathcal{U},\ v\in\mathcal{V}\}.
\end{equation*}
Then $\bar U\bar V$ is an envelope for $\mathcal{U}\odot\mathcal{V}$, and for every $0<\varepsilon\leq1$,
\begin{equation*}
\begin{aligned}
    \sup_Q\log N(\varepsilon\bar U\bar V,\mathcal{U}\odot\mathcal{V},L^2(Q))&\leq\sup_Q\log N\!\left(\tfrac{\varepsilon\bar U}{2},\mathcal{U},L^2(Q)\right)+\sup_Q\log N\!\left(\tfrac{\varepsilon\bar V}{2},\mathcal{V},L^2(Q)\right).
\end{aligned}
\end{equation*}
\end{lemma}

\begin{proof}
The envelope bound follows from Cauchy-Schwarz: for all $O$,
\begin{equation*}
    |\langle u(O),v(O)\rangle|\leq\|u(O)\|\|v(O)\|\leq\bar U\bar V.
\end{equation*}
For the entropy bound, fix $Q$ and let $\{u_i\}_{i=1}^{N_1}$ and $\{v_j\}_{j=1}^{N_2}$ be minimal $L^2(Q)$-covers of $\mathcal{U}$ and $\mathcal{V}$ at radii $\varepsilon\bar U/2$ and $\varepsilon\bar V/2$ respectively. For any $u\in\mathcal{U}$ and $v\in\mathcal{V}$, find the closest cover elements $u_i$ and $v_j$. By the triangle inequality and the pointwise bounds,
\begin{equation*}
\begin{aligned}
    \|\langle u,v\rangle-\langle u_i,v_j\rangle\|_{Q,2}&\leq\|\langle u-u_i,v\rangle\|_{Q,2}+\|\langle u_i,v-v_j\rangle\|_{Q,2}\\
    &\leq\bar V\|u-u_i\|_{Q,2}+\bar U\|v-v_j\|_{Q,2}\\
    &\leq\bar V\cdot\frac{\varepsilon\bar U}{2}+\bar U\cdot\frac{\varepsilon\bar V}{2}=\varepsilon\bar U\bar V.
\end{aligned}
\end{equation*}
So $\{u_i\odot v_j\}$ covers $\mathcal{U}\odot\mathcal{V}$ at radius $\varepsilon\bar U\bar V$ with at most $N_1\cdot N_2$ elements. Taking logarithms and the supremum over $Q$ gives the result.
\end{proof}

\begin{lemma}[Entropy of block linear combinations]
\label{lem:block-linear-combination-entropy}
Let $\mathcal{A}$ be a class of block-vector-valued functions $a=(a_1,\dots,a_d)$ with values in $(\mathbb{R}^q)^d$, equipped with the block norm $\|a(O)\|_{2,d}\coloneqq(\sum_{k=1}^d\|a_k(O)\|^2)^{1/2}$, and suppose $\sup_{a\in\mathcal{A}}\|a(O)\|_{2,d}\leq\bar A$ for all $O$. Define
\begin{equation*}
    \mathscr{S}(\mathcal{A})\coloneqq\left\{O\mapsto\sum_{k=1}^d u_k a_k(O):a\in\mathcal{A},\ u\in \mathbb{S}^{d-1}\right\}.
\end{equation*}
Then $\bar A$ is an envelope for $\mathscr{S}(\mathcal{A})$, and for every $0<\varepsilon\leq1$,
\begin{equation*}
    \sup_Q\log N(\varepsilon\bar A,\mathscr{S}(\mathcal{A}),L^2(Q))\leq d\log(1+4/\varepsilon)+\sup_Q\log N(\varepsilon\bar A/2,\mathcal{A},L^2(Q)).
\end{equation*}
\end{lemma}

\begin{proof}
The envelope bound follows from Cauchy-Schwarz: for all $O$ and $u\in \mathbb{S}^{d-1}$,
\begin{equation*}
    \left|\sum_{k=1}^d u_k a_k(O)\right|\leq\|u\|_2\|a(O)\|_{2,d}\leq\bar A.
\end{equation*}
For the entropy bound, fix $Q$. Let $\{a_i\}_{i=1}^{N_1}$ be a minimal $L^2(Q)$-cover of $\mathcal{A}$ at radius $\varepsilon\bar A/2$ in the block norm, and let $\{u_\ell\}_{\ell=1}^{N_2}$ be an $(\varepsilon/2)$-cover of $\mathbb{S}^{d-1}$ in Euclidean norm, which exists with cardinality $N_2\leq(1+4/\varepsilon)^d$ by a standard volumetric argument \citep{vanderVaart1996weak}. For any $a\in\mathcal{A}$ and $u\in\mathbb{S}^{d-1}$, choose the closest cover elements $a_i$ and $u_\ell$. By the triangle inequality,
\begin{equation*}
\begin{aligned}
    \left\|\sum_{k=1}^d u_k a_k-\sum_{k=1}^d(u_\ell)_k(a_i)_k\right\|_{Q,2}&\leq\left\|\sum_{k=1}^d u_k(a_k-(a_i)_k)\right\|_{Q,2}+\left\|\sum_{k=1}^d(u_k-(u_\ell)_k)(a_i)_k\right\|_{Q,2}\\
    &\leq\|u\|_2\|a-a_i\|_{Q,2,d}+\|u-u_\ell\|_2\bar A\\
    &\leq\frac{\varepsilon\bar A}{2}+\frac{\varepsilon}{2}\bar A=\varepsilon\bar A.
\end{aligned}
\end{equation*}
So $\{u_\ell\odot a_i\}$ covers $\mathscr{S}(\mathcal{A})$ at radius $\varepsilon\bar A$ with at most $N_1\cdot N_2$ elements. Taking logarithms and the supremum over $Q$ gives
\begin{align*}
    \sup_Q\log N(\varepsilon\bar A,\mathscr{S}(\mathcal{A}),L^2(Q))&\leq\log N_2+\sup_Q\log N_1\\
    &\leq d\log(1+4/\varepsilon)+\sup_Q\log N(\varepsilon\bar A/2,\mathcal{A},L^2(Q)).
\end{align*}
\end{proof}

Let $S_1$ and $S_2$ be independent folds of $n$ i.i.d. observations $O=(X,Y,Z)\sim P$. The nuisances are fitted on $S_1$ and the hypergradient is evaluated on $S_2$. For a class $\mathcal{G}$ of real-valued measurable functions, write $\operatorname{star}(\mathcal{G})\coloneqq\{\alpha g:\alpha\in[0,1],\ g\in\mathcal{G}\}$ for its star hull around zero. For $\delta>0$, define the conditional local Rademacher complexity
\begin{equation*}
    \mathfrak{R}_n(\mathcal{G},\delta\mid S_1)\coloneqq\mathbb{E}_{\varepsilon,S_2}\left[\sup_{\substack{g\in\mathcal{G}\\
    \|g\|_{L^2(P)}\leq\delta}}\left|\frac{1}{n}\sum_{i\in S_2}\varepsilon_i g(O_i)\right|\,\middle|\,S_1\right],
\end{equation*}
where $\varepsilon_1,\ldots,\varepsilon_n$ are i.i.d.\ Rademacher random variables independent of $S_2$, and the expectation is over both $\varepsilon$ and $S_2$ conditionally on $S_1$.

\begin{lemma}[Localized OSL maximal inequality]
\label{lem:localized-osl-maximal}
Let $\mathcal G$ be a pointwise measurable real-valued class with envelope $G<\infty$ and $\sup_{g\in\mathcal G}\|g\|_{L_2(P)}\le \rho$. Define $\overline{\mathcal G}^{0}=\{g/G:g\in\mathcal G\}\cup\{0\}$. Suppose $\delta_n$ satisfies
\begin{equation*}
    \mathfrak R_n\left(\operatorname{star}(\overline{\mathcal G}^{0}),\delta_n\right)\le\delta_n^2.
\end{equation*}
Then there exist universal constants $C,c<\infty$ such that
\begin{equation*}
    \mathbb{E}\left[\sup_{g\in\mathcal G}|(P_n-P)g|\right]\le C\left\{\delta_n\rho+G\delta_n^2+G e^{-c n\delta_n^2}\right\}.
\end{equation*}
The same statement holds conditionally on $S_1$ for an $S_1$-measurable class and envelope, with the expectation taken over the evaluation fold.
\end{lemma}

\begin{proof}
Apply Lemma~14 in \citep{foster2023orthogonal} to $\overline{\mathcal G}^{0}$ with identity loss $L_f=f$ and comparator $f^\star=0$. The identity loss is linear, so the lower bound on $\delta_n^2$ in Lemma~14 is not required. Here the dimension parameter in Lemma~14 is one, because $\mathcal G$ is real-valued; the outer-gradient dimension has already been absorbed by scalarization over $u\in S^{d-1}$. Then, with probability at least $1-Ce^{-cn\delta_n^2}$,
\begin{equation*}
    |(P_n-P)\bar g|\le C\delta_n(\|\bar g\|_{L_2(P)}+\delta_n),\qquad\forall \bar g\in \overline{\mathcal G}^{0}.
\end{equation*}
Multiplying by $G$ and using $\|g\|_{L^2(P)}\leq\rho$ gives
\begin{equation*}
    |(P_n-P)g|\leq C\{\delta_n\|g\|_{L^2(P)}+G\delta_n^2\}\leq C\{\delta_n\rho+G\delta_n^2\}
\end{equation*}
on the good event. On the complement, which has probability at most $Ce^{-cn\delta_n^2}$, the supremum is bounded by $2G$ since $|(P_n-P)g|\leq 2\|g\|_\infty\leq 2G$ for all $g\in\mathcal{G}$. Taking expectations and combining gives the result.
\end{proof}

We also use the following standard maximal inequality for classes with finite uniform entropy integral \citep[Chapter~2]{vanderVaart1996weak}; see also \citep[Chapter~8]{kosorok2008introduction}.

\begin{lemma}[Uniform entropy maximal inequality]
\label{lem:entropy-maximal}
Let $\mathcal{F}$ be a pointwise measurable class of real-valued functions with envelope $F\in L^2(P)$. Define the uniform entropy integral
\begin{equation*}
    J(1,\mathcal{F},F)\coloneqq\int_0^1\sup_Q\sqrt{1+\log N\bigl(\varepsilon\|F\|_{Q,2},\mathcal{F},L^2(Q)\bigr)}\,d\varepsilon,
\end{equation*}
where the supremum is over all finitely supported probability measures $Q$. If $J(1,\mathcal{F},F)<\infty$, then there exists a universal constant $0<C<\infty$ such that
\begin{equation*}
    \mathbb{E}\left[\sup_{f\in\mathcal{F}}|(P_n-P)f|\right]\leq\frac{C}{\sqrt{n}}J(1,\mathcal{F},F)\|F\|_{L^2(P)}.
\end{equation*}
\end{lemma}

\begin{corollary}[Polynomial entropy maximal inequality]
\label{cor:polynomial-entropy-maximal}
Suppose that, for some $p\in(0,2)$ and $K<\infty$,
\begin{equation*}
    \sup_Q\log N\bigl(\varepsilon\|F\|_{Q,2},\mathcal{F},L^2(Q)\bigr)\leq K\varepsilon^{-p},\qquad 0<\varepsilon\leq1.
\end{equation*}
Then there exists a constant $C_p<\infty$ depending only on $p$ such that
\begin{equation*}
    \mathbb{E}\left[\sup_{f\in\mathcal{F}}|(P_n-P)f|\right]\leq\frac{C_p}{\sqrt{n}}\|F\|_{L^2(P)}\sqrt{1+K}.
\end{equation*}
\end{corollary}

\begin{proof}
Since $\varepsilon\leq1$ implies $1\leq\varepsilon^{-p}$, we have $1+K\varepsilon^{-p}\leq(1+K)\varepsilon^{-p}$, and therefore
\begin{equation*}
    J(1,\mathcal{F},F)\leq\int_0^1\sqrt{1+K\varepsilon^{-p}}\,d\varepsilon\leq\sqrt{1+K}\int_0^1\varepsilon^{-p/2}\,d\varepsilon=\frac{2\sqrt{1+K}}{2-p}<\infty,
\end{equation*}
where finiteness uses $p<2$. Applying Lemma~\ref{lem:entropy-maximal} with $C_p\coloneqq 2C/(2-p)$ gives the claim.
\end{proof}

\subsection{Conditional maximal inequality}
\label{sec:appendix_quadratic_maximal}
To control the nuisance empirical process, we require the learned nuisance paths to lie in well-controlled function classes, which we now formalize.

\begin{assum}[Training-conditional nuisance classes]
\label{ass:maximal-bounded-nuisance-classes}
Conditionally on $S_1$, there are pointwise measurable random classes $\mathcal{H}_n\subset\LL$, $\mathcal{J}_n\subset\LL^d$, and $\mathcal{M}_n\subset\LL$ containing the learned nuisance paths $\{\hat h_\omega:\omega\in\Omega\}\subset\mathcal{H}_n$, $\{\hat j_\omega:\omega\in\Omega\}\subset\mathcal{J}_n$, and $\hat m\in\mathcal{M}_n$ almost surely, with $S_1$-measurable envelopes $\bar H_n,\bar J_n,\bar M_n<\infty$ satisfying
\begin{equation*}
    \sup_{h\in\mathcal{H}_n}\|h(X)\|\leq\bar H_n,\qquad\sup_{j\in\mathcal{J}_n}\sum_{k=1}^d\|j_k(X)\|^2\leq\bar J_n^2,\qquad\sup_{m\in\mathcal{M}_n}\|m(X)\|\leq\bar M_n.
\end{equation*}
\end{assum}
Kernel ridge regression, sieve estimators, and norm-controlled neural networks satisfy assumptions of this form under the usual envelope and entropy controls; representative sufficient conditions are collected in Appendix~\ref{sec:nuisance-learning-discussion}. The entropy of both the population classes and the training-conditional nuisance classes is controlled by the following two assumptions. Both assumptions below use the following notation: for a function class $\mathcal{C}$ with envelope $E_{\mathcal{C}}$ satisfying $\sup_{f\in\mathcal{C}}\|f\|\leq E_{\mathcal{C}}$ pointwise, $N(\varepsilon E_{\mathcal{C}},\mathcal{C},L^2(Q))$ denotes the $\varepsilon E_{\mathcal{C}}$-covering number of $\mathcal{C}$ in $L^2(Q)$, and $\sup_Q$ is taken over all finitely supported probability measures $Q$ on $\mathcal{X}\times\mathcal{Y}\times\mathcal{Z}$.

\begin{assum}[Entropy of population classes]
\label{ass:maximal-population-entropy}
Define the population function classes indexed by $\omega\in\Omega$: the structural map class $\mathcal{G}_\Omega=\{g_\omega:\omega\in\Omega\}$, its gradient class $\dot{\mathcal{G}}_\Omega=\{(\partial_{\omega_1}g_\omega,\ldots,\partial_{\omega_d}g_\omega):\omega\in\Omega\}$, the oracle inner solution class $\mathcal{H}^\star=\{h^\star_{\omega,P}:\omega\in\Omega\}$, and the oracle Jacobian class $\mathcal{J}^\star=\{j^\star_{\omega,P}:\omega\in\Omega\}$. There exist $p\in(0,2)$ and finite constants $K_{\mathcal{G}_\Omega},K_{\dot{\mathcal{G}}_\Omega},K_{\mathcal{H}^\star},K_{\mathcal{J}^\star}$ such that, for every $0<\varepsilon\leq1$,
\begin{equation*}
    \sup_Q\log N(\varepsilon E_{\mathcal{C}},\mathcal{C},L^2(Q))\leq K_{\mathcal{C}}\varepsilon^{-p},\qquad\mathcal{C}\in\{\mathcal{G}_\Omega,\dot{\mathcal{G}}_\Omega,\mathcal{H}^\star,\mathcal{J}^\star\},
\end{equation*}
with envelopes $E_{\mathcal{G}_\Omega}=E_{\mathcal{H}^\star}=B$ and $E_{\dot{\mathcal{G}}_\Omega}=E_{\mathcal{J}^\star}=D$.
\end{assum}
In words, this bounds how complex the population function classes are allowed to be as $\omega$ ranges over $\Omega$: the logarithm of the number of $\varepsilon$-balls needed to cover each class grows at most polynomially in $1/\varepsilon$, with exponent $p<2$. This is a standard complexity condition from empirical process theory, and it is mild in the sense that it is exactly what holds for the usual nonparametric smoothness classes. For example, standard $s$-Hölder classes on a $d_X$-dimensional domain satisfy such an entropy bound with $p=d_X/s$; hence the condition $p<2$ corresponds to the usual smoothness requirement $s>d_X/2$ in this case.

\begin{assum}[Entropy of training-conditional nuisance classes]
\label{ass:maximal-nuisance-entropy}
Conditionally on $S_1$, there are finite $S_1$-measurable constants $K_{\mathcal{H}_n},K_{\mathcal{J}_n},K_{\mathcal{M}_n}$ such that the learned nuisance classes $\mathcal{H}_n$, $\mathcal{J}_n$, $\mathcal{M}_n$ satisfy the same polynomial entropy bound with exponent $p\in(0,2)$: for every $0<\varepsilon\leq1$,
\begin{equation*}
    \sup_Q\log N(\varepsilon E_{\mathcal{C}},\mathcal{C},L^2(Q))\leq K_{\mathcal{C}}\varepsilon^{-p},\qquad\mathcal{C}\in\{\mathcal{H}_n,\mathcal{J}_n,\mathcal{M}_n\},
\end{equation*}
with envelopes $E_{\mathcal{H}_n}=\bar H_n$, $E_{\mathcal{J}_n}=\bar J_n$, and $E_{\mathcal{M}_n}=\bar M_n$.
\end{assum}
The condition $p\in(0,2)$ ensures that the corresponding entropy integral is finite, providing the uniform empirical-process control required below \citep{vanderVaart1996weak,kosorok2008introduction}. These entropy and envelope conditions are sufficient rather than necessary. They are satisfied by a range of standard nuisance classes, including bounded Hölder classes with smoothness $s>d_X/2$, fixed-dimensional sieve classes, kernel ridge regression classes under suitable bounded-kernel assumptions, and appropriately norm-controlled neural-network classes.

Denote by $\mathcal{E}_{\rm tr}(r_h,r_j,r_m)$ the event where $\sup_{\omega\in\Omega}\|\hat h_\omega-h^\star_{\omega,P}\|_{\LL}\leq r_h$, $\sup_{\omega\in\Omega}\|\hat j_\omega-j^\star_{\omega,P}\|_{\LL^d}\leq r_j$, and $\|\hat m-m^\star_P\|_{\LL}\leq r_m$. 

\begin{theorem}[Expected maximal inequality]
\label{thm:quadratic-maximal-inequality}
Under \Cref{assum:eif-integrability,ass:maximal-bounded-nuisance-classes,ass:maximal-population-entropy,ass:maximal-nuisance-entropy}, there exist constants $C>0$ depending only on $p$ and $K>0$ depending only on the population entropy constants such that, writing $\mathfrak{s}_N\coloneqq\mathbb{E}\left[\sup_{\omega\in\Omega}\|\widehat{\nabla\mathcal F}^{\mathrm{OBi}}(\omega)-\nabla\mathcal{F}(\omega)\|\,\middle|\,S_1\right]$ for the expected uniform error, on $\mathcal{E}_{\rm tr}(r_h,r_j,r_m)$,
\begin{equation}
\label{eq:main-uniform-bound}
    \mathfrak{s}_N\leq\frac{C}{\sqrt{n}}\,(A+B)D\,\sqrt{1+K}+\mathfrak{C}_{B,n}(r_h,r_j,r_m)+r_j(r_h+r_m),
\end{equation}
where $\mathfrak{C}_{B,n}(r_h,r_j,r_m)\coloneqq\mathbb{E}\left[\sup_{f\in\mathcal{F}_{B,n}^{\rm loc}(r_h,r_j,r_m)}\bigl|\bigl(\mathbb{E}_{P_{n,2}}-\mathbb{E}_P\bigr)[f]\bigr|\,\middle|\,S_1\right]$, with localized nuisance-difference class

\begin{align*}
    \mathcal{F}_{B,n}^{\rm loc}(r_h,r_j,r_m)\coloneqq\bigl\{&u^\top\left[\varphi_\omega(\cdot;h,j,m)-\varphi_\omega(\cdot;h^\star_{\omega,P},j^\star_{\omega,P},m^\star_P)\right]:\\
    &\omega\in\Omega,\ u\in \mathbb{S}^{d-1},\ h\in\mathcal{H}_n,\ j\in\mathcal{J}_n,\ m\in\mathcal{M}_n,\notag\\
    &\|h-h^\star_{\omega,P}\|_{\LL}\leq r_h, \|j-j^\star_{\omega,P}\|_{\LL^d}\leq r_j, \|m-m^\star_P\|_{\LL}\leq r_m\bigr\},
\end{align*}
which localizes the nuisance functions to a ball of radius $(r_h,r_j,r_m)$ around the oracle, conditionally on $S_1$.
\end{theorem}
Appendix~\ref{sec:pr_cor} bounds $\mathfrak{C}_{B,n}$ via localized Rademacher critical radii \citep{bartlett2005local,foster2023orthogonal}.

\begin{remark}[Only the nuisance cross-products matter]
\label{rem:nuisance-cross-products}
The population-bias term in \eqref{eq:main-uniform-bound} is $r_j(r_h+r_m)$: only products of the Jacobian error with the conditional-mean errors enter. The nuisance rates need not be balanced. An $O(n^{-1/2})$ conditional expected uniform gradient rate follows if both $r_j(r_h+r_m)$ and $\mathfrak C_{B,n}(r_h,r_j,r_m)$ are $O(n^{-1/2})$; making both $o(n^{-1/2})$ renders the nuisance contributions negligible relative to the oracle first-order term. Thus a slower Jacobian rate can be offset by faster conditional-mean rates. Under the common rate $r_h,r_j,r_m=O(n^{-1/(2+p)})$ in Proposition~\ref{cor:localized-polynomial-rate}, both bias cross-products are $O(n^{-2/(2+p)})=o(n^{-1/2})$ for $p<2$.
\end{remark}

\begin{proof}[Proof of \cref{thm:quadratic-maximal-inequality}]
The proof proceeds in the following 6 steps.
\paragraph{Step 1: population identities.}
Under the quadratic specialization, the unique population inner minimizer is $h^\star_{\omega,P}(X)=\mathbb{E}_P[g_\omega(Z)\mid X]$ and $j^\star_{k,\omega,P}(X)=\mathbb{E}_P[\partial_{\omega_k}g_\omega(Z)\mid X]$ for $k=1,\dots,d$. The outer objective and its gradient are
\begin{equation*}
    \mathcal{F}(\omega)=\frac{1}{2}\mathbb{E}_P\left[\|Y-h^\star_{\omega,P}(X)\|^2\right],\qquad\partial_{\omega_k}\mathcal F(\omega)=\mathbb{E}_P[\langle h^\star_{\omega,P}(X)-Y,j^\star_{k,\omega,P}(X)\rangle].
\end{equation*}
By the conditional Jensen inequality and \Cref{assum:eif-integrability},
\begin{equation*}
    \|h^\star_{\omega,P}(X)\|\leq B,\qquad\|m^\star_P(X)\|\leq A,\qquad\|j^\star_{\omega,P}(X)\|_{2,d}\leq D,
\end{equation*}
where $\|a\|_{2,d}\coloneqq(\sum_{k=1}^d\|a_k\|^2)^{1/2}$ for $a=(a_1,\ldots,a_d)\in(\mathbb{R}^q)^d$.

\paragraph{Step 2: exact bias identity.}
At the oracle nuisances, $\mathbb{E}_P[\varphi_\omega(\cdot;h^\star_{\omega,P},j^\star_{\omega,P},m^\star_P)]=\nabla\mathcal{F}(\omega)$. Indeed, for each coordinate $k$,
\begin{equation*}
\begin{aligned}
    \mathbb{E}_P\left[\varphi_{k,\omega}(X,Y,Z;h^\star_{\omega,P},j^\star_{\omega,P},m^\star_P)\right]&=\mathbb{E}_P\left[\langle g_\omega(Z)-Y,j^\star_{k,\omega,P}(X)\rangle\right]\\
    &\quad+\mathbb{E}_P\left[\langle\partial_{\omega_k}g_\omega(Z)-j^\star_{k,\omega,P}(X),h^\star_{\omega,P}(X)-m^\star_P(X)\rangle\right]\\
    &=\mathbb{E}_P\left[\langle h^\star_{\omega,P}(X)-Y,j^\star_{k,\omega,P}(X)\rangle\right]=\partial_{\omega_k}\mathcal F(\omega),
\end{aligned}
\end{equation*}
where the second term vanishes because $j^\star_{k,\omega,P}(X)=\mathbb{E}_P[\partial_{\omega_k}g_\omega(Z)\mid X]$, so $\mathbb{E}_P[\partial_{\omega_k}g_\omega(Z)-j^\star_{k,\omega,P}(X)\mid X]=0$ and the tower property applies.

For arbitrary $h\in\LL$, $j\in\LL^d$, and $m\in\LL$, the tower property gives
\begin{equation*}
\begin{aligned}
    \mathbb{E}_P[\varphi_{k,\omega}(\cdot;h,j,m)]-\partial_{\omega_k}\mathcal F(\omega)&=-\mathbb{E}_P[\langle j_k(X)-j^\star_{k,\omega,P}(X),h(X)-h^\star_{\omega,P}(X)\rangle]\\
    &\quad+\mathbb{E}_P[\langle j_k(X)-j^\star_{k,\omega,P}(X),m(X)-m^\star_P(X)\rangle].
\end{aligned}
\end{equation*}
By Cauchy-Schwarz, summing over $k$ and using $\|u\|=1$,
\begin{equation}
\label{eq:quadratic-product-bias-bound}
    \bigl\|\mathbb{E}_P[\varphi_\omega(\cdot;h,j,m)]-\nabla\mathcal{F}(\omega)\bigr\|\leq\|j-j^\star_{\omega,P}\|_{\LL^d}\left(\|h-h^\star_{\omega,P}\|_{\LL}+\|m-m^\star_P\|_{\LL}\right).
\end{equation}

\paragraph{Step 3: estimator decomposition.}
Write $\varphi^\star_\omega\coloneqq\varphi_\omega(\cdot;h^\star_{\omega,P},j^\star_{\omega,P},m^\star_P)$ and $\hat\varphi_\omega\coloneqq\varphi_\omega(\cdot;\hat h_\omega,\hat j_\omega,\hat m)$. Then
\begin{equation*}
    \widehat{\nabla\mathcal F}^{\mathrm{OBi}}(\omega)-\nabla\mathcal{F}(\omega)=(\mathbb{E}_{P_{n,2}}-\mathbb{E}_P)[\varphi^\star_\omega]+(\mathbb{E}_{P_{n,2}}-\mathbb{E}_P)[\hat\varphi_\omega-\varphi^\star_\omega]+\Bigl(\mathbb{E}_P[\hat\varphi_\omega]-\nabla\mathcal{F}(\omega)\Bigr).
\end{equation*}
We work throughout on the event $\mathcal{E}_{\rm tr}(r_h,r_j,r_m)$. On this event, the product-bias bound \eqref{eq:quadratic-product-bias-bound} gives
\begin{equation*}
    \sup_{\omega\in\Omega}\bigl\|\mathbb{E}_P[\hat\varphi_\omega]-\nabla\mathcal{F}(\omega)\bigr\|\leq r_j(r_h+r_m).
\end{equation*}
Using the dual representation $\|v\|=\sup_{u\in \mathbb{S}^{d-1}}u^\top v$, where $\mathbb{S}^{d-1}=\{u\in\mathbb{R}^d:\|u\|=1\}$, we obtain
\begin{equation*}
\begin{aligned}
    \mathbf 1_{\mathcal E_{\mathrm{tr}}(r_h,r_j,r_m)}\sup_{\omega\in\Omega}\|\widehat{\nabla\mathcal F}^{\mathrm{OBi}}(\omega)-\nabla\mathcal{F}(\omega)\|&\leq\mathbf 1_{\mathcal E_{\mathrm{tr}}(r_h,r_j,r_m)}\sup_{f\in\mathcal{F}_A}|(P_{n,2}-P)f|\\
    &\quad+\mathbf 1_{\mathcal E_{\mathrm{tr}}(r_h,r_j,r_m)}\sup_{f\in\mathcal{F}_{B,n}^{\rm loc}(r_h,r_j,r_m)}|(P_{n,2}-P)f|\\
    &\quad+r_j(r_h+r_m),
\end{aligned}
\end{equation*}
where
\begin{equation*}
    \mathcal{F}_A\coloneqq\left\{u^\top\left(\varphi_\omega(\cdot;h^\star_{\omega,P},j^\star_{\omega,P},m^\star_P)-\nabla\mathcal{F}(\omega)\right):\omega\in\Omega,\ u\in \mathbb{S}^{d-1}\right\}
\end{equation*}
is the oracle score class and $\mathcal{F}_{B,n}^{\rm loc}(r_h,r_j,r_m)$ is the localized nuisance-difference class defined in \Cref{thm:quadratic-maximal-inequality}.

\paragraph{Step 4: envelopes.}
For the oracle score, by Cauchy-Schwarz and the bounds of Step~1,
\begin{equation*}
\begin{aligned}
    \|\varphi^\star_\omega(X,Y,Z)\|&\leq\|g_\omega(Z)-Y\|\|j^\star_{\omega,P}(X)\|_{2,d}+\|\partial_\omega g_\omega(Z)-j^\star_{\omega,P}(X)\|_{2,d}\|h^\star_{\omega,P}(X)-m^\star_P(X)\|\\
    &\leq(A+B)D+2D(A+B)=3(A+B)D,
\end{aligned}
\end{equation*}
where $\|\partial_\omega g_\omega(Z)-j^\star_{\omega,P}(X)\|_{2,d}\leq \|\partial_\omega g_\omega(Z)\|_{2,d}+\|j^\star_{\omega,P}(X)\|_{2,d}\leq 2D$ by the triangle inequality and \Cref{assum:eif-integrability}. Since $\nabla\mathcal{F}(\omega)=\mathbb{E}_P[\varphi^\star_\omega]$, Jensen's inequality gives $\|\nabla\mathcal{F}(\omega)\|\leq 3(A+B)D$. Therefore $F_A\coloneqq 6(A+B)D$ is an envelope for $\mathcal{F}_A$.

For the nuisance-difference class, let $\mathcal{F}_{B,n}$ denote the ambient class obtained from $\mathcal{F}_{B,n}^{\rm loc}(r_h,r_j,r_m)$ by dropping the radius constraints. Fix $\omega\in\Omega$, $u\in \mathbb{S}^{d-1}$, $h\in\mathcal{H}_n$, $j\in\mathcal{J}_n$, and $m\in\mathcal{M}_n$. Write $\Delta h=h-h^\star_{\omega,P}$, $\Delta j=j-j^\star_{\omega,P}$, and $\Delta m=m-m^\star_P$. The score difference expands as
\begin{equation*}
\begin{aligned}
    \varphi_{k,\omega}(X,Y,Z;h,j,m)-\varphi^\star_{k,\omega}(X,Y,Z)&=\langle g_\omega(Z)-Y,\Delta j_k(X)\rangle\\
    &\quad+\langle\partial_{\omega_k}g_\omega(Z)-j^\star_{k,\omega,P}(X),\Delta h(X)-\Delta m(X)\rangle\\
    &\quad-\langle\Delta j_k(X),h(X)-m(X)\rangle.
\end{aligned}
\end{equation*}
By \Cref{assum:eif-integrability,ass:maximal-bounded-nuisance-classes},
\begin{equation*}
    \|\Delta j(X)\|_{2,d}\leq\bar J_n+D,\qquad\|\Delta h(X)-\Delta m(X)\|\leq\bar H_n+\bar M_n+A+B,
\end{equation*}
and
\begin{equation*}
    \|h(X)-m(X)\|\leq\bar H_n+\bar M_n,\qquad\|\partial_\omega g_\omega(Z)-j^\star_\omega(X)\|_{2,d}\leq 2D.
\end{equation*}
Thus
\begin{equation*}
    F_{B,n}\coloneqq(A+B)(\bar J_n+D)+2D(\bar H_n+\bar M_n+A+B)+(\bar J_n+D)(\bar H_n+\bar M_n)
\end{equation*}
is an $S_1$-measurable envelope for $\mathcal{F}_{B,n}$, and hence also for $\mathcal{F}_{B,n}^{\rm loc}(r_h,r_j,r_m)$. Moreover, every $f\in\mathcal{F}_{B,n}^{\rm loc}(r_h,r_j,r_m)$ satisfies
\begin{equation*}
    \|f\|_{L^2(P)}\leq(A+B+\bar H_n+\bar M_n)r_j+2D(r_h+r_m)\eqqcolon\rho_{B,n},
\end{equation*}
where the three terms in the score difference contribute $(A+B)r_j$, $2D(r_h+r_m)$, and $(\bar H_n+\bar M_n)r_j$ respectively.
\paragraph{Step 5: entropy transfer.}
We transfer the polynomial entropy bounds of \Cref{ass:maximal-population-entropy,ass:maximal-nuisance-entropy} to the score classes $\mathcal{F}_A$ and $\mathcal{F}_{B,n}$ via Lemmas~\ref{lem:entropy-sums}, \ref{lem:bounded-product-entropy}, and \ref{lem:block-linear-combination-entropy}. Throughout, we use the notation
\begin{equation*}
    \mathscr{S}(\mathcal{A})\coloneqq\left\{(X,Y,Z)\mapsto\sum_{k=1}^d u_k a_k(X,Y,Z):a=(a_1,\dots,a_d)\in\mathcal{A},\ u\in \mathbb{S}^{d-1}\right\}
\end{equation*}
for the scalarized class associated with a block-vector class $\mathcal{A}$.

\textit{Oracle score class $\mathcal{F}_A$.} Define $\mathcal{U}_g\coloneqq\mathcal{G}_\Omega-Y$, $\mathcal{V}^\star\coloneqq\mathcal{H}^\star-m^\star_P$, $\mathscr{J}^\star\coloneqq\mathscr{S}(\mathcal{J}^\star)$, $\dot{\mathscr{G}}\coloneqq\mathscr{S}(\dot{\mathcal{G}}_\Omega)$, and $\Delta\dot{\mathscr{G}}^\star\coloneqq\dot{\mathscr{G}}-\mathscr{J}^\star$. Reading off the three terms in the oracle score, we obtain the inclusion
\begin{equation*}
    \mathcal{F}_A\subset(\mathcal{U}_g\odot\mathscr{J}^\star)+(\Delta\dot{\mathscr{G}}^\star\odot\mathcal{V}^\star)-\mathcal{C}_\Psi,
\end{equation*}
where $\mathcal{C}_\Psi\coloneqq\{(X,Y,Z)\mapsto u^\top\nabla\mathcal{F}(\omega):\omega\in\Omega,\ u\in \mathbb{S}^{d-1}\}$ is a class of deterministic constant functions bounded by $3(A+B)D$. Its $\varepsilon$-covering number grows at most at rate $d\log(1/\varepsilon)$, which is dominated by $K\varepsilon^{-p}$ for any $p>0$. Applying Lemmas~\ref{lem:entropy-sums}, \ref{lem:bounded-product-entropy}, and \ref{lem:block-linear-combination-entropy} to each product and sum, together with \Cref{ass:maximal-population-entropy}, yields a deterministic constant $K_A<\infty$ depending only on the population entropy constants such that
\begin{equation*}
    \sup_Q\log N(\varepsilon F_A,\mathcal{F}_A,L^2(Q))\leq K_A\varepsilon^{-p},\qquad 0<\varepsilon\leq1.
\end{equation*}

\textit{Nuisance-difference class $\mathcal{F}_{B,n}$.}
Condition on $S_1$ and define $\mathscr{J}_n\coloneqq\mathscr{S}(\mathcal{J}_n)$, $\Delta\mathscr{J}_n\coloneqq\mathscr{J}_n-\mathscr{J}^\star$, $\Delta\mathcal{H}_n\coloneqq\mathcal{H}_n-\mathcal{H}^\star$, $\Delta\mathcal{M}_n\coloneqq\mathcal{M}_n-m^\star_P$, $\Delta\mathcal{V}_n\coloneqq\Delta\mathcal{H}_n-\Delta\mathcal{M}_n$, and $\mathcal{W}_n\coloneqq\mathcal{H}_n-\mathcal{M}_n$. The score difference expands as
\begin{equation*}
\begin{aligned}
    u^\top\bigl(\varphi_\omega(\cdot;h,j,m)-\varphi^\star_\omega\bigr)&=\left\langle g_\omega-Y,\sum_{k=1}^d u_k(j_k-j^\star_{k,\omega})\right\rangle\\
    &\quad+\left\langle\sum_{k=1}^d u_k(\partial_{\omega_k}g_\omega-j^\star_{k,\omega}),(h-h^\star_\omega)-(m-m^\star_P)\right\rangle\\
    &\quad-\left\langle\sum_{k=1}^d u_k(j_k-j^\star_{k,\omega}),h-m\right\rangle,
\end{aligned}
\end{equation*}
which gives the inclusion
\begin{equation*}
    \mathcal{F}_{B,n}\subset(\mathcal{U}_g\odot\Delta\mathscr{J}_n)+(\Delta\dot{\mathscr{G}}^\star\odot\Delta\mathcal{V}_n)-(\Delta\mathscr{J}_n\odot\mathcal{W}_n).
\end{equation*}
Applying the same three lemmas together with \Cref{ass:maximal-nuisance-entropy} yields an $S_1$-measurable finite constant $K_{B,n}$ such that, conditionally on $S_1$,
\begin{equation*}
    \sup_Q\log N(\varepsilon F_{B,n},\mathcal{F}_{B,n},L^2(Q))\leq K_{B,n}\varepsilon^{-p},\qquad 0<\varepsilon\leq1.
\end{equation*}
Since $\mathcal{F}_{B,n}^{\rm loc}(r_h,r_j,r_m)\subset\mathcal{F}_{B,n}$, the same bound holds for the localized class. Finally, adding the zero function and passing to the star hull each change the covering number by at most a constant depending only on $p$ (since $p>0$), so the normalized class
\begin{equation*}
    \mathcal{F}_{B,n}^{{\rm loc},0}\coloneqq\left\{f/F_{B,n}:f\in\mathcal{F}_{B,n}^{\rm loc}(r_h,r_j,r_m)\right\}\cup\{0\}
\end{equation*}
satisfies the same polynomial entropy bound, up to increasing $K_{B,n}$ by a constant depending only on $p$.

\paragraph{Step 6: empirical process bounds and conclusion.}
By Corollary~\ref{cor:polynomial-entropy-maximal} applied to $\mathcal{F}_A$ with envelope $F_A=6(A+B)D$ and entropy constant $K_A$,
\begin{equation*}
    \mathbb{E}\left[\sup_{f\in\mathcal{F}_A}|(P_{n,2}-P)f|\,\middle|\,S_1\right]\leq\frac{C_p}{\sqrt{n}}F_A\sqrt{1+K_A}=\frac{6C_p}{\sqrt{n}}(A+B)D\sqrt{1+K_A}.
\end{equation*}
By definition of $\mathfrak{C}_{B,n}(r_h,r_j,r_m)$,
\begin{equation*}
    \mathbb{E}\left[\sup_{f\in\mathcal{F}_{B,n}^{\rm loc}(r_h,r_j,r_m)}|(P_{n,2}-P)f|\,\middle|\,S_1\right]=\mathfrak{C}_{B,n}(r_h,r_j,r_m).
\end{equation*}
Taking conditional expectations in Step~3 and combining,
\begin{equation*}
    \mathbb{E}\left[\sup_{\omega\in\Omega}\|\widehat{\nabla\mathcal F}^{\mathrm{OBi}}(\omega)-\nabla\mathcal{F}(\omega)\|_2\,\middle|\,S_1\right]\leq\frac{6C_p}{\sqrt{n}}(A+B)D\sqrt{1+K_A}+\mathfrak{C}_{B,n}(r_h,r_j,r_m)+r_j(r_h+r_m),
\end{equation*}
which gives the stated bound with $C=6C_p$ and $K=K_A$.
\end{proof}

\subsection{Polynomial-entropy instantiation}
\label{sec:pr_cor}

\begin{proposition}[Polynomial-entropy instantiation]
\label{cor:localized-polynomial-rate}
Under the conditions of Theorem~\ref{thm:quadratic-maximal-inequality}, suppose the nuisance envelopes and entropy constants remain bounded and $r_h,r_j,r_m=O(n^{-1/(2+p)})$. Then, on $\mathcal{E}_{\rm tr}(r_h,r_j,r_m)$,
\begin{equation}
\label{eq:polynomial-uniform-rate}
\mathbb{E}\left[\sup_{\omega\in\Omega}\|\widehat{\nabla\mathcal F}^{\mathrm{OBi}}(\omega)-\nabla\mathcal{F}(\omega)\|\,\middle|\,S_1\right]=O(n^{-1/2}).
\end{equation}
\end{proposition}

\begin{proof}
Conditionally on $S_1$, the localized class $\mathcal{F}_{B,n}^{\rm loc}(r_h,r_j,r_m)$ has envelope $F_{B,n}$ and $L^2(P)$ radius at most $\rho_{B,n}=(A+B+\bar H_n+\bar M_n)r_j+2D(r_h+r_m)$. By the entropy bound established in Step~5 of the proof of Theorem~\ref{thm:quadratic-maximal-inequality}, the normalized class $\mathcal{F}_{B,n}^{{\rm loc},0}$ and its star hull have entropy of order $(1+K_{B,n})\varepsilon^{-p}$. Conditionally on $S_1$, set $a_n:=(1+K_{B,n})/n$. Applying the local uniform-entropy inequality of \citep[Theorem~2.1]{vdvw2010} to the signed normalized star hull, as in \eqref{eq:critical-bound-source}, gives, for $0<\delta\le1$,
\begin{equation*}
 \mathfrak R_n\!\left(\operatorname{star}(\mathcal F_{B,n}^{{\rm loc},0}),\delta\,\middle|\,S_1\right)
 \le C_p\left\{\sqrt{a_n}\,\delta^{1-p/2}+a_n\delta^{-p}\right\}.
\end{equation*}
Choose $\delta_{B,n}:=c_*a_n^{1/(2+p)}\wedge1$, where $c_*$ depends only on $p$. On the untruncated branch,
\begin{equation*}
 \frac{\sqrt{a_n}\,\delta_{B,n}^{1-p/2}}{\delta_{B,n}^2}=c_*^{-1-p/2},\qquad\frac{a_n\delta_{B,n}^{-p}}{\delta_{B,n}^2}=c_*^{-2-p}.
\end{equation*}
Taking $c_*$ sufficiently large therefore makes the Rademacher complexity at most $\delta_{B,n}^2$, and Lemma~\ref{lem:localized-osl-maximal} applies. On the truncated branch, the envelope bound $\mathfrak C_{B,n}\le2F_{B,n}$ is absorbed by the $F_{B,n}\delta_{B,n}^2$ term below; bounded $K_{B,n}$ makes this branch occur only for finitely many $n$. Thus, after increasing constants,
\begin{equation*}
    \mathfrak{C}_{B,n}(r_h,r_j,r_m)\leq C\left\{\delta_{B,n}\rho_{B,n}+F_{B,n}\delta_{B,n}^2+F_{B,n}e^{-cn\delta_{B,n}^2}\right\}.
\end{equation*}
Under the assumption that $F_{B,n}$, $K_{B,n}$, $\bar H_n$, and $\bar M_n$ are $O(1)$ and $r_h,r_j,r_m=O(n^{-1/(2+p)})$, we have $\delta_{B,n}=O(n^{-1/(2+p)})$, $\rho_{B,n}=O(n^{-1/(2+p)})$, and $n\delta_{B,n}^2\asymp n^{p/(2+p)}\to\infty$, so:
\begin{itemize}
    \item $\delta_{B,n}\rho_{B,n}=O(n^{-2/(2+p)})=o(n^{-1/2})$ since $p<2$,
    \item $F_{B,n}\delta_{B,n}^2=O(n^{-2/(2+p)})=o(n^{-1/2})$ since $p<2$,
    \item $F_{B,n}e^{-cn\delta_{B,n}^2}=o(n^{-1/2})$ since 
    $n\delta_{B,n}^2\to\infty$.
\end{itemize}
Hence $\mathfrak{C}_{B,n}(r_h,r_j,r_m)=o(n^{-1/2})$. Since the product-bias term satisfies $r_j(r_h+r_m)=O(n^{-2/(2+p)})=o(n^{-1/2})$, substituting into Theorem~\ref{thm:quadratic-maximal-inequality} gives
\begin{equation*}
    \mathbb{E}\left[\sup_{\omega\in\Omega}\|\widehat{\nabla\mathcal F}^{\mathrm{OBi}}(\omega)-\nabla\mathcal{F}(\omega)\|\,\middle|\,S_1\right]\leq\frac{6C_p}{\sqrt{n}}(A+B)D\sqrt{1+K_A}+o(n^{-1/2})=O(n^{-1/2}).
\end{equation*}
\end{proof}

\subsection{Unconditional rates under source and entropy conditions}
\label{sec:source-only}

We prove \Cref{cor:source-only}. The argument uses the same three-term decomposition as the proof of \Cref{thm:quadratic-maximal-inequality}, but with the deterministic nuisance classes of \Cref{ass:source}. Conditioning on the training fold is used only as a device inside the proof; all conclusions concern full expectations. Let $S_1$ and $S_2$ be independent folds of $n$ i.i.d. observations $O=(X,Y,Z)\sim P$, and set $N=2n$. The nuisances are fitted on $S_1$ and the hypergradient is evaluated on $S_2$. Write $P_{n,r}$ for the empirical measure on $S_r$ and $Pf=\E_P[f(O)]$. Retain \Cref{assum:eif-integrability}: $\omega\mapsto g_\omega(z)$ is twice continuously differentiable for every $z$ and
\begin{equation*}
 \norm{Y}\le A\text{ a.s.},\quad
 \sup_{\omega\in\Omega}\norm{g_\omega(z)}\le B,\quad
 \sup_{\omega\in\Omega}\norm{\dot g_\omega(z)}_{2,d}\le D,\quad\sup_{\omega\in\Omega}\|D_\omega^2g_\omega(z)\|_{\mathrm{op}}\leq L
 \quad\forall z\in\mathcal{Z},
\end{equation*}
where $\dot g_\omega=(\partial_{\omega_1}g_\omega,\ldots,\partial_{\omega_d}g_\omega)=D_\omega g_\omega$ and $\norm{a}_{2,d}^2=\sum_{k=1}^d\norm{a_k}^2$ for $a\in(\R^q)^d$, so that $\|\cdot\|_{2,d}$ is the Frobenius norm of the stacked derivative. The nuisance targets are
\begin{equation*}
 \hstar(X)=\E_P[g_\omega(Z)\mid X],\qquad
 \jstar(X)=\E_P[\dot g_\omega(Z)\mid X],\qquad
 \mstar(X)=\E_P[Y\mid X].
\end{equation*}
The feasible set $\Omega\subset\R^d$ is nonempty, compact, and convex, and
$\mathcal F(\omega)=\tfrac12P[\norm{Y-\hstar(X)}^2]$.
In this subsection, write $L_g:=D+L$ for the uniform Lipschitz constant on $\Omega$ supplied by \Cref{assum:eif-integrability} for the pair $(g_\omega,\dot g_\omega)$, using the sum of the two norms. All displayed suprema and empirical-process quantities are assumed measurable. Constants denoted by $C,c>0$ may change from line to line.

For $\ell\in\{h,j,m\}$, let $V_h=V_m=\R^q$, $V_j=(\R^q)^d$, $\Omega_h=\Omega_j=\Omega$, and $\Omega_m=\{0\}$. Use the Euclidean norm on $V_h,V_m$ and the block norm on $V_j$. Set
\begin{align*}
 &(b_h,b_j,b_m)=(B,D,A),\\
 &(T^h_\omega,T^j_\omega,T^m_0)=(g_\omega(Z),\dot g_\omega(Z),Y),\\
 &(f^\star_{h,\omega},f^\star_{j,\omega},f^\star_{m,0})=(\hstar,\jstar,\mstar).
\end{align*}
For classes $\cF_\ell(t)$ of $V_\ell$-valued functions on $\mathcal X$, indexed by $t\in(0,1]$, define the loss and the population and empirical risks
\begin{equation}
 \ell^\ell_\omega(O;f)=\tfrac12\norm{f(X)-T^\ell_\omega(O)}^2,\quad\widehat R^\ell_\omega(f)=P_{n,1}\ell^\ell_\omega(\cdot;f),\quad R^\ell_\omega(f)=P\ell^\ell_\omega(\cdot;f).
 \label{eq:risks-source-appendix}
\end{equation}

Let $U$ be the auxiliary fitting randomness of the estimator setup in \Cref{sec:debiased-clt}, independent of all data, and write $\Ttrain=\sigma(S_1,U)$, omitting $U$ for deterministic fitting. As in that setup, the estimates $\widehat f_{\ell,\omega}\in\cF_\ell(t_n)$ are $\Ttrain$-measurable. Put $(\widehat f_{h,\omega},\widehat f_{j,\omega},\widehat f_{m,0})=(\widehat h_\omega,\widehat j_\omega,\widehat m)$. 
The search classes are
$\cH_n=\cF_h(t_n)$, $\cJ_n=\cF_j(t_n)$, and $\cM_n=\cF_m(t_n)$. Define
\begin{equation}
 \delta_n\coloneqq (nt_n^\gamma)^{-1/(2+p)},\qquad v_n\coloneqq t_n^{2\nu}+\delta_n^2,\qquad\varepsilon_n\coloneqq\max_{\ell\in\{h,j,m\}}\varepsilon_{\ell,n}.
 \label{eq:source-scales}
\end{equation}
Recall the orthogonal score $\varphi_\omega=(\varphi_{1,\omega},\ldots,\varphi_{d,\omega})^\top$, where
\begin{equation}
 \varphi_{k,\omega}(O;h,j,m)=\ip{g_\omega(Z)-Y}{j_k(X)}+\ip{\partial_{\omega_k}g_\omega(Z)-j_k(X)}{h(X)-m(X)}.
 \label{eq:score-source}
\end{equation}
Set $\ghat(\omega)=P_{n,2}\varphi_\omega(\cdot;\widehat h_\omega,\widehat j_\omega,\widehat m)$ and
\begin{equation}
 R_{\ell,n}\coloneqq\sup_{\omega\in\Omega_\ell}\norm{\widehat f_{\ell,\omega}-f^\star_{\ell,\omega}}_{L^2(P_X;V_\ell)},\qquad\Delta_{n,2}\coloneqq\sup_{\omega\in\Omega}\norm{\ghat(\omega)-\nabla\mathcal F(\omega)}.
 \label{eq:errors-source}
\end{equation}
Since the classes and the oracle paths are bounded by $b_\ell$, each $R_{\ell,n}\le 2b_\ell$ is finite, and we write $\rho_{\ell,n}\coloneqq(\E[R_{\ell,n}^2])^{1/2}$.

\begin{proof}[Proof of Theorem~\ref{cor:source-only}]
The first bound, \eqref{eq:generic-source}, is \eqref{eq:product-source} in part~(i) of \Cref{cor:source-only_appendix}, which uses only \Cref{assum:eif-integrability,ass:source} and the estimator setup. The nuisance second-moment bound is \eqref{eq:nuisance-source} in part~(ii), the first place where \Cref{ass:approx,ass:erm} are used. Under \eqref{eq:source-tuning}, \eqref{eq:tuned-source-appendix} gives $\rho_{\ell,n}=O(n^{-\nu/b})$, $\delta_n=n^{-\nu/b}$, and \eqref{eq:tuned-source}.
\end{proof}

\begin{proposition}[Unconditional source--capacity rates]\label{cor:source-only_appendix}
(i) Under \Cref{assum:eif-integrability,ass:source} and the estimator setup above, there exist finite constants $C_0$, depending only on $A,B,D,L_g,d$ and $\operatorname{diam}(\Omega)$, and $C$, depending in addition on $p,\gamma$ and $(K_\ell)_{\ell\in\{h,j,m\}}$, such that, for every $n\ge1$ and every deterministic $t_n\in(0,1]$,
\begin{equation}
 \E\left[\Delta_{n,2}\right]\le C_0 n^{-1/2}+C\Bigl\{\delta_n\sum_{\ell\in\{h,j,m\}}\rho_{\ell,n}+\delta_n^2+\rho_{j,n}(\rho_{h,n}+\rho_{m,n})\Bigr\}.
 \label{eq:product-source}
\end{equation}
(ii) If, in addition, \Cref{ass:approx,ass:erm} hold, there exist finite constants $C_h,C_j,C_m,C_1$, depending only on $A,B,D,L_g,p,\nu,\gamma,d,q$, $\operatorname{diam}(\Omega)$, and $(K_\ell,L_\ell)_{\ell\in\{h,j,m\}}$, such that, for every $n\ge1$,
\begin{equation}
 \E[R_{\ell,n}^2]\le C_\ell\left\{t_n^{2\nu}+(nt_n^\gamma)^{-2/(2+p)}+\varepsilon_{\ell,n}\right\},\qquad \ell\in\{h,j,m\},
 \label{eq:nuisance-source}
\end{equation}
and
\begin{equation}
 \E\left[\Delta_{n,2}\right]\le C_0 n^{-1/2}+C_1\left\{t_n^{2\nu}+(nt_n^\gamma)^{-2/(2+p)}+\varepsilon_n\right\}.
 \label{eq:gradient-source}
\end{equation}
The expectation is over both folds and all fitting randomness. The same gradient bounds hold for the two-fold cross-fitted estimator in \eqref{eq:cf-debiased-estimator}, under the same conditions in both training directions. In particular, let
\begin{equation}
 b\coloneqq\nu(2+p)+\gamma,\qquad t_n=n^{-1/b},\qquad\varepsilon_{\ell,n}\le c_\ell n^{-2\nu/b}
 \label{eq:source-tuning}
\end{equation}
for fixed $c_\ell<\infty$. Then, with constants also allowed to depend on $(c_\ell)_\ell$,
\begin{equation}
 \begin{gathered}
 \E[R_{\ell,n}^2]\le C_\ell' n^{-2\nu/b},\qquad\rho_{\ell,n}=O(n^{-\nu/b}),\qquad\delta_n=n^{-\nu/b},\\
 \E[\Delta_{n,2}]\le C_0 n^{-1/2}+C_2 n^{-2\nu/b}.
 \end{gathered}
 \label{eq:tuned-source-appendix}
\end{equation}
Thus the expected uniform gradient error is $O(n^{-1/2})$ if $\nu(2-p)\ge\gamma$. The second term in the gradient bound is $o(n^{-1/2})$ if $\nu(2-p)>\gamma$.
\end{proposition}

To prove \cref{cor:source-only_appendix}, we use the following localization bound for the training losses and the evaluation score differences. Here $P_n$ is the empirical measure of an independent sample of size $n$ from $P$.

\begin{lemma}[Localization under the entropy scale]\label{lem:source-local}
Let $\mathcal A_n$ be a pointwise measurable class of real-valued functions bounded by one, with
\begin{equation*}
 \sup_Q\log N(\epsilon,\mathcal A_n,L^2(Q))\le Kt_n^{-\gamma}\epsilon^{-p},\qquad 0<\epsilon\le1,
\end{equation*}
where $0\le K<\infty$, $p\in(0,2)$, $\gamma\ge0$, and $t_n\in(0,1]$. Set $\delta_n=(nt_n^\gamma)^{-1/(2+p)}$. There exist constants $C,c,c_0>0$, depending only on $p,K$, such that, whenever $\delta_n\le c_0$, with probability at least $1-Ce^{-cn\delta_n^2}$,
\begin{equation}
 |(P_n-P)a|\le C\left\{\delta_n\norm{a}_{L^2(P)}+\delta_n^2\right\}\quad\text{simultaneously for all }a\in\mathcal A_n.
 \label{eq:relative-source}
\end{equation}
Moreover, for every $n\ge1$ and $r\ge0$,
\begin{equation}
 \E\left[\sup_{\substack{a\in\mathcal A_n\cup\{0\}\\
                         \norm{a}_{L^2(P)}\le r}}
             |(P_n-P)a|\right]\le C\{\delta_n r+\delta_n^2\}.
 \label{eq:localized-expectation-source}
\end{equation}
\end{lemma}

\begin{proof}
Let $\mathcal A_n^\circ=\operatorname{star}(\mathcal A_n\cup\{0\})$, with the star hull taken about zero. A grid for the multiplier in $[0,1]$ shows that
\begin{equation*}
 \sup_Q\log N(\epsilon,\mathcal A_n^\circ,L^2(Q))\le \kappa_n\epsilon^{-p},\qquad\kappa_n=C_p(1+K)t_n^{-\gamma}.
\end{equation*}
Indeed, a product cover at radii $\epsilon/2$ for the function and the multiplier gives an additional term of order $\log(1/\epsilon)$, which is absorbed by $C_p\epsilon^{-p}$. For $0<r\le1$, write
\begin{equation*}
 \mathfrak R_n(\mathcal A_n^\circ,r)\coloneqq\E_{\xi,O_{1:n}}\left[\sup_{\substack{a\in\mathcal A_n^\circ\\\norm{a}_{L^2(P)}\le r}}\left|\frac1n\sum_{i=1}^n\xi_i a(O_i)\right|\right],
\end{equation*}
where $\xi_i$ are independent Rademacher variables, independent of the observations. The local uniform-entropy inequality of \citep[Theorem~2.1]{vdvw2010}, applied to the mean-zero class $(\xi,O)\mapsto\xi a(O)$, yields
\begin{equation}
 \mathfrak R_n(\mathcal A_n^\circ,r)\le C_p\left\{\sqrt{\frac{1+\kappa_n}{n}}\,r^{1-p/2}+\frac{1+\kappa_n}{n}\,r^{-p}\right\}.
 \label{eq:critical-bound-source}
\end{equation}
Here the signed class has the same $L^2$ covering bounds, and its entropy integral up to radius $r$ is at most $C_p\sqrt{1+\kappa_n}\,r^{1-p/2}$. Since $(1+\kappa_n)/n\le C\delta_n^{2+p}$, substitution of $r=c_\star\delta_n$ in \Cref{eq:critical-bound-source} gives
\begin{equation*}
 \frac{\mathfrak R_n(\mathcal A_n^\circ,c_\star\delta_n)}{(c_\star\delta_n)^2}\le C\left\{c_\star^{-1-p/2}+c_\star^{-2-p}\right\}.
\end{equation*}
Choose $c_\star$ so that the right-hand side is at most one, and choose $c_0<1/c_\star$. The critical-radius condition therefore holds at $c_\star\delta_n$ whenever $\delta_n\le c_0$. Apply \citep[Lemma~14]{foster2023orthogonal} to $\mathcal A_n\cup\{0\}$ with identity loss and comparator zero. Its linear-loss version requires no additional logarithmic lower bound on the radius. This gives \eqref{eq:relative-source}, after absorbing $c_\star$ into the constants.

On this event the localized supremum in \eqref{eq:localized-expectation-source} is at most $C(\delta_n r+\delta_n^2)$; on its complement it is at most two. Also $t_n\le1$ and $\gamma\ge0$ imply
\begin{equation}
 \delta_n\ge n^{-1/(2+p)},\qquad e^{-cn\delta_n^2}\le e^{-c n^{p/(2+p)}}\le C n^{-2/(2+p)}\le C\delta_n^2.
 \label{eq:tail-source}
\end{equation}
Taking expectations proves \eqref{eq:localized-expectation-source} for $\delta_n\le c_0$. If $\delta_n>c_0$, the bound follows from the envelope after increasing $C$, since $2\le2c_0^{-2}\delta_n^2$.
\end{proof}

\begin{proof}[Proof of Proposition~\ref{cor:source-only_appendix}]
The proof proceeds in eight steps. Steps~1--5 prove part~(i). They use only \Cref{assum:eif-integrability,ass:source}, the class membership and measurability of the fitted nuisances, and the independence of $S_2$ and $\Ttrain$. Steps~6--8 prove part~(ii) and the tuned rates; \Cref{ass:approx,ass:erm} enter only in Step~7.

\step{Step 1: population paths and entropy transfer.}
Work on a common measurable set of full $P$-probability on which Assumption A holds. Conditional Jensen gives
\begin{equation*}
 \norm{\hstar(X)}\le B,\qquad\norm{\jstar(X)}_{2,d}\le D,\qquad\norm{\mstar(X)}\le A.
\end{equation*}
Choose the conditional means using a regular conditional law. Conditional expectation preserves the supremum-norm Lipschitz bound implied by \Cref{assum:eif-integrability}, so the oracle paths satisfy the same bound on a common set of full $P_X$-probability. Dominated differentiation also gives $\partial_{\omega_k}\hstar=j^\star_{k,\omega,P}$ and
\begin{equation*}
 \partial_{\omega_k}\mathcal F(\omega)=P\ip{\hstar(X)-Y}{j^\star_{k,\omega,P}(X)}.
\end{equation*}
For each of the classes $\mathcal G_\Omega$, $\dot{\mathcal G}_\Omega$, $\mathcal H^\star$, and $\mathcal J^\star$, a Euclidean net of $\Omega$ induces a supremum-norm cover. In particular,
\begin{equation}
 \log N(\eta,\mathcal C,\norm{\cdot}_\infty)\le d\log\left(1+\frac{C L_g\operatorname{diam}(\Omega)}{\eta}\right),\qquad \eta>0.
 \label{eq:oracle-entropy-source}
\end{equation}
The block norm is used for the derivative classes. Since $p>0$, this logarithmic bound is dominated by $C\eta^{-p}$ on bounded intervals of positive radii. Thus \Cref{ass:maximal-population-entropy} holds with fixed constants, without any inclusion of $\mathcal H^\star$ or $\mathcal J^\star$ in the fitted classes. The class membership of the estimator setup and \Cref{ass:source} give \Cref{ass:maximal-bounded-nuisance-classes,ass:maximal-nuisance-entropy} with deterministic envelopes $(B,D,A)$ and entropy constants $(K_\ell t_n^{-\gamma})_\ell$.

\step{Step 2: three-term decomposition.}
Write $\varphi^\star_\omega\coloneqq\varphi_\omega(\cdot;\hstar,\jstar,\mstar)$ and $\widehat\varphi_\omega\coloneqq\varphi_\omega(\cdot;\widehat h_\omega,\widehat j_\omega,\widehat m)$, and let $P\widehat\varphi_\omega$ denote the integral over a new observation with the fitted functions held fixed. By \Cref{thm:mixed-bias} at the true nuisances, $P\varphi^\star_\omega=\nabla\mathcal F(\omega)$. Since $\ghat(\omega)=P_{n,2}\widehat\varphi_\omega$,
\begin{equation}
 \ghat(\omega)-\nabla\mathcal F(\omega)=(P_{n,2}-P)\varphi^\star_\omega+(P_{n,2}-P)(\widehat\varphi_\omega-\varphi^\star_\omega)+\bigl\{P\widehat\varphi_\omega-\nabla\mathcal F(\omega)\bigr\}.
 \label{eq:three-term-source}
\end{equation}
Taking norms and suprema over $\omega\in\Omega$ gives $\Delta_{n,2}\le\mathrm I_n+\mathrm{II}_n+\mathrm{III}_n$, where
\begin{equation*}
 \mathrm I_n\coloneqq\sup_{\omega\in\Omega}\norm{(P_{n,2}-P)\varphi^\star_\omega},\quad\mathrm{II}_n\coloneqq\sup_{\omega\in\Omega}\norm{(P_{n,2}-P)(\widehat\varphi_\omega-\varphi^\star_\omega)},\quad\mathrm{III}_n\coloneqq\sup_{\omega\in\Omega}\norm{P\widehat\varphi_\omega-\nabla\mathcal F(\omega)}.
\end{equation*}
Steps~3--5 bound the expectations of these three terms separately.

\step{Step 3: oracle empirical-process term (\Cref{assum:eif-integrability}).}
Let $\mathcal F_A\coloneqq\{u^\top(\varphi^\star_\omega-\nabla\mathcal F(\omega)):\omega\in\Omega,\ u\in\mathbb S^{d-1}\}$. By the dual representation of the Euclidean norm, $\mathrm I_n=\sup_{f\in\mathcal F_A}|(P_{n,2}-P)f|$. This class does not depend on the training fold, so no conditioning is needed. By Step~1 and the argument of Step~4 in the proof of \Cref{thm:quadratic-maximal-inequality}, $F_A\coloneqq6(A+B)D$ is a constant envelope for $\mathcal F_A$. The oracle score is formed from $g_\omega$, $\dot g_\omega$, $\hstar$, $\jstar$, $Y$, and the single function $\mstar$ by bounded sums and inner products. The constants $u^\top\nabla\mathcal F(\omega)$ and the sphere contribute entropy at most $d\log(1+C/\epsilon)$. Therefore \Cref{lem:entropy-sums,lem:bounded-product-entropy,lem:block-linear-combination-entropy}, together with \eqref{eq:oracle-entropy-source}, give
\begin{equation*}
 \sup_Q\log N(\epsilon F_A,\mathcal F_A,L^2(Q))\le K_A\epsilon^{-1},\qquad 0<\epsilon\le1,
\end{equation*}
where $K_A$ depends only on $A,B,D,L_g,d$, and $\operatorname{diam}(\Omega)$. Here the logarithmic entropy is dominated by the fixed exponent one rather than by $p$. \Cref{cor:polynomial-entropy-maximal}, applied with exponent one to the i.i.d.\ sample $S_2$, gives
\begin{equation}
 \E[\mathrm I_n]\le C_0n^{-1/2},\qquad C_0\coloneqq 6C_{(1)}(A+B)D\sqrt{1+K_A},
 \label{eq:oracle-term-source}
\end{equation}
where $C_{(1)}$ is the universal constant of \Cref{cor:polynomial-entropy-maximal} for exponent one. Thus $C_0$ depends only on $A,B,D,L_g,d$, and $\operatorname{diam}(\Omega)$. It does not depend on $n$, $t_n$, or the nuisance classes.

\step{Step 4: nuisance-dependent empirical-process term (\Cref{assum:eif-integrability,ass:source}).}
Define the deterministic class
\begin{equation*}
 \cV_n\coloneqq\left\{u^\top\{\varphi_\omega(\cdot;h,j,m)-\varphi^\star_\omega\}:\omega\in\Omega,\ u\in\mathbb S^{d-1},\ h\in\cH_n,\ j\in\cJ_n,\ m\in\cM_n\right\}\cup\{0\}.
\end{equation*}
By the envelope \eqref{eq:source-envelope} and Step~1, every candidate score and every oracle score has norm at most $3(A+B)D$, so $M_\varphi\coloneqq 6(A+B)D$ is an envelope for $\cV_n$. The bounded sum and product bounds of \Cref{lem:entropy-sums,lem:bounded-product-entropy,lem:block-linear-combination-entropy}, including scalarization over $u\in\mathbb S^{d-1}$, together with the entropy bound \eqref{eq:source-entropy}, give
\begin{equation}
 \sup_Q\log N(\epsilon M_\varphi,\cV_n,L^2(Q))\le C(1+K_h+K_j+K_m)t_n^{-\gamma}\epsilon^{-p},\qquad 0<\epsilon\le1.
 \label{eq:score-entropy-source}
\end{equation}
In this calculation the sphere contributes only $d\log(1+C/\epsilon)$, and the oracle paths are controlled by \Cref{eq:oracle-entropy-source}.

For a deterministic radius triple $r=(r_h,r_j,r_m)$, restrict $\cV_n$ to candidates with $L^2$ errors at most $r_h,r_j,r_m$, respectively, and retain zero. Call the resulting class $\cV_n(r)$. With $\Delta h=h-\hstar$, $\Delta j=j-\jstar$, and $\Delta m=m-\mstar$, the exact expansion is
\begin{align*}
 \varphi_{k,\omega}(O;h,j,m)-\varphi^\star_{k,\omega}(O)={}&\ip{g_\omega(Z)-Y}{\Delta j_k(X)}\\
 &+\ip{\partial_{\omega_k}g_\omega(Z)-j^\star_{k,\omega,P}(X)}
           {\Delta h(X)-\Delta m(X)}\\
 &-\ip{\Delta j_k(X)}{h(X)-m(X)}.
\end{align*}
Hence every $v\in\cV_n(r)$ satisfies
\begin{equation}
 \norm{v}_{L^2(P)}\le 2(A+B)r_j+2D(r_h+r_m)\le C(r_h+r_j+r_m).
 \label{eq:score-radius-source}
\end{equation}
Applying \Cref{eq:localized-expectation-source} to $\cV_n/M_\varphi$, with the independent sample $S_2$, yields, for every $n\ge1$,
\begin{equation}
 \E\Bigl[\sup_{v\in\cV_n(r)}|(P_{n,2}-P)v|\Bigr]\le C\{\delta_n(r_h+r_j+r_m)+\delta_n^2\}.
 \label{eq:cb-source}
\end{equation}
The class $\cV_n(r)$ contains the localized class of \Cref{thm:quadratic-maximal-inequality}, so \eqref{eq:cb-source} also bounds $\mathfrak C_{B,n}(r_h,r_j,r_m)$.

We now pass to the realized nuisance errors. Conditioning on $\Ttrain$ fixes all fitted nuisance functions, while $S_2$ remains an i.i.d.\ sample from $P$ because it is independent of $\Ttrain$. Let $\etr(r)\coloneqq\{R_{h,n}\le r_h,\ R_{j,n}\le r_j,\ R_{m,n}\le r_m\}$, a $\Ttrain$-measurable event. On $\etr(r)$, class membership implies that $u^\top(\widehat\varphi_\omega-\varphi^\star_\omega)\in\cV_n(r)$ for every $\omega\in\Omega$ and $u\in\mathbb S^{d-1}$, so $\mathrm{II}_n\le\sup_{v\in\cV_n(r)}|(P_{n,2}-P)v|$. Since $\cV_n(r)$ is deterministic, \eqref{eq:cb-source} and independence give
\begin{equation}
 \E[\mathrm{II}_n\mid\Ttrain]\le C\{\delta_n(r_h+r_j+r_m)+\delta_n^2\}\quad\text{on }\etr(r).
 \label{eq:conditional-source}
\end{equation}
Apply \Cref{eq:conditional-source} simultaneously to all positive rational radius triples, outside a common null set. For each training realization, the errors are finite. Take rational triples decreasing to $(R_{h,n},R_{j,n},R_{m,n})$ from above. Each corresponding training event holds, and the right-hand side is continuous in $r$. Thus
\begin{equation}
 \E[\mathrm{II}_n\mid\Ttrain]\le C\{\delta_n(R_{h,n}+R_{j,n}+R_{m,n})+\delta_n^2\}\qquad\text{almost surely}.
 \label{eq:realized-source}
\end{equation}
Taking expectations and using $\E[R_{\ell,n}]\le\rho_{\ell,n}$ gives
\begin{equation}
 \E[\mathrm{II}_n]\le C\Bigl\{\delta_n\sum_{\ell\in\{h,j,m\}}\rho_{\ell,n}+\delta_n^2\Bigr\}.
 \label{eq:ep-term-source}
\end{equation}
This step uses the bounds of \Cref{assum:eif-integrability} and the envelope and entropy conditions of \Cref{ass:source}. It uses neither approximation nor fitting optimality.

\step{Step 5: population bias (\Cref{thm:mixed-bias}).}
For fixed $\omega$, apply \Cref{thm:mixed-bias} coordinatewise to the fitted candidates $(\widehat h_\omega,\widehat j_{k,\omega},\widehat m)$, which are fixed given $\Ttrain$. Cauchy--Schwarz in $\LL$ and summation over $k$ then give
\begin{align}
 \norm{P\widehat\varphi_\omega-\nabla\mathcal F(\omega)}&=\Bigl(\sum_{k=1}^d\ip{\widehat j_{k,\omega}-j^\star_{k,\omega,P}}{(\widehat m-\mstar)-(\widehat h_\omega-\hstar)}_{\LL}^2\Bigr)^{1/2}\nonumber\\
 &\le\norm{\widehat j_\omega-\jstar}_{\LL^d}\bigl(\norm{\widehat h_\omega-\hstar}_{\LL}+\norm{\widehat m-\mstar}_{\LL}\bigr).
 \label{eq:uniform-product-bias}
\end{align}
Taking the supremum over $\omega$ gives $\mathrm{III}_n\le R_{j,n}(R_{h,n}+R_{m,n})$ pointwise. By Cauchy--Schwarz for the expectation,
\begin{equation}
 \E[\mathrm{III}_n]\le\E[R_{j,n}R_{h,n}]+\E[R_{j,n}R_{m,n}]\le\rho_{j,n}(\rho_{h,n}+\rho_{m,n}).
 \label{eq:bias-term-source}
\end{equation}
No independence among the nuisance errors is used. This term is controlled by orthogonality alone, separately from the localization argument of Step~4.

Combining \eqref{eq:oracle-term-source}, \eqref{eq:ep-term-source}, and \eqref{eq:bias-term-source} with $\Delta_{n,2}\le\mathrm I_n+\mathrm{II}_n+\mathrm{III}_n$, and increasing $C$ to at least one, proves \eqref{eq:product-source}. This completes part~(i).

\step{Step 6: simultaneous relative deviations of the training losses (\Cref{ass:source}).}
For each $\ell$, define
\begin{equation*}
 q^\ell_{\omega,f}\coloneqq\ell^\ell_\omega(\cdot;f)-\ell^\ell_\omega(\cdot;f^\star_{\ell,\omega}),\qquad\cQ_{\ell,n}\coloneqq\{q^\ell_{\omega,f}:\omega\in\Omega_\ell,f\in\cF_\ell(t_n)\}\cup\{0\}.
\end{equation*}
The conditional-mean identity for squared loss gives
\begin{equation}
 Pq^\ell_{\omega,f}=\tfrac12\norm{f-f^\star_{\ell,\omega}}_{L^2(P_X;V_\ell)}^2,\qquad\norm{q^\ell_{\omega,f}}_{L^2(P)}\le2b_\ell\norm{f-f^\star_{\ell,\omega}}_{L^2(P_X;V_\ell)}.
 \label{eq:loss-identity-source}
\end{equation}
The second inequality follows pointwise from
\begin{equation*}
 q^\ell_{\omega,f}=\tfrac12\ip{f(X)-f^\star_{\ell,\omega}(X)}{f(X)+f^\star_{\ell,\omega}(X)-2T^\ell_\omega(O)}.
\end{equation*}
Consequently, $M_\ell\coloneqq 4b_\ell^2$ is a fixed envelope for $\cQ_{\ell,n}$. \Cref{lem:entropy-sums,lem:bounded-product-entropy}, applied to this bounded-product representation, together with \Cref{eq:source-entropy,eq:oracle-entropy-source}, imply
\begin{equation}
 \sup_Q\log N(\epsilon M_\ell,\cQ_{\ell,n},L^2(Q))\le C(1+K_\ell)t_n^{-\gamma}\epsilon^{-p},\qquad 0<\epsilon\le1.
 \label{eq:loss-entropy-source}
\end{equation}
Here $Q$ ranges over finitely supported laws of $O$; the covers for the candidate functions use its $X$-marginal. Applying \Cref{lem:source-local} to $\cQ_{\ell,n}/M_\ell$ and using \Cref{eq:loss-identity-source} shows that, whenever $\delta_n\le c_0$, an event $\cE_{\ell,n}$ of probability at least $1-Ce^{-cn\delta_n^2}$ satisfies
\begin{equation}
 |(P_{n,1}-P)q^\ell_{\omega,f}|\le C\left\{\delta_n\norm{f-f^\star_{\ell,\omega}}_{L^2(P_X;V_\ell)}+\delta_n^2\right\}
 \label{eq:loss-relative-source}
\end{equation}
simultaneously over $\omega\in\Omega_\ell$ and $f\in\cF_\ell(t_n)$. The event depends only on $S_1$.

\step{Step 7: uniform nuisance moments (\Cref{ass:approx,ass:erm}).}
Fix $\ell$ and work on $\cE_{\ell,n}$ and the almost-sure fitting event of \Cref{ass:erm}. For any $\omega\in\Omega_\ell$ and comparator $f\in\cF_\ell(t_n)$, write
\begin{equation*}
 r_\omega=\norm{\widehat f_{\ell,\omega}-f^\star_{\ell,\omega}}_{L^2(P_X;V_\ell)},\qquad a_\omega(f)=\norm{f-f^\star_{\ell,\omega}}_{L^2(P_X;V_\ell)}.
\end{equation*}
Equations~\eqref{eq:erm-source}, \eqref{eq:loss-identity-source}, and \eqref{eq:loss-relative-source} give
\begin{align*}
 \tfrac12 r_\omega^2&=Pq^\ell_{\omega,\widehat f_{\ell,\omega}}\\
 &\le P_{n,1}q^\ell_{\omega,\widehat f_{\ell,\omega}}
          +C\delta_n r_\omega+C\delta_n^2\\
 &\le P_{n,1}q^\ell_{\omega,f}+\varepsilon_{\ell,n}
          +C\delta_n r_\omega+C\delta_n^2\\
 &\le\tfrac12 a_\omega(f)^2+\varepsilon_{\ell,n}
          +C\delta_n\{r_\omega+a_\omega(f)\}+C\delta_n^2.
\end{align*}
Using $C\delta_n r_\omega\le\tfrac14 r_\omega^2+C'\delta_n^2$ and $C\delta_n a_\omega(f)\le\tfrac12a_\omega(f)^2+C'\delta_n^2$, then taking the infimum over $f$ and the supremum over $\omega$, yields
\begin{equation}
 R_{\ell,n}^2\le C\left\{L_\ell^2t_n^{2\nu}+\delta_n^2+\varepsilon_{\ell,n}\right\}\quad\text{on }\cE_{\ell,n}.
 \label{eq:good-training-source}
\end{equation}
This uses \Cref{eq:source-approx} directly. No population minimizer in $\cF_\ell(t_n)$ is needed: take arbitrarily accurate comparators before passing to the infimum. On the complement, $R_{\ell,n}\le2b_\ell$ by the envelope and conditional Jensen. Taking expectations in \Cref{eq:good-training-source} therefore gives
\begin{equation*}
 \E[R_{\ell,n}^2]\le C\left\{t_n^{2\nu}+\delta_n^2+\varepsilon_{\ell,n}+e^{-cn\delta_n^2}\right\}\le C\{v_n+\varepsilon_{\ell,n}\},
\end{equation*}
where the last inequality uses \Cref{eq:tail-source}. If $\delta_n>c_0$, the same conclusion follows from $R_{\ell,n}\le2b_\ell$ after increasing $C$. This proves \Cref{eq:nuisance-source} for every $n$. Since the deviation event is simultaneous over the fitted classes, the conclusion includes all auxiliary fitting randomness.

Set $s_{\ell,n}\coloneqq(v_n+\varepsilon_{\ell,n})^{1/2}$. Then \Cref{eq:nuisance-source} gives $\rho_{\ell,n}\le Cs_{\ell,n}$. Substituting into \eqref{eq:product-source}, and using $\delta_n\le\sqrt{v_n+\varepsilon_n}$ and $s_{\ell,n}\le\sqrt{v_n+\varepsilon_n}$, bounds the braces by $C(v_n+\varepsilon_n)$. This proves \Cref{eq:gradient-source} with expectation over all randomness.

For two-fold cross-fitting, write $\widehat{\nabla\mathcal F}^{\mathrm{OBi}}_2=\ghat$ for the estimate trained on $S_1$, and let $\widehat{\nabla\mathcal F}^{\mathrm{OBi}}_1$ denote the estimate trained on $S_2$. Pointwise in both samples,
\begin{equation*}
 \sup_{\omega\in\Omega}\left\|\frac{\widehat{\nabla\mathcal F}^{\mathrm{OBi}}_1(\omega)+\ghat(\omega)}2-\nabla\mathcal F(\omega)\right\|\le \frac12\sum_{r=1}^2\sup_{\omega\in\Omega}\norm{\widehat{\nabla\mathcal F}^{\mathrm{OBi}}_r(\omega)-\nabla\mathcal F(\omega)}.
\end{equation*}
Apply \Cref{eq:product-source,eq:gradient-source} separately in each training direction and take expectations. Independence of the two resulting estimates is not required.

\step{Step 8: balancing approximation and estimation.}
Under \Cref{eq:source-tuning},
\begin{equation*}
 t_n^{2\nu}=n^{-2\nu/b},\qquad(nt_n^\gamma)^{-2/(2+p)}=n^{-\frac{2}{2+p}(1-\gamma/b)}=n^{-2\nu/b},\qquad\varepsilon_n\le\Bigl(\max_\ell c_\ell\Bigr)n^{-2\nu/b}.
\end{equation*}
Since $1-\gamma/b=\nu(2+p)/b$, the middle identity is $\delta_n^2=n^{-2\nu/b}$, that is, $\delta_n=n^{-\nu/b}$. Substituting into \Cref{eq:nuisance-source,eq:gradient-source} gives \Cref{eq:tuned-source-appendix}, including $\rho_{\ell,n}\le(C_\ell')^{1/2}n^{-\nu/b}$. Finally,
\begin{equation*}
 \frac{2\nu}{\nu(2+p)+\gamma}\ge\frac12\quad\Longleftrightarrow\quad\nu(2-p)\ge\gamma,
\end{equation*}
with strict inequalities equivalent as well. This proves the stated root-$n$ conclusions.
\end{proof}

\subsection{Minimax optimality of OBiGrad}
\label{sec:appendix_minimax}
\Cref{sec:debias} established a finite-sample bound on $\mathfrak{s}=\mathbb{E}[\sup_{\omega\in\Omega}\|\widehat{\nabla\mathcal{F}}^{\mathrm{OBi}}(\omega)-\nabla\mathcal{F}(\omega)\|]$ for the OBiGrad estimator and showed that this error attains the parametric rate $O(n^{-1/2})$ whenever the entropy--approximation exponents of the nuisance classes in \Cref{ass:source,ass:approx} satisfy a suitable condition. We are now interested in establishing a matching lower bound. To avoid dealing with the degenerate case when the population hypergradient $\nabla\mathcal{F}(\omega)$ is identically $0$, we make the following non-degeneracy assumption on the Jacobian $j^\star_{\omega,P}(x)=(j^\star_{1,\omega,P}(x),\ldots,j^\star_{d,\omega,P}(x))$.
 
\begin{assum}[Gradient non-degeneracy]\label{ass:nondeg}
There exists a parameter $\omega_0\in\Omega$ for which the matrix $M_0\coloneqq\mathbb{E}\bigl[j^\star_{\omega_0,P}(X)^\top j^\star_{\omega_0,P}(X)\bigr]$ has strictly positive largest eigenvalue, i.e.\ $\lambda_0\coloneqq\lambda_{\max}(M_0)>0$.
\end{assum}
 
Since $P$ varies in what follows, we write $\mathcal{F}_P$ and $\nabla\mathcal{F}_P$ for the value function \eqref{eq:FBO} and its hypergradient formed under the law $P$. The lower bound is to be compared with the upper bound of \Cref{cor:source-only}, so the two must range over the same laws. \Cref{cor:source-only} operates at the single approximation scale $t_n=n^{-1/b}$, so it is the approximation condition \eqref{eq:source-approx} at that scale that must hold. Writing $\mathcal{P}_Y(P_{XZ})$ for the set of laws $P$ of $(X,Y,Z)$ with $(X,Z)\sim P_{XZ}$ and $\|Y\|\le A$ almost surely, we therefore restrict attention to
\begin{equation}
\label{eq:class-reg}
  \mathcal{P}^{\mathrm{reg}}_Y(P_{XZ};n)\coloneqq\Bigl\{P\in\mathcal{P}_Y(P_{XZ}):\ \inf_{f\in\mathcal{F}_m(t_n)}\|f-m^\star_P\|^2_{\LL}\le L_m^2\,t_n^{2\nu}\Bigr\}.
\end{equation}
Condition \eqref{eq:source-approx} is imposed for each $\ell\in\{h,j,m\}$. Its instances for $\ell\in\{h,j\}$ concern $h^\star_{\omega,P}$ and $j^\star_{\omega,P}$, which depend on $P$ only through $P_{XZ}$. Since $P_{XZ}$ is held fixed, these hold for every member of $\mathcal{P}_Y(P_{XZ})$ or for none, and we assume they hold for the given $P_{XZ}$. The outcome regression $m^\star_P$ is therefore the only nuisance that varies over the class, and \eqref{eq:class-reg} is the one remaining instance, $\ell=m$. Consequently \Cref{cor:source-only} applies to every $P\in\mathcal{P}^{\mathrm{reg}}_Y(P_{XZ};n)$.
 
\begin{proposition}[Lower bound]
\label{prop:minimax-lower-bound}
Fix a law $P_{XZ}$ over $(X,Z)$ on $\mathcal{X}\times\mathcal{Z}$. Let $g$ be a structural map such that \Cref{assum:eif-integrability} holds for $g$ and \Cref{ass:nondeg} holds at some $\omega_0\in\Omega$ for the pair $(g,P_{XZ})$, and assume $0\in\mathcal{F}_m(t)$ for every $t\in(0,1]$ and $L_m\ge AD/\sqrt{\lambda_0}$. Then, for any integer $n\ge D^2/\lambda_0$, it holds
that:
\begin{equation*}
  \inf_{\widehat\Psi}\ \sup_{P\in\mathcal{P}^{\mathrm{reg}}_Y(P_{XZ};n)}\mathbb{E}_P\bigl[\|\widehat\Psi-\nabla\mathcal{F}_P(\omega_0)\|\bigr]\ \ge\ c_0\,A\,\sqrt{\lambda_0}\;n^{-1/2},\qquad c_0\coloneqq\frac{1-\sqrt{1/2}}{4},
\end{equation*}
where the infimum is over all estimators $\widehat\Psi$ measurable with respect to $n$ i.i.d.\ draws from $P$.
\end{proposition}
 
\Cref{prop:minimax-lower-bound} shows that there exists a worst-case problem for which no estimator of the hypergradient at a non-degenerate point $\omega_0$, whose existence is guaranteed by \Cref{ass:nondeg}, can attain a rate faster than the parametric rate. The proof reduces the problem to Bernoulli mean estimation through a one-dimensional submodel and applies a standard Le Cam two-point argument \citep[Chapter~2]{tsybakov2009introduction}. The same lower bound holds for the uniform risk $\sup_{\omega\in\Omega}\|\widehat\Psi_\omega-\nabla\mathcal{F}_P(\omega)\|$ of any estimator of the full target $\nabla\mathcal{F}_P(\omega)$, since this uniform risk dominates the fixed-$\omega_0$ risk considered in \Cref{prop:minimax-lower-bound}. Combined with \Cref{cor:source-only}, which yields an upper bound $O(n^{-1/2})$ on the gradient estimation error whenever $\nu(2-p)\ge\gamma$, \Cref{prop:minimax-lower-bound} shows that this rate matches the lower bound. Thus, in this regime, OBiGrad is minimax rate-optimal for hypergradient estimation over $\mathcal{P}^{\mathrm{reg}}_Y(P_{XZ};n)$.

\begin{proof}[Proof of \Cref{prop:minimax-lower-bound}]
\emph{Step 1 (a perturbation direction and a family of laws $P_b$).} Let $u^\star\in\mathbb{R}^d$ be a unit eigenvector of $M_0$ associated with $\lambda_0$, and set
\begin{equation}
\label{eq:lb-direction}
  \zeta(x)\coloneqq j^\star_{\omega_0,P_{XZ}}(x)\,u^\star\in\mathbb{R}^q,\qquad\varrho(x)\coloneqq\|\zeta(x)\|,\qquad e(x)\coloneqq\begin{cases}\varrho(x)^{-1}\zeta(x),&\zeta(x)\neq0,\\ e_1,&\text{otherwise},\end{cases}
\end{equation}
so that $\langle e(x),\zeta(x)\rangle=\varrho(x)$ and $\|e(x)\|=1$ for every $x$; all three are measurable functions of $x$ alone. By the choice of $u^\star$, and by conditional Jensen together
with \Cref{assum:eif-integrability},
\begin{equation}
\label{eq:lb-facts}
  \|\varrho\|^2_{\LL}=\mathbb{E}\|j^\star_{\omega_0,P_{XZ}}(X)u^\star\|^2=(u^\star)^\top M_0u^\star=\lambda_0,\qquad |\varrho(x)|\le\|j^\star_{\omega_0,P_{XZ}}(x)\|_{2,d}\le D .
\end{equation}
This is the only place \Cref{ass:nondeg} is used, and it is what makes the gap in Step~2 nonzero.
 
For a measurable $b:\mathcal{X}\to[-1,1]$, define $P_b$ to be the law under which $(X,Z)\sim P_{XZ}$ and, conditionally on $(X,Z)$, $Y=A\,\xi\,e(X)$ for a random sign $\xi\in\{-1,+1\}$ with
\begin{equation*}
  \mathbb{P}(\xi=1\mid X,Z)=\frac{1+b(X)}{2}.
\end{equation*}
Two deterministic facts about $Y$ follow immediately. First, since $|\xi|=1$ and $\|e(X)\|=1$, $\|Y\|=A|\xi|\,\|e(X)\|=A$ for every realization. Second, $\mathbb{E}[\xi\mid X,Z]=1\cdot\tfrac{1+b(X)}{2}+(-1)\cdot\tfrac{1-b(X)}{2}=b(X)$ already depends on $X$ alone (not on $Z$), so the tower rule collapses immediately and $m^\star_{P_b}(X)\coloneqq\mathbb{E}[Y\mid X]=A\,b(X)\,e(X)$. Because the $(X,Z)$-law is unchanged, the nuisances $h^\star_{\omega,P_b}(X)=\mathbb{E}[g_\omega(Z)\mid X]$ and $j^\star_{k,\omega,P_b}(X)=\mathbb{E}[\partial_{\omega_k}g_\omega(Z)\mid X]$ are identical, as functions of $X$, to those of $P_{XZ}$, for \emph{every} choice of $b$; in particular $\zeta$, $\varrho$ and $e$ of \eqref{eq:lb-direction} are the same under every $P_b$. The two conditions defining $\mathcal{P}_Y(P_{XZ})$ are therefore both satisfied, so $P_b\in\mathcal{P}_Y(P_{XZ})$.
 
It remains to check membership in $\mathcal{P}^{\mathrm{reg}}_Y(P_{XZ};n)$ for the two members used below. Using $\varrho(x)e(x)=\zeta(x)$, these have $m^\star_{P_{b_0}}=0$ and $m^\star_{P_{b_1}}=(A\tau/D)\,\zeta$ with $\tau$ as in \eqref{eq:lb-twopoints}. As $0\in\mathcal{F}_m(t_n)$ the first is immediate, and for the second, comparing with $f=0$ and using $\|\zeta\|_{\LL}=\sqrt{\lambda_0}$ from \eqref{eq:lb-facts},
\begin{equation*}
  \inf_{f\in\mathcal{F}_m(t_n)}\|f-m^\star_{P_{b_1}}\|^2_{\LL}\le\|m^\star_{P_{b_1}}\|^2_{\LL}=\frac{A^2\tau^2\lambda_0}{D^2}\le\frac{A^2D^2}{\lambda_0}\,n^{-1}\le L_m^2\,t_n^{2\nu},
\end{equation*}
the last two steps by $\tau\le D/(\sqrt{\lambda_0}\sqrt n)$, by $2\nu/b=2\nu/(\nu(2+p)+\gamma)\le1$ so that $n^{-1}\le n^{-2\nu/b}=t_n^{2\nu}$, and by $A^2D^2/\lambda_0\le L_m^2$. Comparing with $f=0$ suffices precisely because the perturbation shrinks at least as fast as $n^{-1/2}$; no approximation power of $\mathcal{F}_m$ beyond containing the origin is required.
 
\emph{Step 2 (the target under $P_b$).} By the direct form $\partial_{\omega_k}\mathcal{F}_P(\omega)=\mathbb{E}_P[\langle h^\star_{\omega,P}(X)-Y,j^\star_{k,\omega,P}(X)\rangle]$: since $h^\star_{\omega_0,P_b}$ and $j^\star_{k,\omega_0,P_b}$ are both already functions of $X$ alone, conditioning on $X$ first replaces $Y$ by its conditional mean $m^\star_{P_b}(X)=Ab(X)e(X)$ from Step~1, and contracting with $u^\star$ and using $\langle e(X),\zeta(X)\rangle=\varrho(X)$,
\begin{align*}
  \bigl\langle u^\star,\nabla\mathcal{F}_{P_b}(\omega_0)\bigr\rangle=&\mathbb{E}\bigl[\bigl\langle h^\star_{\omega_0,P_{XZ}}(X)-m^\star_{P_b}(X),\,\zeta(X)\bigr\rangle\bigr]\\
  =&\underbrace{\mathbb{E}\bigl[\langle h^\star_{\omega_0,P_{XZ}}(X),\zeta(X)\rangle\bigr]}_{\text{independent of }b}
   -\ A\,\mathbb{E}\bigl[b(X)\varrho(X)\bigr].
\end{align*}
The first term does not depend on $b$ at all, so it will cancel exactly once two members of the family are subtracted below; only the second term remains with $b$. The pointwise bound $|\varrho(x)|\le D$ recorded in \eqref{eq:lb-facts} is what keeps the perturbation constructed next inside the legal range $[-1,1]$.
 
Now pick the two specific members of the family $P_b$ in Step~1: $b_0\coloneqq0$ and
\begin{equation}
\label{eq:lb-twopoints}
  b_1\coloneqq\frac{\tau\,\varrho}{D},\qquad\tau\coloneqq\min\Bigl\{1,\ \frac{D}{\sqrt{\lambda_0}\,\sqrt n}\Bigr\},
\end{equation}
the perturbation size $\tau$ being distinct from the approximation scale $t$ of \Cref{ass:source}. Admissibility of $b_1$ is immediate from \eqref{eq:lb-facts}: $|b_1(x)|=\tau|\varrho(x)|/D\le\tau\le1$. Since the $b$-independent term cancels between $b_1$ and $b_0$, the two targets differ, in the direction $u^\star$, by
\begin{equation}
\label{eq:lb-gap}
  \Delta\coloneqq\bigl|\bigl\langle u^\star,\nabla\mathcal{F}_{P_{b_1}}(\omega_0)-\nabla\mathcal{F}_{P_{b_0}}(\omega_0)\bigr\rangle\bigr|=A\,\mathbb{E}[b_1(X)\varrho(X)]=\frac{A\tau\,\|\varrho\|^2_{\LL}}{D}=\frac{A\tau\lambda_0}{D},
\end{equation}
the signed difference being $-A\,\mathbb{E}[b_1\varrho]$. Substituting $\tau$, and pulling the positive constant $\sqrt{\lambda_0}/D$ through the minimum (valid since $c\cdot\min\{a,b\}=\min\{ca,cb\}$ for $c>0$), gives
\begin{equation*}
  \Delta=A\sqrt{\lambda_0}\,\min\Bigl\{n^{-1/2},\ \frac{\sqrt{\lambda_0}}{D}\Bigr\},
\end{equation*}
so that $\Delta=A\sqrt{\lambda_0}\,n^{-1/2}$ for every $n\ge D^2/\lambda_0$.
 
\emph{Step 3 (the divergence between the $n$-fold products).} Since $P_{b_1}$ and $P_{b_0}$ share the identical $(X,Z)$-law $P_{XZ}$ (Step~1), the chain rule for KL leaves only the conditional piece, and since $Y\mid X,Z$ is, under either law, a relabeling of $\mathrm{Bern}(\tfrac{1+b(X)}{2})$ depending on $X$ alone, this reduces to a Bernoulli comparison. KL also tensorizes over the $n$ i.i.d.\ draws, giving
\begin{equation*}
  \mathrm{KL}\bigl(P_{b_1}^{\otimes n}\,\big\|\,P_{b_0}^{\otimes n}\bigr)=n\,\mathbb{E}_X\Bigl[\mathrm{KL}\Bigl(\mathrm{Bern}\bigl(\tfrac{1+b_1(X)}{2}\bigr)\,\Big\|\,\mathrm{Bern}(\tfrac12)\Bigr)\Bigr].
\end{equation*}
Bounding each Bernoulli KL by its chi-squared divergence ($\mathrm{KL}\le\chi^2$) and computing that divergence directly,
\begin{equation*}
  \chi^2\Bigl(\mathrm{Bern}\bigl(\tfrac{1+b}{2}\bigr)\,\Big\|\,\mathrm{Bern}\bigl(\tfrac12\bigr)\Bigr)=b^2,
\end{equation*}
then substituting $b_1=\tau\varrho/D$, gives
\begin{equation}
\label{eq:lb-kl}
  \mathrm{KL}\bigl(P_{b_1}^{\otimes n}\,\big\|\,P_{b_0}^{\otimes n}\bigr)\le n\,\mathbb{E}[b_1(X)^2]=\frac{n\,\tau^2\lambda_0}{D^2}\ \le\ 1,
\end{equation}
where the last step holds by exactly how $\tau$ was chosen in Step~2, in either branch of the minimum. So no matter how large $n$ is, $P_{b_1}^{\otimes n}$ and $P_{b_0}^{\otimes n}$ stay within KL divergence $1$ of each other even though their targets differ by the non-vanishing gap $\Delta$ from Step~2.
 
\emph{Step 4.} Write $\widehat\Psi_{u^\star}\coloneqq\langle u^\star,\widehat\Psi\rangle$ and $\theta_i\coloneqq\langle u^\star,\nabla\mathcal{F}_{P_{b_i}}(\omega_0)\rangle$, so that $|\theta_1-\theta_0|=\Delta$ by \eqref{eq:lb-gap}. We show
\begin{equation}
\label{eq:lb-lecam-max}
  \max_{i\in\{0,1\}}\ \mathbb{P}^{\otimes n}_{b_i}\bigl(|\widehat\Psi_{u^\star}-\theta_i|\ge\Delta/2\bigr)\ \ge\ \tfrac12\bigl(1-\sqrt{1/2}\bigr).
\end{equation}
Define the auxiliary set $\mathcal{A}\coloneqq\{y\in\mathbb{R}:|y-\theta_0|<\Delta/2\}$, the symbol being distinct from the constant $A$ bounding $\|Y\|$. For $y\in\mathcal{A}$, the triangle inequality together with $|\theta_0-\theta_1|=\Delta$ (Step~2) gives
\begin{equation*}
  |y-\theta_1|\ \ge\ \Delta-|y-\theta_0|\ >\ \Delta-\Delta/2\ =\ \Delta/2,
\end{equation*}
so $\mathcal{A}\subseteq\{y\in\mathbb{R}:|y-\theta_1|\ge\Delta/2\}$. Since $\widehat\Psi_{u^\star}$ is a measurable function of the sample, the total variation distance between its laws under $\mathbb{P}^{\otimes n}_{b_0}$ and $\mathbb{P}^{\otimes n}_{b_1}$ is at most $\mathrm{TV}(P_{b_0}^{\otimes n},P_{b_1}^{\otimes n})$. Evaluating at the set $\mathcal{A}$, and applying Pinsker's inequality to Step~3's bound \eqref{eq:lb-kl} --- legitimate in this order because $\mathrm{TV}$ is symmetric while $\mathrm{KL}$ is not --- gives
\begin{equation}
\label{eq:lb-lecam}
  \mathbb{P}^{\otimes n}_{b_0}(\widehat\Psi_{u^\star}\in\mathcal{A})-\mathbb{P}^{\otimes n}_{b_1}(\widehat\Psi_{u^\star}\in\mathcal{A})\ \le\ \mathrm{TV}\bigl(P_{b_1}^{\otimes n},P_{b_0}^{\otimes n}\bigr)\ \le\ \sqrt{\tfrac12\,\mathrm{KL}\bigl(P_{b_1}^{\otimes n}\|P_{b_0}^{\otimes n}\bigr)}\ \le\ \sqrt{1/2}.
\end{equation}
By definition of $\mathcal{A}$, $\mathbb{P}^{\otimes n}_{b_0}(\widehat\Psi_{u^\star}\in\mathcal{A})=1-\mathbb{P}^{\otimes n}_{b_0}(|\widehat\Psi_{u^\star}-\theta_0|\ge\Delta/2)$, and using $\mathcal{A}\subseteq\{y:|y-\theta_1|\ge\Delta/2\}$, $\mathbb{P}^{\otimes n}_{b_1}(\widehat\Psi_{u^\star}\in\mathcal{A}) \le\mathbb{P}^{\otimes n}_{b_1}(|\widehat\Psi_{u^\star}-\theta_1|\ge\Delta/2)$. Substituting both into \eqref{eq:lb-lecam} and rearranging gives
\begin{equation*}
  \mathbb{P}^{\otimes n}_{b_0}\bigl(|\widehat\Psi_{u^\star}-\theta_0|\ge\Delta/2\bigr)+\mathbb{P}^{\otimes n}_{b_1}\bigl(|\widehat\Psi_{u^\star}-\theta_1|\ge\Delta/2\bigr)\ \ge\ 1-\sqrt{1/2},
\end{equation*}
and since $\max(x,y)\ge(x+y)/2$ for any $x,y\ge0$, \eqref{eq:lb-lecam-max} follows.
 
\emph{Step 5.} Let $i^\star$ denote an index attaining the maximum in \eqref{eq:lb-lecam-max}.
Since $\|u^\star\|=1$, Cauchy--Schwarz gives $\|\widehat\Psi-\nabla\mathcal{F}_P(\omega_0)\|\ge|\widehat\Psi_{u^\star}-\langle u^\star,\nabla\mathcal{F}_P(\omega_0)\rangle|$, so by Markov's inequality applied at level $\Delta/2$ under $\mathbb{P}^{\otimes n}_{b_{i^\star}}$,
\begin{equation*}
  \mathbb{E}_{P_{b_{i^\star}}}\bigl[\|\widehat\Psi-\nabla\mathcal{F}_{P_{b_{i^\star}}}(\omega_0)\|\bigr]\ \ge\ \frac{\Delta}{2}\cdot\mathbb{P}^{\otimes n}_{b_{i^\star}}\bigl(|\widehat\Psi_{u^\star}-\theta_{i^\star}|\ge\Delta/2\bigr)\ \ge\ \frac{\Delta}{2}\cdot\frac{1-\sqrt{1/2}}{2}\ =\ c_0\Delta .
\end{equation*}
Since $P_{b_0},P_{b_1}\in\mathcal{P}^{\mathrm{reg}}_Y(P_{XZ};n)$ (Step~1):
\begin{equation*}
  \sup_{P\in\mathcal{P}^{\mathrm{reg}}_Y(P_{XZ};n)}\mathbb{E}_P\bigl[\|\widehat\Psi-\nabla\mathcal{F}_P(\omega_0)\|\bigr]\ \ge\ \max_{i\in\{0,1\}}\\mathbb{E}_{P_{b_i}}\bigl[\|\widehat\Psi-\nabla\mathcal{F}_{P_{b_i}}(\omega_0)\|\bigr]\ \ge\ c_0\Delta .
\end{equation*}
Substituting $\Delta$ from Step~2 gives the stated bound, for every estimator $\widehat\Psi$ and hence for the infimum. The uniform-risk version follows since $\sup_\omega\|\widehat\Psi_\omega-\nabla\mathcal{F}_P(\omega)\|\ge\|\widehat\Psi_{\omega_0}-\nabla\mathcal{F}_P(\omega_0)\|$ pointwise.
\end{proof}

\section{Nuisance Learning and Plug-in Hypergradients}
\label{sec:appendix_nuisance_plugin}

\subsection{Nuisance learning}
\label{sec:nuisance-learning-discussion}

This appendix records standard sufficient conditions under which the nuisance assumptions used in Section~\ref{sec:debiased-clt} can be verified. We recall:
\begin{equation*}
    h_{\omega,P}^\star(X)=\mathbb{E}_P[g_\omega(Z)\mid X],\qquad j_{k,\omega,P}^\star(X)=\mathbb{E}_P[\partial_{\omega_k}g_\omega(Z)\mid X],\qquad m^\star_P(X)=\mathbb{E}_P[Y\mid X].
\end{equation*}
Thus nuisance learning reduces to standard regression. For example, define
\begin{equation*}
    R_\omega^h(h)=\frac12\mathbb{E}_P[\|h(X)-g_\omega(Z)\|^2],\qquad R_\omega^j(j)=\frac12\mathbb{E}_P[\|j(X)-\partial_\omega g_\omega(Z)\|_{2,d}^2],
\end{equation*}
and $R^m(m)=\frac12\mathbb{E}_P\|m(X)-Y\|^2$. Their minimizers are respectively $h_{\omega,P}^\star$, $j_{\omega,P}^\star$, and $m^\star_P$.

A generic verification route is the following. Let $\mathcal H_n,\mathcal J_n,\mathcal M_n$ be the classes used to estimate $h_{\omega,P}^\star$, $j_{\omega,P}^\star$, and $m^\star_P$. Let $h_{n,\omega}^\circ$, $j_{n,\omega}^\circ$, and $m_n^\circ$ denote the corresponding population risk minimizers over these classes. For squared loss,
\begin{equation*}
    R_\omega^h(h)-R_\omega^h(h_\omega^\star)=\frac12\|h-h_\omega^\star\|_{\LL}^2,
\end{equation*}
and similarly for \(j\) and \(m\). Hence an oracle inequality such as $R_\omega^h(\hat h_\omega)-R_\omega^h(h_{n,\omega}^\circ)=O_p(\delta_{h,n}^2)$ implies
\begin{equation*}
    \|\hat h_\omega-h_\omega^\star\|_{\LL}=O_p(a_{h,n}+\delta_{h,n}),\qquad a_{h,n}:=\|h_{n,\omega}^\circ-h_\omega^\star\|_{\LL}.
\end{equation*}
Uniformly over \(\omega\), the same argument gives
\begin{equation*}
    \sup_{\omega\in\Omega}\|\hat h_\omega-h_\omega^\star\|_{\LL}=O_p(a_{h,n}+\delta_{h,n})
\end{equation*}
whenever the oracle inequality is uniform in \(\omega\). Analogous bounds hold for $\hat j_\omega$ and $\hat m$. For the expected uniform guarantees under \Cref{assum:eif-integrability,ass:source,ass:approx,ass:erm}, \eqref{eq:nuisance-source} supplies the required uniform second-moment bounds. Such oracle inequalities are standard consequences of empirical-process bounds for least-squares estimators over entropy- or Rademacher-controlled classes \citep{vanderVaart1996weak,vanDeGeer2000empirical,kosorok2008introduction,wainwright2019high}.

The rate $\delta_n$ depends on the size of the regression class. A common formulation uses a localized critical radius. If $\delta_n$ satisfies a local Rademacher complexity inequality of the form
\begin{equation*}
    \mathfrak R_n\bigl(\mathcal F_n\cap\{f:\|f-f_n^\circ\|_{\LL}\le \delta_n\}\bigr)\lesssim \delta_n^2,
\end{equation*}
then least-squares ERM over \(\mathcal F_n\) typically satisfies $\|\hat f-f^\star\|_{\LL}=O_p(a_n+\delta_n)$, up to standard boundedness and measurability conditions \citep{bartlett2005local,wainwright2019high,vanDerLaan2026researchers}. For classes with metric entropy $\log N(\varepsilon,\mathcal F_n,L_2)\lesssim \varepsilon^{-p}$, $p\in(0,2)$, the corresponding nonparametric norm rate is typically $n^{-1/(2+p)}$ modulo approximation error. Since $p<2$ gives $1/(2+p)>1/4$, such rates are compatible with the second-order conditions in Assumption~\ref{assum:second-order-rates}, provided the approximation terms are negligible.

Fixed-$\omega$ inference in Appendix~\ref{sec:appendix_CLT} uses \Cref{assum:second-order-rates}: each nuisance error is $o_p(N^{-1/4})$, so the relevant cross-products and squared errors are $o_p(N^{-1/2})$. In particular, the mixed-bias bound is controlled by $\|\hat j_\omega-j^\star_{\omega,P}\|_{\LL^d}(\|\hat h_\omega-h^\star_{\omega,P}\|_{\LL}+\|\hat m-m_P^\star\|_{\LL})$. Standard least-squares, entropy, and Rademacher arguments can verify these rates under stronger versions of the regression conditions above \citep{chernozhukov2018double,foster2023orthogonal,vanDerLaan2025automatic}. The empirical remainder is handled by the linear-plus-product decomposition in Remark~\ref{assum:score-stability}.

For the uniform result in Section~\ref{sec:debias}, one also needs complexity control of the score classes generated by the learned nuisance paths. Assumptions~\ref{ass:maximal-bounded-nuisance-classes} and~\ref{ass:maximal-nuisance-entropy} give a sufficient condition: conditionally on the training fold, the learned nuisances lie in bounded classes with polynomial entropy. For kernel ridge regression with bounded kernel $K(x,x)\le\kappa^2$ and ridge parameter $\lambda_n>0$, comparison with the zero function gives $\lambda_n\|\hat f\|_{\mathcal H_K}^2\le C^2$ when the regression target is bounded by $C$. Hence $\|\hat f\|_{\mathcal H_K}\le C/\sqrt{\lambda_n}$ and $\|\hat f\|_\infty\le \kappa C/\sqrt{\lambda_n}$. Entropy or Rademacher bounds for the induced RKHS ball then yield the required class-complexity control \citep{scholkopf2001generalized,steinwart2008support,caponnetto2007optimal}. For neural-network classes with bounded depth, width, activation Lipschitz constants, and weight norms, standard covering-number bounds give entropy of order $s_n\log(C_n/\varepsilon)$, where $s_n$ is the number of effective parameters and $C_n$ depends on the norm constraints. If the architecture grows with $n$, these constants grow accordingly and must be small enough for the localized critical-radius terms in Theorem~\ref{thm:quadratic-maximal-inequality} to vanish.

For nonquadratic inner problems, additional Riesz or adjoint correction nuisances may appear. The same verification principle applies: one needs oracle inequalities for the additional nuisance learners and product-rate conditions matching the corresponding von Mises remainder \citep{vanDerLaan2025automatic}.

\subsection{Plug-in hypergradient estimation}
\label{sec:plugin-appendix}

This section records what the direct plug-in hypergradient estimates in the quadratic specialization, and why it should not be expected to deliver semiparametric-efficient inference. Functional implicit differentiation and KBO both lead to plug-in hypergradients in this broad sense: one estimates the inner solution, or an adjoint/sensitivity representation of it, and substitutes the estimate into a gradient formula \citep{petrulionyte2024functional,elkhoury2025learning}. This is natural for optimization. It is not, by itself, an orthogonal statistical estimator.

For clarity, we analyze the sample-split plug-in estimator. Let $\hat h_\omega$ and $\hat j_\omega=(\hat j_{1,\omega},\dots,\hat j_{d,\omega})$ be trained on $S_1$ and evaluated on the independent fold $S_2$:
\begin{equation}
\label{eq:plugin-estimator}
    [\widehat{\nabla\mathcal F}^{\mathrm{PI}}(\omega)]_k\coloneqq \mathbb{E}_{P_{n,2}}\big[\langle \hat h_\omega(X)-Y,\hat j_{k,\omega}(X)\rangle\big],\qquad k=1,\dots,d.
\end{equation}
The full-sample version replaces $P_{n,2}$ by $P_n$, but then the empirical-process term involves a function estimated on the same sample on which it is averaged. The sample-split version isolates the statistical issue cleanly.

\paragraph{Exact plug-in bias.}
For arbitrary $h,v\in\LL$, define
\begin{equation*}
    \varphi^{\mathrm{PI}}_{k,\omega}(X,Y,Z;h,v)\coloneqq\langle h(X)-Y,v(X)\rangle.
\end{equation*}
Since $h(X)$ and $v(X)$ are functions of $X$,
\begin{equation*}
    \mathbb{E}_P\bigl[\varphi^{\mathrm{PI}}_{k,\omega}(\cdot;h,v)\bigr]=\langle h-m^\star_P,v\rangle_{\LL}.
\end{equation*}
Also,
\begin{equation*}
    \partial_{\omega_k}\mathcal F(\omega)=\langle h^\star_{\omega,P}-m^\star_P,j^\star_{k,\omega,P}\rangle_{\LL}.
\end{equation*}
Therefore
\begin{equation}
\label{eq:plugin-exact-bias}
\begin{aligned}
    &\mathbb{E}_P\bigl[\varphi^{\mathrm{PI}}_{k,\omega}(\cdot;h,v)\bigr]-\partial_{\omega_k}\mathcal F(\omega)\\
    &\quad =\langle h-h^\star_{\omega,P},j^\star_{k,\omega,P}\rangle_{\LL}+\langle h^\star_{\omega,P}-m^\star_P,v-j^\star_{k,\omega,P}\rangle_{\LL}+\langle h-h^\star_{\omega,P},v-j^\star_{k,\omega,P}\rangle_{\LL}.
\end{aligned}
\end{equation}
The first two terms are first-order nuisance errors. Thus the plain plug-in hypergradient is generally not Neyman-orthogonal with respect to $h^\star_\omega$ or $j^\star_{k,\omega}$.

In particular, if $\|j^\star_{k,\omega,P}\|_{\LL}<\infty$ and $\|h^\star_{\omega,P}-m^\star_P\|_{\LL}<\infty$, then
\begin{equation}
\label{eq:plugin-bias-bound}
\begin{aligned}
    \big|\mathbb{E}_P[\varphi^{\mathrm{PI}}_{k,\omega}(\cdot;h,v)]-\partial_{\omega_k}\mathcal F(\omega)\big|&\le\|j^\star_{k,\omega,P}\|_{\LL}\|h-h^\star_{\omega,P}\|_{\LL}\\
    &\quad+\|h^\star_{\omega,P}-m^\star_P\|_{\LL}\|v-j^\star_{k,\omega,P}\|_{\LL}\\
    &\quad+\|h-h^\star_{\omega,P}\|_{\LL}\|v-j^\star_{k,\omega,P}\|_{\LL}.
\end{aligned}
\end{equation}

\begin{proposition}[Rate of the sample-split plug-in estimator]
\label{prop:plugin-rate}
Fix $\omega\in\Omega$ and $k\in\{1,\dots,d\}$. Suppose $\hat h_\omega$ and $\hat j_{k,\omega}$ are trained on $S_1$, and suppose
\begin{equation*}
    \mathbb{E}_P\big[\varphi^{\mathrm{PI}}_{k,\omega}(\cdot;\hat h_\omega,\hat j_{k,\omega})^2\big]=O_p(1).
\end{equation*}
If $\|\hat h_\omega-h^\star_{\omega,P}\|_{\LL}=O_p(n^{-\alpha})$ and $\|\hat j_{k,\omega}-j^\star_{k,\omega,P}\|_{\LL}=O_p(n^{-\beta})$, then
\begin{equation*}
    [\widehat{\nabla\mathcal F}^{\mathrm{PI}}(\omega)]_k-\partial_{\omega_k}\mathcal F(\omega)=O_p\left(n^{-1/2}+n^{-\alpha}+n^{-\beta}+n^{-(\alpha+\beta)}\right).
\end{equation*}
\end{proposition}

\begin{proof}
Conditionally on $S_1$, the function $\varphi^{\mathrm{PI}}_{k,\omega}(\cdot;\hat h_\omega,\hat j_{k,\omega})$ is fixed, and the evaluation observations in $S_2$ are independent. Hence
\begin{equation*}
    (\mathbb{E}_{P_{n,2}}-\mathbb{E}_P)[\varphi^{\mathrm{PI}}_{k,\omega}(\cdot;\hat h_\omega,\hat j_{k,\omega})]=O_p(n^{-1/2}).
\end{equation*}
The population bias is bounded by \eqref{eq:plugin-bias-bound}. Substituting the assumed nuisance rates gives the result.
\end{proof}

This proposition should be read literally. For the direct plug-in estimator, the generic bias is not a product of nuisance errors. Under the score-moment condition above, a sufficient condition for root-$n$ consistency is that the sum of the bias terms in \eqref{eq:plugin-exact-bias} is $O_p(n^{-1/2})$. The centered oracle-score expansion below instead requires this sum to be $o_p(n^{-1/2})$, together with score stability. Bounding the first-order terms separately gives sufficient nuisance rates of essentially parametric order, but cancellations may relax this requirement. A product-rate condition such as $\alpha+\beta>1/2$ is not enough by itself unless additional structure controls the linear terms.

\paragraph{Asymptotic distribution and efficiency.}
If the plug-in bias in \eqref{eq:plugin-exact-bias} is $o_p(n^{-1/2})$ and
\begin{equation*}
    \bigl\|\varphi^{\mathrm{PI}}_{k,\omega}(\cdot;\hat h_\omega,\hat j_{k,\omega})-\varphi^{\mathrm{PI}}_{k,\omega}(\cdot;h^\star_{\omega,P},j^\star_{k,\omega,P})\bigr\|_{L^2(P)}=o_p(1),
\end{equation*}
then the sample-split plug-in estimator admits an asymptotic linear representation with oracle-score summand
\begin{equation*}
    \varphi^{\mathrm{PI}}_{k,\omega}(\cdot;h^\star_{\omega,P},j^\star_{k,\omega,P})-\partial_{\omega_k}\mathcal F(\omega).
\end{equation*}
Consequently,
\begin{equation*}
    \sqrt n\left([\widehat{\nabla\mathcal F}^{\mathrm{PI}}(\omega)]_k-\partial_{\omega_k}\mathcal F(\omega)\right)\xrightarrow{d}\mathcal N\left(0,\operatorname{Var}_{P}\left[\langle h^\star_{\omega,P}(X)-Y,j^\star_{k,\omega,P}(X)\rangle\right]\right).
\end{equation*}
The oracle-score summand generally differs from the canonical gradient in Theorem~\ref{thm:EIF}, which includes additional correction terms. Its asymptotic variance need not equal $\Sigma_{P,\omega,kk}$, the semiparametric efficiency bound.

\paragraph{Contrast with the debiased estimator.}
In the same quadratic specialization, the debiased score used in Section~\ref{sec:debiased-clt} is
\begin{equation*}
    \varphi_{k,\omega}(X,Y,Z;h,j,m)=\langle g_\omega(Z)-Y,j_k(X)\rangle+\langle \partial_{\omega_k}g_\omega(Z)-j_k(X),h(X)-m(X)\rangle.
\end{equation*}
For arbitrary $h,j,m$, its population bias satisfies
\begin{equation}
\label{eq:dr-product-bias}
    \mathbb{E}_P\bigl[\varphi_{k,\omega}(\cdot;h,j,m)\bigr]-\partial_{\omega_k}\mathcal F(\omega)=-\langle j_k-j^\star_{k,\omega,P},h-h^\star_{\omega,P}\rangle_{\LL}+\langle j_k-j^\star_{k,\omega,P},m-m^\star_P\rangle_{\LL}.
\end{equation}
Thus
\begin{equation*}
    \left|\mathbb{E}_P\bigl[\varphi_{k,\omega}(\cdot;h,j,m)\bigr]-\partial_{\omega_k}\mathcal F(\omega)\right|\le\|j_k-j^\star_{k,\omega,P}\|_{\LL}\left(\|h-h^\star_{\omega,P}\|_{\LL}+\|m-m^\star_P\|_{\LL}\right).
\end{equation*}
The first-order terms present in the plug-in bias \eqref{eq:plugin-exact-bias} are absent. This is the practical content of orthogonality: the estimator can be root-$n$ and asymptotically normal under product-rate nuisance conditions, such as
\begin{equation*}
    \|\hat j_{k,\omega}-j^\star_{k,\omega}\|_{\LL}\left(\|\hat h_\omega-h^\star_{\omega,P}\|_{\LL}+\|\hat m-m^\star_P\|_{\LL}\right)=o_p(n^{-1/2}),
\end{equation*}
rather than requiring each nuisance to be estimated at a parametric rate. Rates at the $n^{-1/4}$ boundary are not enough as big-$O$ statements; the product must be little-$o(n^{-1/2})$ for the centered CLT in Theorem~\ref{thm:CLT}.

This is the core distinction. For the plug-in hypergradient, first-order nuisance errors can obstruct root-$n$ inference. The debiased estimator changes the score so that those first-order terms cancel, leaving only second-order nuisance products and the empirical average of the efficient influence function.

\subsection{Uniform plug-in maximal inequality}
\label{sec:appendix_plugin_uniform}

Proposition~\ref{prop:plugin-rate} controls the plug-in estimator at a single fixed $\omega$. We now give its uniform-over-$\Omega$ analogue, paralleling Theorem~\ref{thm:quadratic-maximal-inequality}. Stack the coordinate-wise plug-in bias identity \eqref{eq:plugin-exact-bias} into the vector
\begin{equation*}
    \varphi^{\mathrm{PI}}_\omega(X,Y;h,j)\coloneqq\bigl(\langle h(X)-Y,j_1(X)\rangle,\ldots,\langle h(X)-Y,j_d(X)\rangle\bigr)^\top,
\end{equation*}
using a shared nuisance pair $(h,j)=(h,(j_1,\ldots,j_d))\in\LL\times\LL^d$, with $\hat h_\omega,\hat j_\omega$ trained on $S_1$ and evaluated on $S_2$, as in \Cref{sec:debiased-clt}. Define the one-fold plug-in estimator
\begin{equation*}
    \widehat{\nabla\mathcal F}^{\mathrm{PI}}(\omega)\coloneqq \mathbb{E}_{P_{n,2}}[\,\varphi^{\mathrm{PI}}_\omega(\cdot;\hat h_\omega,\hat j_\omega)].
\end{equation*}
Since $\varphi^{\mathrm{PI}}_\omega$ does not depend on $g_\omega$, $\partial_\omega g_\omega$, or the regression nuisance $m$, the following result needs only the restrictions of \Cref{ass:maximal-bounded-nuisance-classes,ass:maximal-population-entropy,ass:maximal-nuisance-entropy} to the classes $\mathcal{H}_n,\mathcal{J}_n,\mathcal{H}^\star,\mathcal{J}^\star$ (dropping $\mathcal{G}_\Omega,\dot{\mathcal{G}}_\Omega,\mathcal{M}_n$, which the plug-in score does not use).

\begin{theorem}[Uniform plug-in maximal inequality]
\label{thm:plugin-maximal-inequality}
Under \Cref{assum:eif-integrability} and the restrictions of \Cref{ass:maximal-bounded-nuisance-classes,ass:maximal-population-entropy,ass:maximal-nuisance-entropy} described above, there exist constants $C^{\mathrm{PI}}>0$ depending only on $p$ and $K_A^{\mathrm{PI}}>0$ depending only on the population entropy constants of $\mathcal{H}^\star,\mathcal{J}^\star$ such that, on the event $\sup_{\omega\in\Omega}\|\hat h_\omega-h^\star_{\omega,P}\|_{\LL}\leq r_h$, $\sup_{\omega\in\Omega}\|\hat j_\omega-j^\star_{\omega,P}\|_{\LL^d}\leq r_j$,
\begin{align}
\label{eq:plugin-main-uniform-bound}
    &\mathbb{E}\left[\sup_{\omega\in\Omega}\|\widehat{\nabla\mathcal F}^{\mathrm{PI}}(\omega)-\nabla\mathcal{F}(\omega)\|\,\middle|\,S_1\right]\nonumber\\
    &\leq\frac{C^{\mathrm{PI}}}{\sqrt{n}}(A+B)D\sqrt{1+K_A^{\mathrm{PI}}}+\mathfrak{C}_{B,n}^{\mathrm{PI}}(r_h,r_j)+Dr_h+(A+B)r_j+r_hr_j,
\end{align}
where $\mathfrak{C}_{B,n}^{\mathrm{PI}}(r_h,r_j)$ is the localized nuisance-empirical-process complexity of the plug-in nuisance-difference class.
\end{theorem}

\begin{proof}
We follow the six-step structure of the proof of Theorem~\ref{thm:quadratic-maximal-inequality} (\Cref{sec:appendix_quadratic_maximal}), noting exactly where the plug-in score departs from the debiased score.

\emph{Step 1 (population identities)} is unchanged: $\|h^\star_{\omega,P}(X)\|\leq B$, $\|Y\|\leq A$, $\|j^\star_{\omega,P}(X)\|_{2,d}\leq D$.

\emph{Step 2 (exact bias identity)} is given by \eqref{eq:plugin-exact-bias}, summed over coordinates via Cauchy--Schwarz exactly as in Step~2 of the proof of Theorem~\ref{thm:quadratic-maximal-inequality}:
\begin{align}
\label{eq:plugin-quadratic-product-bias-bound}
    \bigl\|\mathbb{E}_P[\varphi^{\mathrm{PI}}_\omega(\cdot;h,j)]-\nabla\mathcal{F}(\omega)\bigr\|\leq& D\|h-h^\star_{\omega,P}\|_{\LL}+(A+B)\|j-j^\star_{\omega,P}\|_{\LL^d}\nonumber\\
    &+\|h-h^\star_{\omega,P}\|_{\LL}\|j-j^\star_{\omega,P}\|_{\LL^d}.
\end{align}
Unlike the debiased bias bound \eqref{eq:quadratic-product-bias-bound}, this is dominated by the two \emph{first-order} terms $Dr_h+(A+B)r_j$, not by a product alone -- the algebraic source of the rate difference with Theorem~\ref{thm:quadratic-maximal-inequality}.

\emph{Step 3 (estimator decomposition)} is unchanged: writing $\varphi^{\mathrm{PI},\star}_\omega\coloneqq\varphi^{\mathrm{PI}}_\omega(\cdot;h^\star_{\omega,P},j^\star_{\omega,P})$ and $\hat\varphi^{\mathrm{PI}}_\omega\coloneqq\varphi^{\mathrm{PI}}_\omega(\cdot;\hat h_\omega,\hat j_\omega)$,
\begin{equation*}
    \widehat{\nabla\mathcal F}^{\mathrm{PI}}(\omega)-\nabla\mathcal{F}(\omega)=(\mathbb{E}_{P_{n,2}}-\mathbb{E}_P)[\varphi^{\mathrm{PI},\star}_\omega]+(\mathbb{E}_{P_{n,2}}-\mathbb{E}_P)[\hat\varphi^{\mathrm{PI}}_\omega-\varphi^{\mathrm{PI},\star}_\omega]+\bigl(\mathbb{E}_P[\hat\varphi^{\mathrm{PI}}_\omega]-\nabla\mathcal{F}(\omega)\bigr),
\end{equation*}
and the dual-representation reduction to scalar empirical processes over $(\omega,u)\in\Omega\times S^{d-1}$ proceeds identically, using \eqref{eq:plugin-quadratic-product-bias-bound} for the population-bias term.

\emph{Step 4 (envelopes).} By Cauchy--Schwarz, $\|\varphi^{\mathrm{PI},\star}_\omega(X,Y)\|\leq\|h^\star_{\omega,P}(X)-Y\|\,\|j^\star_{\omega,P}(X)\|_{2,d}\leq(A+B)D$, and since $\nabla\mathcal{F}(\omega)=\mathbb{E}_P[\varphi^{\mathrm{PI},\star}_\omega]$, Jensen's inequality gives $\|\nabla\mathcal{F}(\omega)\|\leq(A+B)D$; thus $F_A^{\mathrm{PI}}\coloneqq2(A+B)D$ is an envelope for the oracle score class $\mathcal{F}_A^{\mathrm{PI}}\coloneqq\{u^\top(\varphi^{\mathrm{PI},\star}_\omega-\nabla\mathcal{F}(\omega)):\omega\in\Omega,u\in S^{d-1}\}$. For the nuisance-difference class, writing $\Delta h=h-h^\star_{\omega,P}$, $\Delta j=j-j^\star_{\omega,P}$,
\begin{align*}
    \varphi^{\mathrm{PI}}_{k,\omega}(X,Y;h,j)-\varphi^{\mathrm{PI},\star}_{k,\omega}(X,Y)=&\langle\Delta h(X),j^\star_{k,\omega,P}(X)\rangle+\langle\Delta h(X),\Delta j_k(X)\rangle\\
    &+\langle h^\star_{\omega,P}(X)-Y,\Delta j_k(X)\rangle,
\end{align*}
a sum of three bounded products of the same form as those already controlled in Step~5 of the proof of Theorem~\ref{thm:quadratic-maximal-inequality} (there, the analogous terms bounded the second and third pieces of the debiased score difference). With $\bar H_n,\bar J_n$ the envelopes of $\mathcal{H}_n,\mathcal{J}_n$, this gives an $S_1$-measurable envelope $F_{B,n}^{\mathrm{PI}}\coloneqq(\bar H_n+A)\bar J_n+(A+B)D$ for the plug-in nuisance-difference class $\mathcal{F}_{B,n}^{\mathrm{PI},\rm loc}(r_h,r_j)$, with $L^2(P)$ radius at most $\rho_{B,n}^{\mathrm{PI}}\coloneqq Dr_h+\bar H_nr_j+(A+B)r_j$ on the event of Theorem~\ref{thm:plugin-maximal-inequality}.

\emph{Step 5 (entropy transfer).} The oracle class satisfies $\mathcal{F}_A^{\mathrm{PI}}\subset(\mathcal{U}_h\odot\mathscr{J}^\star)-\mathcal{C}_\Psi$, where $\mathcal{U}_h\coloneqq\mathcal{H}^\star-Y$ with envelope $A+B$, and $\mathscr{J}^\star=\mathscr{S}(\mathcal{J}^\star)$ with envelope $D$. A single application of Lemma~\ref{lem:bounded-product-entropy} (in place of the two-product combination needed there) transfers the polynomial entropy bound of \Cref{ass:maximal-population-entropy} on $\mathcal{H}^\star,\mathcal{J}^\star$ to $\mathcal{F}_A^{\mathrm{PI}}$ with exponent $p$ and a constant $K_A^{\mathrm{PI}}<\infty$ depending only on $K_{\mathcal{H}^\star},K_{\mathcal{J}^\star}$. The nuisance-difference class transfers analogously from \Cref{ass:maximal-nuisance-entropy} applied to $\mathcal{H}_n,\mathcal{J}_n$, via Lemmas~\ref{lem:entropy-sums} and~\ref{lem:bounded-product-entropy} applied to the three-term expansion above, giving an $S_1$-measurable constant $K_{B,n}^{\mathrm{PI}}<\infty$ with the same exponent $p$.

\emph{Step 6 (conclusion).} Corollary~\ref{cor:polynomial-entropy-maximal} applied to $\mathcal{F}_A^{\mathrm{PI}}$ gives the oracle term $\frac{C^{\mathrm{PI}}}{\sqrt n}(A+B)D\sqrt{1+K_A^{\mathrm{PI}}}$ with $C^{\mathrm{PI}}=C_p$; the localized-Rademacher argument of \Cref{sec:pr_cor}, applied verbatim to $\mathcal{F}_{B,n}^{\mathrm{PI},\rm loc}(r_h,r_j)$ (which depends only on its envelope, entropy exponent $p$, and $L^2(P)$ radius, all of which transfer identically), defines $\mathfrak{C}_{B,n}^{\mathrm{PI}}(r_h,r_j)$; and \eqref{eq:plugin-quadratic-product-bias-bound} gives the remaining bias terms $Dr_h+(A+B)r_j+r_hr_j$. Combining as in Step~3 gives \eqref{eq:plugin-main-uniform-bound}.
\end{proof}

\begin{corollary}[Plug-in polynomial-entropy rate]
\label{cor:plugin-polynomial-rate}
Under the conditions of Theorem~\ref{thm:plugin-maximal-inequality}, suppose the nuisance envelopes and entropy constants remain bounded and $r_h,r_j=O(n^{-1/(2+p)})$. Then, on the same training event,
\begin{equation}
\label{eq:plugin-polynomial-rate}
    \mathbb{E}\left[\sup_{\omega\in\Omega}\|\widehat{\nabla\mathcal F}^{\mathrm{PI}}(\omega)-\nabla\mathcal{F}(\omega)\|\,\middle|\,S_1\right]=O(n^{-1/(2+p)}).
\end{equation}
\end{corollary}

\begin{proof}
As in the proof of Proposition~\ref{cor:localized-polynomial-rate} (\Cref{sec:pr_cor}), $\mathfrak{C}_{B,n}^{\mathrm{PI}}(r_h,r_j)=o(n^{-1/2})$ under these conditions, since that argument only used the envelope, entropy exponent, and $L^2(P)$ radius of the localized class, all of which are the same order here. The first-order population-bias terms satisfy $Dr_h+(A+B)r_j+r_hr_j=O(n^{-1/(2+p)})$. Since $1/(2+p)<1/2$, this upper bound determines the displayed rate $O(n^{-1/2}+n^{-1/(2+p)})=O(n^{-1/(2+p)})$; it does not rule out faster decay through cancellation.
\end{proof}

Under the Hölder parametrization $p=d_X/s$ and the conditions of \Cref{cor:localized-polynomial-rate,cor:plugin-polynomial-rate}, the conditional plug-in upper-bound rate is $O(n^{-s/(2s+d_X)})$, compared with $O(n^{-1/2})$ for OBiGrad. This uniform comparison is the exact analogue, over the full outer parameter space, of the fixed-$\omega$ contrast already drawn between \eqref{eq:plugin-bias-bound} and \eqref{eq:dr-product-bias}.

\paragraph{Plug-in rates under the main-text assumptions.}
Under \Cref{assum:eif-integrability,ass:source,ass:approx,ass:erm} and the tuning of \Cref{cor:source-only}, set $b=\nu(2+p)+\gamma$ and $\delta_n=(nt_n^\gamma)^{-1/(2+p)}=n^{-\nu/b}$. The nuisance errors defined in \eqref{eq:errors-source} satisfy $\mathbb E[R_{\ell,n}^2]=O(n^{-2\nu/b})$ for $\ell\in\{h,j\}$. For the fitted and oracle plug-in scores defined in the proof above, bounded-product entropy gives a deterministic nuisance-difference class with entropy of order $t_n^{-\gamma}\epsilon^{-p}$. Its localized $L^2(P)$ radius is at most $Dr_h+(A+B)r_j$, since the score difference can be written coordinatewise as $\langle h-h^\star_{\omega,P},j_k\rangle+\langle h^\star_{\omega,P}-Y,j_k-j_{k,\omega,P}^\star\rangle$. Applying \Cref{lem:source-local} on the independent evaluation fold and averaging over training and fitting randomness gives
\begin{equation*}
\begin{aligned}
&\mathbb E\left[\sup_{\omega\in\Omega}\left\|(\mathbb E_{P_{n,2}}-\mathbb E_P)[\hat\varphi^{\mathrm{PI}}_\omega-\varphi^{\mathrm{PI},\star}_\omega]\right\|\right]\\
&\qquad\le C\{\delta_n(\mathbb E[R_{h,n}]+\mathbb E[R_{j,n}])+\delta_n^2\}=O(n^{-2\nu/b}).
\end{aligned}
\end{equation*}
By \eqref{eq:plugin-quadratic-product-bias-bound} and Cauchy--Schwarz, the expected uniform population bias is at most $D\mathbb E[R_{h,n}]+(A+B)\mathbb E[R_{j,n}]+\mathbb E[R_{h,n}R_{j,n}]=O(n^{-\nu/b})$. Adding the $O(n^{-1/2})$ oracle empirical-process term yields, for finite constants $C_0,C_1$ independent of $n$,
\begin{equation*}
\mathbb E\left[\sup_{\omega\in\Omega}\|\widehat{\nabla\mathcal F}^{\mathrm{PI}}(\omega)-\nabla\mathcal F(\omega)\|\right]
\le C_0n^{-1/2}+C_1n^{-\nu/b}.
\end{equation*}
The expectation is over both samples and any fitting randomness.

\subsection{Sample complexity: OBiGrad versus plug-in}
\label{sec:appendix_sample_complexity}

Under \Cref{assum:eif-integrability,ass:source,ass:approx,ass:erm} and the tuning of \Cref{cor:source-only}, combine \Cref{cor:gradient-descent-convergence} with the unconditional gradient bounds of \Cref{cor:source-only} and Appendix~\ref{sec:appendix_plugin_uniform}. For fixed step size, $T=O(\varepsilon^{-2})$ iterations and sufficient sample sizes of orders
\begin{equation*}
n_{\mathrm{OBiGrad}}=O\!\left(\varepsilon^{-\max\{2,b/(2\nu)\}}\right),\qquad n_{\mathrm{plug\text{-}in}}=O\!\left(\varepsilon^{-b/\nu}\right),\qquad b=\nu(2+p)+\gamma,
\end{equation*}
give expected best-iterate stationarity error at most $\varepsilon$. Here $n$ is the size of each fold. When $\nu(2-p)\ge\gamma$, the OBiGrad sample-size bound reduces to $O(\varepsilon^{-2})$. These are sufficient upper bounds. The following comparison retains the milder, training-conditional assumptions and applies on their training event.

The deterministic inequality in the proof of \Cref{cor:gradient-descent-convergence} applies to any gradient estimator. Here we take its expectation conditionally on $S_1$, using the training-event bounds of \Cref{thm:quadratic-maximal-inequality,thm:plugin-maximal-inequality}. For the plug-in-driven projected update $\omega_{t+1}=\mathrm{Proj}_\Omega(\omega_t-\eta\widehat{\nabla\mathcal F}^{\mathrm{PI}}(\omega_t))$, the corresponding expected uniform error is bounded by
\begin{equation*}
    \mathfrak{s}_N^{\mathrm{PI}}\coloneqq\frac{C^{\mathrm{PI}}}{\sqrt{n}}(A+B)D\sqrt{1+K_A^{\mathrm{PI}}}+\mathfrak{C}_{B,n}^{\mathrm{PI}}(r_h,r_j)+Dr_h+(A+B)r_j+r_hr_j,
\end{equation*}
the bound of Theorem~\ref{thm:plugin-maximal-inequality}, giving on the training event, for every $T\geq1$,
\begin{equation}
\label{eq:plugin-gradient-descent-convergence}
    \mathbb{E}\left[\min_{0\leq t<T}\|G_\eta(\omega_t)\|\,\middle|\,S_1\right]\lesssim\left(\frac{\mathcal{F}(\omega_0)-F_{\inf}}{\eta T}\right)^{1/2}+\mathfrak{s}_N^{\mathrm{PI}}.
\end{equation}

Under the conditions of \Cref{cor:localized-polynomial-rate,cor:plugin-polynomial-rate} and the Hölder parametrization $p=d_X/s$, the bounds satisfy $\mathfrak{s}_N=O(n^{-1/2})$ and $\mathfrak{s}_N^{\mathrm{PI}}=O(n^{-s/(2s+d_X)})$. Applying the deterministic inequality from the proof of \Cref{cor:gradient-descent-convergence} conditionally on $S_1$ gives sufficient choices of $T=O(\varepsilon^{-2})$ for fixed step size and
\begin{equation*}
    n_{\mathrm{OBiGrad}}=O(\varepsilon^{-2}),\qquad n_{\mathrm{plug\text{-}in}}=O\bigl(\varepsilon^{-(2+d_X/s)}\bigr),
\end{equation*}
up to logarithmic factors from the entropy constants and the constant $\eta\leq1/L_{\mathcal F}$. For $d_X/s>0$, these upper bounds give a smaller sufficient sample-size exponent for OBiGrad; the difference between the exponents grows with the conditioning-covariate dimension $d_X$ relative to smoothness $s$. The gradient error is controlled for all outer steps using the same data: because Theorems~\ref{thm:quadratic-maximal-inequality} and~\ref{thm:plugin-maximal-inequality} control the gradient error uniformly over $\Omega$ from one training fold, the same event $\mathcal{E}_{\rm tr}$ covers every iterate $\omega_t$ along the (possibly adaptive) trajectory, so the statistical error does not compound with $T$. These are upper bounds derived from Theorems~\ref{thm:quadratic-maximal-inequality} and~\ref{thm:plugin-maximal-inequality}, not a matching lower bound for every possible plug-in-based procedure; we make no claim that they improve on existing sample-complexity guarantees for stochastic bilevel algorithms that redraw fresh data at every outer step, which is a different sampling regime from the fixed-sample, cross-fitted estimator analyzed here.

\section{Efficient Influence Function}
\label{sec:appendix_EIF}

Throughout, $L^2(P)$ denotes the space of real-valued functions of $(X,Y,Z)$ that are square-integrable under $P$, equipped with inner product $\langle f,f'\rangle_{L^2(P)}=\mathbb{E}_P[ff']$ and induced norm $\|f\|_{L^2(P)}^2=\mathbb{E}_P[f^2]$; $L^1(P)$ denotes its integrable counterpart, with norm $\|f\|_{L^1(P)}=\mathbb{E}_P[|f|]$; and $L^2_0(P)\coloneqq\{S\in L^2(P):\mathbb{E}_P[S]=0\}$ denotes the mean-zero subspace of $L^2(P)$, the tangent space of scores of regular parametric submodels in the nonparametric model.

The efficient influence function of $\nabla\mathcal{F}(\omega)$, obtained via influence function calculus \citep{vonmises1947asymptotic,hampel1974influence,fernholz1983mises}, provides a principled correction for this bias \citep{newey1994asymptotic,chernozhukov2018double}. We make this precise below.

The target functional $\nabla\mathcal{F}(\omega)$ depends on $P$ through two infinite-dimensional nuisances, the inner solution $h^\star_{\omega,P}$ and its derivative $j^\star_{k,\omega,P}$, so the plug-in estimator suggested by \Cref{sec:problem} inherits their estimation error at first order. Correcting this bias requires understanding how the target, written locally in this appendix as $\Psi_{k,\omega}(P)\coloneqq\partial_{\omega_k}\mathcal F_P(\omega)$, responds to small changes in $P$. Since $P$ cannot be perturbed directly, semiparametric efficiency theory instead considers every smooth one-parameter family of distributions $(P_\epsilon)_\epsilon$ passing through $P$ at $\epsilon=0$, called a \emph{parametric submodel}. A canonical example is the linear submodel $\dd P_\epsilon\coloneqq(1+\epsilon S)\,\dd P$, for any bounded, mean-zero function $S$ of the data: since $\mathbb{E}_P[S]=0$, this integrates to one for every $\epsilon$, and for $\epsilon$ small enough that $1+\epsilon S>0$ everywhere, it defines a valid probability distribution with $P_0=P$. Here $S$ is called the \emph{score} of this submodel: since $\dd P_\epsilon=(1+\epsilon S)\,\dd P$, observations where $S$ is positive become more likely as $\epsilon$ increases, and observations where $S$ is negative become less likely. The \emph{efficient influence function} $\chi_{P,\omega}$ is the function such that, for every submodel with score $S$,
\begin{equation*}
    \left.\frac{\dd}{\dd\epsilon}\Psi_{k,\omega}(P_\epsilon)\right|_{\epsilon=0}=\left\langle\chi_{k,P,\omega}, S\right\rangle_{L^2(P)}.
\end{equation*}
This equation is what lets us fix the plug-in bias: adding $\chi_{k,P,\omega}$ to a plug-in estimate of $\Psi_{k,\omega}$ cancels its first-order sensitivity to nuisance estimation error, which is exactly the correction we build in \Cref{sec:debiased-clt}. Requiring this to hold for every perturbation direction $S$ is what makes $\Psi_{k,\omega}$ \emph{pathwise differentiable} at $P$, with $\chi_{k,P,\omega}$ its efficient influence function \citep{bickel1993efficient,newey1994asymptotic,vanderVaart1998asymptotic,tsiatis2006semiparametric,kosorok2008introduction}. We derive $\chi_{P,\omega}$ below by directly differentiating $\Psi_{k,\omega}$ along parametric submodels, which requires the nuisances $h^\star_{\omega,P}$ and $j^\star_{k,\omega,P}$ to vary smoothly along them. Rather than assuming this outright, the argument below shows it is already guaranteed by \Cref{assum:eif-integrability}: since these nuisances are themselves conditional expectations of the bounded quantities $g_\omega(Z)$ and $\partial_{\omega_k}g_\omega(Z)$ under our quadratic losses, and conditional expectations of bounded quantities vary smoothly along any regular parametric submodel, differentiability holds automatically.

The linear submodel $\dd P_\epsilon\coloneqq(1+\epsilon S)\,\dd P$ is defined for any bounded, mean-zero $S$: since $\mathbb{E}_P[S]=0$ it integrates to one for every $\epsilon$, and for $|\epsilon|<1/\|S\|_\infty$, so that $1+\epsilon S>0$ everywhere, it is a valid probability measure with $P_0=P$. This is a regular parametric submodel in the formal (differentiable-in-quadratic-mean) sense of the semiparametric efficiency literature \citep{bickel1993efficient,vanderVaart1998asymptotic}, with score $S$.

Define $a^\star_{\omega,P}(X)\coloneqq m_P^\star(X)-h^\star_{\omega,P}(X)$.

\begin{theorem}[Efficient influence function]
\label{thm:EIF}
Under \Cref{assum:eif-integrability}, the functional $\Psi_{k,\omega}$ is pathwise differentiable at $P$ with efficient influence function
\begin{align}
\label{eq:eif}
    \chi_{k,P,\omega}(X,Y,Z)=&-\langle Y-g_\omega(Z),\,j^\star_{k,\omega,P}(X)\rangle-\langle\partial_{\omega_k}g_\omega(Z),\,a^\star_{\omega,P}(X)\rangle+\langle j^\star_{k,\omega,P}(X),\,a^\star_{\omega,P}(X)\rangle\nonumber\\
    &-\Psi_{k,\omega}(P).
\end{align}
Consequently, $\chi_{P,\omega}\coloneqq(\chi_{1,P,\omega},\ldots,\chi_{d,P,\omega})^\top$ is the vector-valued efficient influence function for $\nabla\mathcal{F}(\omega)$ in the nonparametric model.
\end{theorem}

The proof follows below; the accompanying functional von Mises expansion in Appendix~\ref{sec:appendix_VME} quantifies the effect of replacing oracle nuisances by candidates. In brief, once the correction terms in \eqref{eq:eif} are subtracted from the plug-in score, the residual has no first-order dependence on the nuisance estimation errors $h-h^\star_{\omega,P}$ and $v-j^\star_{k,\omega,P}$; only cross-products of correction- and nuisance-estimation errors and a second-order remainder survive (\Cref{eq:vme,eq:vme-remainder} in Appendix~\ref{sec:appendix_VME}). Equation~\eqref{eq:eif} directly suggests the debiased pseudo-outcome cross-fitted in Appendix~\ref{sec:appendix_CLT}.

\begin{lemma}[Score identity]
\label{lem:score-identity}
For every bounded, mean-zero $S$ and every $f\in L^2(P)$,
\begin{equation*}
    \left.\frac{\dd}{\dd\epsilon}\mathbb{E}_{P_\epsilon}[f]\right|_{\epsilon=0}=\mathbb{E}_P[fS].
\end{equation*}
\end{lemma}

\begin{proof}
By definition of $P_\epsilon$ and $fS\in L^1(P)$ (Cauchy--Schwarz, as $f,S\in L^2(P)$),
\begin{equation*}
    \mathbb{E}_{P_\epsilon}[f]=\int f(1+\epsilon S)\,\dd P=\mathbb{E}_P[f]+\epsilon\,\mathbb{E}_P[fS]
\end{equation*}
for every $\epsilon$, exactly linear in $\epsilon$; differentiating at $\epsilon=0$ gives the claim.
\end{proof}

Taking $f\equiv1$ recovers $\mathbb{E}_P[S]=0$, consistent with $S$ being mean-zero by construction. \Cref{lem:score-identity} justifies the score identity used at the start of the proof of Theorem~\ref{thm:EIF} below, without needing to assume differentiation under the integral sign as a separate step.

\begin{proposition}[Weak differentiability of conditional expectations]
\label{prop:path-diff-consequence}
Let $S$ be bounded and mean-zero, let $f=f(Z)\in\mathbb{R}^q$ be a random vector with $\|f(Z)\|\le M<\infty$ almost surely, and write $\mu_\epsilon(X)\coloneqq\mathbb{E}_{P_\epsilon}[f(Z)\mid X]\in\mathbb{R}^q$ and $\mu\coloneqq\mu_0$. Then $\epsilon\mapsto\mu_\epsilon$ is weakly differentiable in $\LL$ at $\epsilon=0$, meaning
\begin{equation*}
    \left\langle\frac{\mu_\epsilon-\mu}{\epsilon},\,b\right\rangle_{\LL}\xrightarrow{\ \epsilon\to0\ }\langle\dot\mu,\,b\rangle_{\LL}\qquad\text{for every bounded }b\in\LL,
\end{equation*}
with $\dot\mu(X)=\mathbb{E}_P[(f(Z)-\mu(X))S\mid X]\in\mathbb{R}^q$.
\end{proposition}

\begin{proof}
Write $\mu(X)=\mathbb{E}_P[f(Z)\mid X]$, $a(X)\coloneqq\mathbb{E}_P[f(Z)S\mid X]$, and $c(X)\coloneqq\mathbb{E}_P[S\mid X]$; none of these depend on $\epsilon$. By Bayes' rule, since $\dd P_\epsilon=(1+\epsilon S)\,\dd P$,
\begin{equation*}
    \mu_\epsilon(X)=\frac{\mathbb{E}_P[f(Z)(1+\epsilon S)\mid X]}{\mathbb{E}_P[1+\epsilon S\mid X]}=\frac{\mu(X)+\epsilon\,a(X)}{1+\epsilon\,c(X)}
\end{equation*}
is (for a.e.\ $X$) an ordinary ratio of two affine functions of $\epsilon$, with $|c(X)|\le\|S\|_\infty$, so the denominator is bounded away from zero, uniformly in $X$, for $|\epsilon|<1/(2\|S\|_\infty)$. Subtracting $\mu(X)$ and putting the result over this common denominator gives
\begin{equation*}
    \mu_\epsilon(X)-\mu(X)=\frac{\epsilon\bigl(a(X)-\mu(X)c(X)\bigr)}{1+\epsilon\,c(X)},
\end{equation*}
so, by the quotient rule, $\epsilon\mapsto\mu_\epsilon(X)$ is differentiable at $\epsilon=0$ for a.e.\ $X$, with $\dot\mu(X)=a(X)-\mu(X)c(X)=\mathbb{E}_P[(f(Z)-\mu(X))S\mid X]$, using that $\mu(X)$ is $X$-measurable.

Since $\|f(Z)\|\le M$, we have $\|a(X)\|\le M\|S\|_\infty$ and $\|\mu(X)\|\le M$ (Jensen's inequality), so $\|a(X)-\mu(X)c(X)\|\le2M\|S\|_\infty$, and the display above gives, still for $|\epsilon|<1/(2\|S\|_\infty)$,
\begin{equation*}
    \left\|\frac{\mu_\epsilon(X)-\mu(X)}{\epsilon}\right\|\le\frac{2M\|S\|_\infty}{1-|\epsilon|\|S\|_\infty}\le4M\|S\|_\infty\eqqcolon C,
\end{equation*}
a bound uniform in $X$ and in $\epsilon$ near $0$. Fix a bounded $b\in\LL$. Since $\frac{\mu_\epsilon(X)-\mu(X)}{\epsilon}\to\dot\mu(X)$ pointwise as $\epsilon\to0$ and $\bigl\|\langle\frac{\mu_\epsilon-\mu}{\epsilon},b\rangle\bigr\|\le C\|b(X)\|$, with $C\|b\|\in L^1(P)$, dominated convergence gives
\begin{equation*}
    \left\langle\frac{\mu_\epsilon-\mu}{\epsilon},\,b\right\rangle_{\LL}=\mathbb{E}_P\left[\left\langle\frac{\mu_\epsilon(X)-\mu(X)}{\epsilon},\,b(X)\right\rangle\right]\longrightarrow\mathbb{E}_P[\langle\dot\mu(X),b(X)\rangle]=\langle\dot\mu,b\rangle_{\LL}.
\end{equation*}
\end{proof}

\begin{proposition}[Differentiability of the nuisances]
\label{prop:path_diff_nuis}
Under \Cref{assum:eif-integrability}, along every submodel in this family, the maps $\epsilon\mapsto h^\star_{\omega,P_\epsilon}$ and $\epsilon\mapsto j^\star_{k,\omega,P_\epsilon}$ are weakly differentiable at $\epsilon=0$ in $\LL$: this is not an independent restriction on the submodel beyond \Cref{assum:eif-integrability}.
\end{proposition}

\begin{proof}
Since $h^\star_{\omega,P}(X)=\mathbb{E}_P[g_\omega(Z)\mid X]$ and $j^\star_{k,\omega,P}(X)=\mathbb{E}_P[\partial_{\omega_k}g_\omega(Z)\mid X]$ are themselves conditional expectations of the bounded quantities $g_\omega(Z)$ and $\partial_{\omega_k}g_\omega(Z)$ under the quadratic specialization, \Cref{prop:path-diff-consequence} applies with $f=g_\omega$ and $f=\partial_{\omega_k}g_\omega$, both $\mathbb{R}^q$-valued and bounded in norm under \Cref{assum:eif-integrability}, giving the pathwise derivatives
\begin{equation}
\label{eq:hdot-jdot}
\begin{aligned}
    \dot h(X)&\coloneqq\left.\frac{\dd}{\dd\epsilon}h^\star_{\omega,P_\epsilon}(X)\right|_{\epsilon=0}=\mathbb{E}_P\bigl[(g_\omega(Z)-h^\star_{\omega,P}(X))S\mid X\bigr],\\
    \dot j_k(X)&\coloneqq\left.\frac{\dd}{\dd\epsilon}j^\star_{k,\omega,P_\epsilon}(X)\right|_{\epsilon=0}=\mathbb{E}_P\bigl[(\partial_{\omega_k}g_\omega(Z)-j^\star_{k,\omega,P}(X))S\mid X\bigr],
\end{aligned}
\end{equation}
in the weak sense of \Cref{prop:path-diff-consequence}, which is the only sense in which they are used below.
\end{proof}

\begin{proof}[Proof of Theorem~\ref{thm:EIF}]
Fix $k\in\{1,\dots,d\}$ and a bounded, mean-zero $S$, and let $(P_\epsilon)_\epsilon$ be the corresponding linear submodel defined above. Recall from \Cref{lem:score-identity} the score identity, $\frac{\dd}{\dd\epsilon}\mathbb{E}_{P_\epsilon}[f]|_{\epsilon=0}=\mathbb{E}_P[fS]$ for every $f\in L^2(P)$.

\noindent Our goal is to find $\chi_{k,P,\omega}\in L^2_0(P)$ such that
\begin{equation*}
    \left.\frac{\dd}{\dd\epsilon}\Psi_{k,\omega}(P_\epsilon)\right|_{\epsilon=0}=\langle\chi_{k,P,\omega},S\rangle_{L^2(P)},\qquad\forall\text{ bounded, mean-zero }S.
\end{equation*}
As recalled above, this already uniquely determines $\chi_{k,P,\omega}$ over the whole nonparametric tangent space $L^2_0(P)$, which is the efficient influence function.

\paragraph{Decomposition.}
Recall that $\Psi_{k,\omega}(P_\epsilon)=\mathbb{E}_{P_\epsilon}\bigl[\langle h^\star_{\omega,P_\epsilon}(X)-Y,\,j^\star_{k,\omega,P_\epsilon}(X)\rangle\bigr]$. There are three sources of $\epsilon$-dependence: the measure $P_\epsilon$, the inner solution $h^\star_{\omega,P_\epsilon}$, and the Jacobian $j^\star_{k,\omega,P_\epsilon}$. By the product rule,
\begin{equation*}
\begin{aligned}
    \left.\frac{\dd}{\dd\epsilon}\Psi_{k,\omega}(P_\epsilon)\right|_{\epsilon=0}
    &=\underbrace{
        \left.\frac{\dd}{\dd\epsilon}
        \mathbb{E}_{P_\epsilon}\bigl[\langle h^\star_{\omega,P}(X)-Y,\,j^\star_{k,\omega,P}(X)\rangle\bigr]
        \right|_{\epsilon=0}
    }_{\text{(A): measure perturbation}}\\
    &\quad+\underbrace{
        \mathbb{E}_P\bigl[\langle\dot h(X),\,j^\star_{k,\omega,P}(X)\rangle\bigr]
    }_{\text{(B): effect on }h^\star_{\omega,P_\epsilon}}\\
    &\quad+\underbrace{
        \mathbb{E}_P\bigl[\langle h^\star_{\omega,P}(X)-Y,\,\dot j_k(X)\rangle\bigr]
    }_{\text{(C): effect on }j^\star_{k,\omega,P_\epsilon}}.
\end{aligned}
\end{equation*}

\paragraph{Term (A): measure perturbation.}
By the score identity,
\begin{equation*}
    A=\mathbb{E}_P\bigl[\langle h^\star_{\omega,P}(X)-Y,\,j^\star_{k,\omega,P}(X)\rangle\,S\bigr].
\end{equation*}
This is the plug-in term: the direct contribution from the measure shifting.

\paragraph{Term (B): effect of $h^\star_{\omega,P_\epsilon}$ changing.}
Substituting $\dot h$ from \eqref{eq:hdot-jdot} and applying the tower property, since $j^\star_{k,\omega,P}(X)$ is $X$-measurable,
\begin{equation*}
    B=\mathbb{E}_P\bigl[\langle\mathbb{E}_P[(g_\omega(Z)-h^\star_{\omega,P}(X))S\mid X],\,j^\star_{k,\omega,P}(X)\rangle\bigr]=\mathbb{E}_P\bigl[\langle g_\omega(Z)-h^\star_{\omega,P}(X),\,j^\star_{k,\omega,P}(X)\rangle\,S\bigr].
\end{equation*}

\paragraph{Term (C): effect of $j^\star_{k,\omega,P_\epsilon}$ changing.}
Substituting $\dot j_k$ from \eqref{eq:hdot-jdot} and applying the tower property twice, first to replace $h^\star_{\omega,P}(X)-Y$ by its $X$-conditional mean $\mathbb{E}_P[h^\star_{\omega,P}(X)-Y\mid X]=h^\star_{\omega,P}(X)-m^\star_P(X)=-a^\star_{\omega,P}(X)$, and then to the remaining $X$-conditional term,
\begin{equation*}
\begin{aligned}
    C&=\mathbb{E}_P\bigl[\langle h^\star_{\omega,P}(X)-Y,\,\mathbb{E}_P[(\partial_{\omega_k}g_\omega(Z)-j^\star_{k,\omega,P}(X))S\mid X]\rangle\bigr]\\
    &=\mathbb{E}_P\bigl[\langle-a^\star_{\omega,P}(X),\,(\partial_{\omega_k}g_\omega(Z)-j^\star_{k,\omega,P}(X))S\rangle\bigr]\\
    &=-\mathbb{E}_P\bigl[\langle\partial_{\omega_k}g_\omega(Z),\,a^\star_{\omega,P}(X)\rangle\,S\bigr]+\mathbb{E}_P\bigl[\langle j^\star_{k,\omega,P}(X),\,a^\star_{\omega,P}(X)\rangle\,S\bigr].
\end{aligned}
\end{equation*}

\paragraph{Combining and centering.}
The $h^\star_{\omega,P}(X)$ terms in $A$ and $B$ cancel, giving
\begin{equation*}
    A+B=\mathbb{E}_P\bigl[\langle g_\omega(Z)-Y,\,j^\star_{k,\omega,P}(X)\rangle\,S\bigr]=-\mathbb{E}_P\bigl[\langle Y-g_\omega(Z),\,j^\star_{k,\omega,P}(X)\rangle\,S\bigr].
\end{equation*}
Adding $C$,
\begin{equation*}
    \left.\frac{\dd}{\dd\epsilon}\Psi_{k,\omega}(P_\epsilon)\right|_{\epsilon=0}=A+B+C=\mathbb{E}_P\bigl[\tilde\chi_{k,P,\omega}\,S\bigr],
\end{equation*}
where
\begin{equation*}
    \tilde\chi_{k,P,\omega}=-\langle Y-g_\omega(Z),\,j^\star_{k,\omega,P}(X)\rangle-\langle\partial_{\omega_k}g_\omega(Z),\,a^\star_{\omega,P}(X)\rangle+\langle j^\star_{k,\omega,P}(X),\,a^\star_{\omega,P}(X)\rangle.
\end{equation*}
Since $\mathbb{E}_P[S]=0$, we may subtract any constant from $\tilde\chi_{k,P,\omega}$ without changing the inner product with $S$. Subtracting $\Psi_{k,\omega}(P)$ centers the function and ensures $\mathbb{E}_P[\chi_{k,P,\omega}]=0$, giving the canonical gradient $\chi_{k,P,\omega}=\tilde\chi_{k,P,\omega}-\Psi_{k,\omega}(P)$, matching \eqref{eq:eif}. By Assumption~\ref{assum:eif-integrability}, $\chi_{k,P,\omega}\in L^2_0(P)$, and the identity $\left.\frac{\dd}{\dd\epsilon}\Psi_{k,\omega}(P_\epsilon)\right|_{\epsilon=0}=\langle\chi_{k,P,\omega},S\rangle_{L^2(P)}$ holds for every bounded, mean-zero $S$, since $S$ was arbitrary above. This identifies $\chi_{k,P,\omega}$ as the unique element of $L^2_0(P)$ representing this derivative relative to the nonparametric tangent space, hence the efficient influence function. Stacking coordinates, $\chi_{P,\omega}=(\chi_{1,P,\omega},\ldots,\chi_{d,P,\omega})^\top$ is the vector-valued efficient influence function for $\nabla\mathcal{F}(\omega)$.
\end{proof}

\section{Functional von Mises Expansion}
\label{sec:appendix_VME}

This appendix proves the functional von Mises expansion underlying Theorem~\ref{thm:EIF}. It quantifies the error made when the oracle nuisances are replaced by candidate functions, and shows explicitly why the corrected score \eqref{eq:eif} is orthogonal. Throughout, write $\psi_{k,P}(h,v)\coloneqq\mathbb{E}_P[\langle h(X)-Y,v(X)\rangle]$, so that $\Psi_{k,\omega}(P)\coloneqq\partial_{\omega_k}\mathcal F_P(\omega)=\psi_{k,P}(h^\star_{\omega,P},j^\star_{k,\omega,P})$ by the direct form of \Cref{sec:problem}.

\begin{theorem}[Functional von Mises expansion]
\label{thm:VME}
Under \Cref{assum:eif-integrability}, for each $k=1,\ldots,d$ there exists a neighborhood $\mathcal N_k$ of $(h^\star_{\omega,P},j^\star_{k,\omega,P})$ such that, for every $(h,v)\in\mathcal N_k$ and every $\alpha_{1,k},\alpha_2\in \LL$,
\begin{align}
\label{eq:vme}
    &\psi_{k,P}(h,v)-\psi_{k,P}(h^\star_{\omega,P},j^\star_{k,\omega,P})-\mathbb{E}_P\bigl[\langle h(X)-g_\omega(Z),\alpha_{1,k}(X)\rangle\bigr]+\mathbb{E}_P\bigl[\langle\partial_{\omega_k}g_\omega(Z),\alpha_2(X)\rangle\bigr]\nonumber\\
    &-\mathbb{E}_P\bigl[\langle v(X),\alpha_2(X)\rangle\bigr]\nonumber\\
    =&\langle j^\star_{k,\omega,P}-\alpha_{1,k},h-h^\star_{\omega,P}\rangle_{\LL}-\langle a^\star_{\omega,P}+\alpha_2,v-j^\star_{k,\omega,P}\rangle_{\LL}\nonumber\\
    &+\mathrm{Rem}_{k}(h-h^\star_{\omega,P},v-j^\star_{k,\omega,P}),
\end{align}
where
\begin{equation}
\label{eq:vme-remainder}
    |\mathrm{Rem}_{k}(h-h^\star_{\omega,P},v-j^\star_{k,\omega,P})|\le \frac{1}{2}\left(\|h-h^\star_{\omega,P}\|_{\LL}^2+\|v-j^\star_{k,\omega,P}\|_{\LL}^2\right).
\end{equation}
\end{theorem}
Equation~\eqref{eq:vme} is the key orthogonality identity: once the correction terms are subtracted from the plug-in score, the residual has no first-order dependence on the nuisance estimation errors $h-h^\star_{\omega,P}$ and $v-j^\star_{k,\omega,P}$. The only remaining terms are cross-products of the correction errors $\|j^\star_{k,\omega,P}-\alpha_{1,k}\|_{\LL}$ and $\|a^\star_{\omega,P}+\alpha_2\|_{\LL}$ with nuisance-estimation errors $\|h-h^\star_{\omega,P}\|_{\LL}$ and $\|v-j^\star_{k,\omega,P}\|_{\LL}$, plus a second-order remainder controlled by \eqref{eq:vme-remainder}. The proof below establishes this expansion.

\begin{proposition}
\label{prop:psi-smoothness}
Under the quadratic specialization, for each $k=1,\ldots,d$, the map $(h,v)\mapsto\psi_{k,P}(h,v)$ is Fréchet differentiable on all of $\LL\times\LL$, with globally Lipschitz Fréchet derivative of constant $1$.
\end{proposition}

\begin{proof}
Fix $k\in\{1,\ldots,d\}$. By definition, $\psi_{k,P}(h,v)=\mathbb{E}_P[\langle h(X)-Y,v(X)\rangle]$. We can write this as
\begin{equation*}
    \psi_{k,P}(h,v)=-\mathbb{E}_P[\langle Y,v(X)\rangle]+\mathbb{E}_P[\langle h(X),v(X)\rangle],
\end{equation*}
which is affine in $h$ and linear in $v$.

\paragraph{Fr\'echet differentiability.}
For any $\delta h\in\LL$,
\begin{equation*}
    \psi_{k,P}(h+\delta h,v)-\psi_{k,P}(h,v)=\mathbb{E}_P[\langle\delta h(X),v(X)\rangle]\eqqcolon\partial_1\psi_{k,P}(h,v)[\delta h],
\end{equation*}
with remainder zero, so $\psi_{k,P}$ is exactly linear in $h$. The functional $\partial_1\psi_{k,P}(h,v)$ is bounded since by Cauchy-Schwarz,
\begin{equation*}
    |\partial_1\psi_{k,P}(h,v)[\delta h]|\leq\|v\|_{\LL}\|\delta h\|_{\LL}.
\end{equation*}
For any $\delta v\in\LL$,
\begin{equation*}
    \psi_{k,P}(h,v+\delta v)-\psi_{k,P}(h,v)=-\mathbb{E}_P[\langle Y-h(X),\delta v(X)\rangle]\eqqcolon\partial_2\psi_{k,P}(h,v)[\delta v],
\end{equation*}
with remainder zero, so $\psi_{k,P}$ is exactly linear in $v$. The functional $\partial_2\psi_{k,P}(h,v)$ is bounded since by Cauchy-Schwarz,
\begin{equation*}
    |\partial_2\psi_{k,P}(h,v)[\delta v]|\leq\|Y-h(X)\|_{L^2(P)}\|\delta v\|_{\LL}.
\end{equation*}
Since both partial Fr\'echet derivatives exist and are bounded, $\psi_{k,P}$ is Fr\'echet differentiable on all of $\LL\times\LL$.

\paragraph{Lipschitz derivative.}
For any $(h_1,v_1),(h_2,v_2)\in\LL\times\LL$, by Cauchy-Schwarz,
\begin{align*}
    \|\partial_1\psi_{k,P}(h_1,v_1)-\partial_1\psi_{k,P}(h_2,v_2)\|_{\mathrm{op}}&\leq\|v_1-v_2\|_{\LL},\\
    \|\partial_2\psi_{k,P}(h_1,v_1)-\partial_2\psi_{k,P}(h_2,v_2)\|_{\mathrm{op}}&\leq\|h_1-h_2\|_{\LL}.
\end{align*}
Combining,
\begin{equation*}
    \|D\psi_{k,P}(h_1,v_1)-D\psi_{k,P}(h_2,v_2)\|_{\mathrm{op}}\leq\|(h_1-h_2,v_1-v_2)\|_{\LL\times\LL},
\end{equation*}
so the Fr\'echet derivative is globally Lipschitz with constant $1$.
\end{proof}

\begin{lemma}[Coordinatewise Taylor expansion]
\label{lem:quadratic_expansion}
Fix $k\in\{1,\dots,d\}$. Under the quadratic specialization, for all $h,v\in\LL$,
\begin{equation*}
\begin{aligned}
    &\psi_{k,P}(h^\star_{\omega,P}+h,j^\star_{k,\omega,P}+v)-\psi_{k,P}(h^\star_{\omega,P},j^\star_{k,\omega,P})\\
    &\quad=\partial_1\psi_{k,P}(h^\star_{\omega,P},j^\star_{k,\omega,P})[h]+\partial_2\psi_{k,P}(h^\star_{\omega,P},j^\star_{k,\omega,P})[v]+\mathrm{Rem}_{k}(h,v),
\end{aligned}
\end{equation*}
where
\begin{equation*}
    |\mathrm{Rem}_{k}(h,v)|\leq \frac{1}{2}\left(\|h\|_{\LL}^2+\|v\|_{\LL}^2\right).
\end{equation*}
\end{lemma}

\begin{proof}
Let $z^\star=(h^\star_{\omega,P},j^\star_{k,\omega,P})$ and $r=(h,v)$. By Proposition~\ref{prop:psi-smoothness}, $\psi_{k,P}$ is Fréchet differentiable on all of $\LL\times\LL$ with globally Lipschitz derivative of constant $1$. By the integral form of Taylor's theorem,
\begin{equation*}
    \psi_{k,P}(z^\star+r)-\psi_{k,P}(z^\star)=\int_0^1 D\psi_{k,P}(z^\star+tr)[r]\,\dd t.
\end{equation*}
Adding and subtracting $D\psi_{k,P}(z^\star)[r]$ gives
\begin{equation*}
    \psi_{k,P}(z^\star+r)-\psi_{k,P}(z^\star)=D\psi_{k,P}(z^\star)[r]+\int_0^1\left\{D\psi_{k,P}(z^\star+tr)-D\psi_{k,P}(z^\star)\right\}[r]\,\dd t,
\end{equation*}
where the first term equals
\begin{equation*}
    D\psi_{k,P}(z^\star)[r]=\partial_1\psi_{k,P}(h^\star_{\omega,P},j^\star_{k,\omega,P})[h]+\partial_2\psi_{k,P}(h^\star_{\omega,P},j^\star_{k,\omega,P})[v].
\end{equation*}
For the remainder, the Lipschitz bound from Proposition~\ref{prop:psi-smoothness} gives
\begin{align*}
    \left|\int_0^1\left\{D\psi_{k,P}(z^\star+tr)-D\psi_{k,P}(z^\star)\right\}[r]\,\dd t\right|&\leq\int_0^1\|D\psi_{k,P}(z^\star+tr)-D\psi_{k,P}(z^\star)\|_{\mathrm{op}}\|r\|\,\dd t\\
    &\leq\int_0^1 t\|r\|^2\,\dd t=\frac{1}{2}\|r\|^2,
\end{align*}
so
\begin{equation*}
    |\mathrm{Rem}_{k}(h,v)|\leq\frac{1}{2}\|(h,v)\|_{\LL\times\LL}^2\leq\frac{1}{2}\left(\|h\|_{\LL}^2+\|v\|_{\LL}^2\right).
\end{equation*}
\end{proof}

\begin{proof}[Proof of Theorem~\ref{thm:VME}]
Fix $k\in\{1,\dots,d\}$.

\paragraph{Step 1: First-order expansion of $\psi_{k,P}$.}By Lemma~\ref{lem:quadratic_expansion}, applied at the base point $(h^\star_{\omega,P},j^\star_{k,\omega,P})$,
\begin{equation*}
\begin{aligned}
    &\psi_{k,P}(h,v)-\psi_{k,P}(h^\star_{\omega,P},j^\star_{k,\omega,P})\\
    =&\partial_1\psi_{k,P}(h^\star_{\omega,P},j^\star_{k,\omega,P})[h-h^\star_{\omega,P}]+\partial_2\psi_{k,P}(h^\star_{\omega,P},j^\star_{k,\omega,P})[v-j^\star_{k,\omega,P}]\\
    &+\mathrm{Rem}_{k}(h-h^\star_{\omega,P},v-j^\star_{k,\omega,P}),
\end{aligned}
\end{equation*}
where 
\begin{equation*}
    |\mathrm{Rem}_{k}(h-h^\star_{\omega,P},v-j^\star_{k,\omega,P})|\leq\frac{1}{2}(\|h-h^\star_{\omega,P}\|_{\LL}^2+\|v-j^\star_{k,\omega,P}\|_{\LL}^2).
\end{equation*}
By the identities established above, the first-order terms equal $\langle j^\star_{k,\omega,P},h-h^\star_{\omega,P}\rangle_{\LL}$ and $-\langle a^\star_{\omega,P},v-j^\star_{k,\omega,P}\rangle_{\LL}$ respectively, so
\begin{equation}
\label{eq:vme_step1}
\begin{aligned}
    &\psi_{k,P}(h,v)-\psi_{k,P}(h^\star_{\omega,P},j^\star_{k,\omega,P})\\
    &=\langle j^\star_{k,\omega,P},h-h^\star_{\omega,P}\rangle_{\LL}-\langle a^\star_{\omega,P},v-j^\star_{k,\omega,P}\rangle_{\LL}+\mathrm{Rem}_{k}(h-h^\star_{\omega,P},v-j^\star_{k,\omega,P}).
\end{aligned}
\end{equation}

\paragraph{Step 2: Expanding the correction terms.}
We now express each correction term in \eqref{eq:vme} in terms of $\langle\cdot,\cdot\rangle_{\LL}$.

\textit{First correction: $\mathbb{E}_P[\langle h(X)-g_\omega(Z),\alpha_{1,k}(X)\rangle]$.}
By bilinearity of the inner product, pointwise $\langle h(X)-g_\omega(Z),\alpha_{1,k}(X)\rangle=\langle h^\star_{\omega,P}(X)-g_\omega(Z),\alpha_{1,k}(X)\rangle+\langle h(X)-h^\star_{\omega,P}(X),\alpha_{1,k}(X)\rangle$, so taking expectations,
\begin{equation*}
    \mathbb{E}_P\bigl[\langle h(X)-g_\omega(Z),\alpha_{1,k}(X)\rangle\bigr]=\mathbb{E}_P\bigl[\langle h^\star_{\omega,P}(X)-g_\omega(Z),\alpha_{1,k}(X)\rangle\bigr]+\langle h-h^\star_{\omega,P},\alpha_{1,k}\rangle_{\LL}.
\end{equation*}
Since $h^\star_{\omega,P}(X)=\mathbb{E}_P[g_\omega(Z)\mid X]$, the first term vanishes by the tower property, giving
\begin{equation}
\label{eq:vme_corr1}
    \mathbb{E}_P\bigl[\langle h(X)-g_\omega(Z),\alpha_{1,k}(X)\rangle\bigr]=\langle h-h^\star_{\omega,P},\alpha_{1,k}\rangle_{\LL}.
\end{equation}

\textit{Second correction: $\mathbb{E}_P[\langle\partial_{\omega_k}g_\omega(Z),\alpha_2(X)\rangle]$.} This does not involve $h$ at all, as the formula shows. Differentiating the first-order condition $\mathbb{E}_P[\langle h^\star_{\omega,P}(X)-g_\omega(Z),u(X)\rangle]=0$ (for every $u\in\LL$) with respect to $\omega_k$, and using $j^\star_{k,\omega,P}=\partial_{\omega_k}h^\star_{\omega,P}$, gives
\begin{equation}
\label{eq:vme_corr2}
    \mathbb{E}_P\bigl[\langle\partial_{\omega_k}g_\omega(Z),\alpha_2(X)\rangle\bigr]=\langle j^\star_{k,\omega,P},\alpha_2\rangle_{\LL}.
\end{equation}

\textit{Third correction: $\mathbb{E}_P[\langle v(X),\alpha_2(X)\rangle]$.} This already equals $\langle v,\alpha_2\rangle_{\LL}$ exactly, by definition of the inner product on $\LL$.

\paragraph{Step 3: Assembling the expansion.}
Substituting \eqref{eq:vme_corr1} and \eqref{eq:vme_corr2} into $\psi_{k,P}(h,v)-\psi_{k,P}(h^\star_{\omega,P},j^\star_{k,\omega,P})-\mathbb{E}_P[\langle h(X)-g_\omega(Z),\alpha_{1,k}(X)\rangle]+\mathbb{E}_P[\langle\partial_{\omega_k}g_\omega(Z),\alpha_2(X)\rangle]-\mathbb{E}_P[\langle v(X),\alpha_2(X)\rangle]$ via \eqref{eq:vme_step1}, and using symmetry of $\langle\cdot,\cdot\rangle_{\LL}$,
\begin{equation*}
\begin{aligned}
    &\psi_{k,P}(h,v)-\psi_{k,P}(h^\star_{\omega,P},j^\star_{k,\omega,P})-\mathbb{E}_P\bigl[\langle h(X)-g_\omega(Z),\alpha_{1,k}(X)\rangle\bigr]+\mathbb{E}_P\bigl[\langle\partial_{\omega_k}g_\omega(Z),\alpha_2(X)\rangle\bigr]\\
    &-\mathbb{E}_P\bigl[\langle v(X),\alpha_2(X)\rangle\bigr]\\
    =&\langle j^\star_{k,\omega,P},h-h^\star_{\omega,P}\rangle_{\LL}-\langle\alpha_{1,k},h-h^\star_{\omega,P}\rangle_{\LL}\\
    &-\langle a^\star_{\omega,P},v-j^\star_{k,\omega,P}\rangle_{\LL}+\langle j^\star_{k,\omega,P},\alpha_2\rangle_{\LL}-\langle v,\alpha_2\rangle_{\LL}+\mathrm{Rem}_k(h-h^\star_{\omega,P},v-j^\star_{k,\omega,P})\\
    =&\langle j^\star_{k,\omega,P}-\alpha_{1,k},h-h^\star_{\omega,P}\rangle_{\LL}-\langle a^\star_{\omega,P}+\alpha_2,v-j^\star_{k,\omega,P}\rangle_{\LL}+\mathrm{Rem}_{k}(h-h^\star_{\omega,P},v-j^\star_{k,\omega,P}),
\end{aligned}
\end{equation*}
where the last line combines, using symmetry of $\langle\cdot,\cdot\rangle_{\LL}$,
\begin{align*}
    &-\langle a^\star_{\omega,P},v-j^\star_{k,\omega,P}\rangle_{\LL}+\langle j^\star_{k,\omega,P},\alpha_2\rangle_{\LL}-\langle v,\alpha_2\rangle_{\LL}\\
    =&-\langle a^\star_{\omega,P},v-j^\star_{k,\omega,P}\rangle_{\LL}-\langle\alpha_2,v-j^\star_{k,\omega,P}\rangle_{\LL}.
\end{align*}
\end{proof}

\section{Asymptotic Normality}
\label{sec:appendix_CLT}
For a sequence of random variables $(X_N)_{N\ge1}$ and a deterministic sequence of positive reals $(a_N)_{N\ge1}$, we write $X_N=o_p(a_N)$ if $X_N/a_N\to0$ in probability as $N\to\infty$, i.e., $P(|X_N/a_N|>\varepsilon)\to0$ for every $\varepsilon>0$; and $X_N=O_p(a_N)$ if $X_N/a_N$ is bounded in probability, i.e., for every $\varepsilon>0$ there exists $M<\infty$ such that $\limsup_{N\to\infty}P(|X_N/a_N|>M)<\varepsilon$.

For any nuisance tuple $\eta$, define the centered score
\begin{equation*}
    \chi_{\eta,\omega}(X,Y,Z)\coloneqq\varphi_\omega(\eta)-\mathbb{E}_P\bigl[\varphi_\omega(\eta)\bigr].
\end{equation*}
\textbf{Two-fold cross-fitted doubly robust estimator.}
We now construct a $\sqrt{N}$-consistent estimator for $\nabla\mathcal{F}(\omega)$. The construction follows the standard cross-fitting principle: estimate the nuisance functions on one fold, evaluate the orthogonal score on the held-out fold, and average \citep{bickel1993efficient,newey1994asymptotic,chernozhukov2018double}. 
Let $S_1$ and $S_2$ be two independent folds of i.i.d.\ observations from $P$, each of size $n$, and set $N=2n$. Write $P_{n,r}$ for the empirical measure on $S_r$, $r=1,2$. For each fold $r$, let $\hat\eta_\omega^{(-r)}=(\hat h_\omega^{(-r)},\hat j_\omega^{(-r)},\hat m^{(-r)})$ denote nuisance estimates trained on the opposite fold $S_{3-r}$, where $\hat j_\omega^{(-r)}=(\hat j_{1,\omega}^{(-r)},\ldots,\hat j_{d,\omega}^{(-r)})\in\LL^d$ and $\hat m^{(-r)}$ is an estimate of $m^\star_P$ trained on $S_{3-r}$. For any nuisance tuple $\eta=(h,j,m)$, let $\varphi_\omega(\eta)=(\varphi_{1,\omega}(\eta),\ldots,\varphi_{d,\omega}(\eta))^\top$. Consequently, the two-fold cross-fitted doubly robust estimator is
\begin{equation}
\label{eq:cf-debiased-estimator}
    \totgraddr{\omega}\coloneqq\frac{1}{2}\sum_{r=1}^2\mathbb{E}_{P_{n,r}}\bigl[\varphi_\omega(\hat\eta_\omega^{(-r)})\bigr].
\end{equation}
Under the conditions below, $\totgraddr{\omega}$ is asymptotically efficient at fixed $\omega$; we establish its $\sqrt N$ limit and the corresponding pointwise Wald intervals.

Since $\totgraddr{\omega}$ is an empirical average, its asymptotic distribution is obtained via a central limit theorem. For this to hold, the estimated score $\chi_{\hat\eta_\omega^{(-r)},\omega}$ must be close enough to the oracle $\chi_{P,\omega}$, and the nuisance estimation errors must vanish at a sufficient rate. The following assumption formalizes these requirements.

\begin{assum}[Nuisance rates]
\label{assum:second-order-rates}
For $r=1,2$,
\begin{gather*}
    \|\hat h_\omega^{(-r)}-h^\star_{\omega,P}\|_{\LL}=o_p(N^{-1/4}),\qquad
    \|\hat j_\omega^{(-r)}-j^\star_{\omega,P}\|_{\LL^d}=o_p(N^{-1/4}),\\
    \|\hat m^{(-r)}-m^\star_P\|_{\LL}=o_p(N^{-1/4}).
\end{gather*}
\end{assum}
These rates can be verified by standard least-squares regression arguments under appropriate approximation and entropy or Rademacher-complexity conditions; see Appendix~\ref{sec:nuisance-learning-discussion}.

\begin{remark}[Linear-plus-product score decomposition]\label{assum:score-stability}
Write $\eta^\star_\omega=(h^\star_{\omega,P},j^\star_{\omega,P},m_P^\star)$ and $\chi_{P,\omega}=\chi_{\eta^\star_\omega,\omega}$. For fold $r$, set $\Delta h_r=\hat h_\omega^{(-r)}-h^\star_{\omega,P}$, $\Delta j_r=\hat j_\omega^{(-r)}-j^\star_{\omega,P}$, $\Delta m_r=\hat m^{(-r)}-m_P^\star$, and $\Delta r_r=\Delta h_r-\Delta m_r$. Expanding the quadratic score gives exactly
\begin{align*}
 \varphi_{k,\omega}(\hat\eta_\omega^{(-r)})-\varphi_{k,\omega}(\eta^\star_\omega)
 &=\underbrace{\langle g_\omega(Z)-Y-h^\star_{\omega,P}(X)+m_P^\star(X),\Delta j_{k,r}(X)\rangle}_{L_{k,r}^{(1)}}\\
 &\quad+\underbrace{\langle\partial_{\omega_k}g_\omega(Z)-j^\star_{k,\omega,P}(X),\Delta r_r(X)\rangle}_{L_{k,r}^{(2)}}
 -\underbrace{\langle\Delta j_{k,r}(X),\Delta r_r(X)\rangle}_{Q_{k,r}}.
\end{align*}
The coefficients of the first two terms have conditional mean zero given $X$; under \Cref{assum:eif-integrability}, their norms are bounded by $2(A+B)$ and $2D$, respectively. Thus, holding the training fold fixed, $PL_{k,r}^{(1)}=PL_{k,r}^{(2)}=0$ and
\begin{equation*}
 \|L_{k,r}^{(1)}\|_{L^2(P)}\le2(A+B)\|\Delta j_{k,r}\|_{\LL},\qquad
 \|L_{k,r}^{(2)}\|_{L^2(P)}\le2D\|\Delta r_r\|_{\LL}.
\end{equation*}
The product term only needs an $L^1(P)$ bound:
\begin{align*}
 P|Q_{k,r}|&\le\|\Delta j_{k,r}\|_{\LL}\|\Delta r_r\|_{\LL}\\
 &\le\|\Delta j_r\|_{\LL^d}(\|\Delta h_r\|_{\LL}+\|\Delta m_r\|_{\LL})=o_p(N^{-1/2})
\end{align*}
by \Cref{assum:second-order-rates}. This linear-plus-product decomposition suffices for the empirical remainder; it does not assert $L^2(P)$ stability of the product or of the full fitted score.
\end{remark}

Let $\Sigma_{P,\omega}\coloneqq \mathbb{E}_P[\chi_{P,\omega}\chi_{P,\omega}^\top]\in\mathbb{R}^{d\times d}$ denote the covariance matrix of the oracle efficient influence function. The following theorem provides the asymptotic normality of the cross-fitted debiased estimator.

\begin{theorem}[Asymptotic normality of the cross-fitted debiased estimator]
\label{thm:CLT}
Under \Cref{assum:eif-integrability,assum:second-order-rates},
\begin{equation}
\label{eq:clt-expansion}
    \sqrt{N}\left(\totgraddr{\omega}-\nabla\mathcal{F}(\omega)\right)=\frac{1}{\sqrt{N}}\sum_{r=1}^2\sum_{(X_i,Y_i,Z_i)\in S_r}\chi_{P,\omega}(X_i,Y_i,Z_i)+o_p(1).
\end{equation}
Consequently,
\begin{equation*}
    \sqrt{N}\left(\totgraddr{\omega}-\nabla\mathcal{F}(\omega)\right)\xrightarrow{d}\mathcal N(0,\Sigma_{P,\omega}).
\end{equation*}
\end{theorem}

The expansion \eqref{eq:clt-expansion} separates the statistical error into three parts: the oracle empirical average, a stochastic equicontinuity term, and a population bias remainder controlled by Assumption~\ref{assum:second-order-rates}. The correction terms make the estimating function first-order insensitive to nuisance perturbations at the oracle nuisances, so only product-rate and quadratic remainders remain.

\begin{proof}[Proof of Theorem~\ref{thm:CLT}]
For $r=1,2$, write $\hat\eta_r\coloneqq\hat\eta_\omega^{(-r)}$ for brevity. By definition of the two-fold estimator and since $\mathbb{E}_P[\varphi_\omega(\cdot;\eta^\star_\omega)]=\nabla\mathcal{F}(\omega)$,
\begin{equation*}
    \totgraddr{\omega}-\nabla\mathcal{F}(\omega)=\frac{1}{2}\sum_{r=1}^2\left\{\mathbb{E}_{P_{n,r}}\bigl[\varphi_\omega(\cdot;\hat\eta_r)\bigr]-\mathbb{E}_P\bigl[\varphi_\omega(\cdot;\eta^\star_\omega)\bigr]\right\}.
\end{equation*}
Adding and subtracting $\mathbb{E}_P[\varphi_\omega(\cdot;\hat\eta_r)]$ and $\mathbb{E}_{P_{n,r}}[\varphi_\omega(\cdot;\eta^\star_\omega)]$ for each fold $r$ yields the decomposition
\begin{equation*}
    \totgraddr{\omega}-\nabla\mathcal{F}(\omega)=A_N+B_N+R_N,
\end{equation*}
where
\begin{align*}
    A_N&\coloneqq\frac{1}{2}\sum_{r=1}^2\bigl(\mathbb{E}_{P_{n,r}}-\mathbb{E}_P\bigr)\bigl[\chi_{P,\omega}\bigr],\\
    B_N&\coloneqq\frac{1}{2}\sum_{r=1}^2\bigl(\mathbb{E}_{P_{n,r}}-\mathbb{E}_P\bigr)\bigl[\chi_{\hat\eta_r,\omega}-\chi_{P,\omega}\bigr],\\
    R_N&\coloneqq\frac{1}{2}\sum_{r=1}^2\left\{\mathbb{E}_P\bigl[\varphi_\omega(\cdot;\hat\eta_r)\bigr]-\mathbb{E}_P\bigl[\varphi_\omega(\cdot;\eta^\star_\omega)\bigr]\right\}.
\end{align*}
We show that $\sqrt{N}\,A_N\xrightarrow{d}\mathcal{N}(0,\Sigma_{P,\omega})$ and $B_N,R_N=o_p(N^{-1/2})$.

\paragraph{Term $A_N$.}
Since the two folds are independent and each has size $n=N/2$,
\begin{equation*}
    A_N=\frac{1}{N}\sum_{r=1}^2\sum_{(X_i,Y_i,Z_i)\in S_r}\chi_{P,\omega}(X_i,Y_i,Z_i).
\end{equation*}
Under \Cref{assum:eif-integrability}, each term of the efficient influence function \eqref{eq:eif} is square-integrable. Indeed, the almost sure bounds $\|Y\|\leq A$, $\sup_\omega\|g_\omega(Z)\|\leq B$, and $\sup_{\omega\in\Omega}\sqrt{\sum_{k=1}^{d}\|\partial_{\omega_k}g_\omega(Z)\|^2}\leq D$ imply that $h^\star_{\omega,P}$ and $j^\star_{k,\omega,P}$ are bounded in $\LL$, while $a^\star_{\omega,P}=m^\star_P-h^\star_{\omega,P}\in\LL$ follows since it is a difference of two bounded functions. Together these imply $\mathbb{E}_P\bigl[\|\chi_{P,\omega}\|_2^2\bigr]<\infty$. Therefore the multivariate central limit theorem gives
\begin{equation*}
    \sqrt{N}\,A_N=\frac{1}{\sqrt{N}}\sum_{r=1}^2\sum_{(X_i,Y_i,Z_i)\in S_r}\chi_{P,\omega}(X_i,Y_i,Z_i)\xrightarrow{d}\mathcal{N}(0,\Sigma_{P,\omega}).
\end{equation*}

\paragraph{Term $B_N$.}
Fix $r\in\{1,2\}$ and condition on $S_{3-r}$ (and any independent fitting randomness), so the nuisances are fixed and the evaluation fold is independent. Since empirical centering annihilates constants,
\begin{equation*}
 (P_{n,r}-P)(\chi_{\hat\eta_r,\omega}-\chi_{P,\omega})
 =(P_{n,r}-P)\{\varphi_\omega(\hat\eta_r)-\varphi_\omega(\eta^\star_\omega)\}.
\end{equation*}
Use Remark~\ref{assum:score-stability} and put $L_{k,r}=L_{k,r}^{(1)}+L_{k,r}^{(2)}$. Its conditional mean is zero, and
\begin{equation*}
 \mathbb E\left[\left|\sqrt n(P_{n,r}-P)L_{k,r}\right|^2\,\middle|\,S_{3-r}\right]
 \le PL_{k,r}^2
 \le C\{\|\Delta j_{k,r}\|_{\LL}^2+\|\Delta r_r\|_{\LL}^2\}=o_p(1).
\end{equation*}
Conditional Chebyshev, followed by averaging the conditional probability bound truncated at one, gives $(P_{n,r}-P)L_{k,r}=o_p(n^{-1/2})$.
For the product, use $q_{k,r}:=P|Q_{k,r}|=o_p(N^{-1/2})$ and
\begin{equation*}
 |(P_{n,r}-P)Q_{k,r}|\le P_{n,r}|Q_{k,r}|+q_{k,r},\qquad
 \mathbb E[P_{n,r}|Q_{k,r}|\mid S_{3-r}]=q_{k,r}.
\end{equation*}
For every $\varepsilon>0$, conditional Markov gives
\begin{equation*}
 \mathbb P\{\sqrt N P_{n,r}|Q_{k,r}|>\varepsilon\}
 \le\mathbb E\!\left[\min\{1,\sqrt N q_{k,r}/\varepsilon\}\right]\longrightarrow0,
\end{equation*}
because the bounded integrand converges to zero in probability. Hence $(P_{n,r}-P)Q_{k,r}=o_p(N^{-1/2})$, without an $L^2(P)$ assumption on $Q_{k,r}$. As $d$ is fixed and $N=2n$, summing coordinates and averaging the two folds gives $B_N=o_p(N^{-1/2})$.

\paragraph{Term $R_N$.}
We decompose the population remainder coordinatewise. For each $r=1,2$ and $k=1,\dots,d$, Theorem~\ref{thm:VME} gives
\begin{equation*}
    \mathbb{E}_P\bigl[\varphi_\omega(\cdot;\hat\eta_r)\bigr]-\mathbb{E}_P\bigl[\varphi_\omega(\cdot;\eta^\star_\omega)\bigr]=C_r+D_r+E_r,
\end{equation*}
where $C_r=(C_{r,1},\dots,C_{r,d})^\top$, $D_r=(D_{r,1},\dots,D_{r,d})^\top$, $E_r=(E_{r,1},\dots,E_{r,d})^\top$, and
\begin{align*}
    C_{r,k}&\coloneqq\left\langle j^\star_{k,\omega,P}-\hat j_{k,\omega}^{(-r)},\hat h_\omega^{(-r)}-h^\star_{\omega,P}\right\rangle_{\LL},\\
    D_{r,k}&\coloneqq-\left\langle a^\star_{\omega,P}+\hat h_\omega^{(-r)}-\hat m^{(-r)},\hat j_{k,\omega}^{(-r)}-j^\star_{k,\omega,P}\right\rangle_{\LL},\\
    E_{r,k}&\coloneqq\mathrm{Rem}_{k}\left(\hat h_\omega^{(-r)}-h^\star_{\omega,P},\hat j_{k,\omega}^{(-r)}-j^\star_{k,\omega,P}\right).
\end{align*}
By the Cauchy-Schwarz inequality,
\begin{align*}
    \|C_r\|&\leq\left\|\hat j_\omega^{(-r)}-j^\star_{\omega,P}\right\|_{\LL^d}\left\|\hat h_\omega^{(-r)}-h^\star_{\omega,P}\right\|_{\LL}=o_p(N^{-1/2}),\\
    \|D_r\|&\leq\left\|\hat h_\omega^{(-r)}-\hat m^{(-r)}+a^\star_{\omega,P}\right\|_{\LL}\left\|\hat j_\omega^{(-r)}-j^\star_{\omega,P}\right\|_{\LL^d}\\
    &\leq\left(\left\|\hat h_\omega^{(-r)}-h^\star_{\omega,P}\right\|_{\LL}+\left\|\hat m^{(-r)}-m^\star_P\right\|_{\LL}\right)\left\|\hat j_\omega^{(-r)}-j^\star_{\omega,P}\right\|_{\LL^d}=o_p(N^{-1/2}),
\end{align*}
using $a^\star_{\omega,P}=m^\star_P-h^\star_{\omega,P}$ and the triangle inequality for the second bound, by Assumption~\ref{assum:second-order-rates}. For the remainder, Lemma~\ref{lem:quadratic_expansion} gives
\begin{equation*}
    \|E_r\|\leq\frac{1}{2}\left(\left\|\hat h_\omega^{(-r)}-h^\star_{\omega,P}\right\|_{\LL}^2+\left\|\hat j_\omega^{(-r)}-j^\star_{\omega,P}\right\|_{\LL^d}^2\right)=o_p(N^{-1/2}),
\end{equation*}
by the first two lines of Assumption~\ref{assum:second-order-rates}. Summing over both folds gives $R_N=o_p(N^{-1/2})$.

\paragraph{Conclusion.}
Combining the three terms,
\begin{equation*}
    \sqrt{N}\left(\totgraddr{\omega}-\nabla\mathcal{F}(\omega)\right)=\frac{1}{\sqrt{N}}\sum_{r=1}^2\sum_{(X_i,Y_i,Z_i)\in S_r}\chi_{P,\omega}(X_i,Y_i,Z_i)+o_p(1).
\end{equation*}
The central limit theorem applied to $A_N$ and Slutsky's theorem give
\begin{equation*}
    \sqrt{N}\left(\totgraddr{\omega}-\nabla\mathcal{F}(\omega)\right)\xrightarrow{d}\mathcal{N}(0,\Sigma_{P,\omega}).
\end{equation*}
\end{proof}

\paragraph{Confidence intervals for fixed-parameter gradient coordinates.}
The asymptotic covariance $\Sigma_{P,\omega}$ is estimated by the empirical covariance of the cross-fitted pseudo-outcomes:
\begin{equation*}
    \widehat{\Sigma}_{P,\omega}\coloneqq\frac{1}{2}\sum_{r=1}^2P_{n,r}\left[\left(\varphi_\omega(\hat\eta_\omega^{(-r)})-\totgraddr{\omega}\right)\left(\varphi_\omega(\hat\eta_\omega^{(-r)})-\totgraddr{\omega}\right)^\top\right].
\end{equation*}
Under $\widehat{\Sigma}_{P,\omega}\xrightarrow{p}\Sigma_{P,\omega}$, for a coordinate $k$ with $\Sigma_{P,\omega,kk}>0$, an asymptotically valid $(1-\alpha)$ Wald confidence interval for the $k$-th coordinate of $\nabla\mathcal{F}(\omega)$ is
\begin{equation}
\label{eq:wald-ci}
    \left[[\totgraddr{\omega}]_k\pm z_{1-\alpha/2}\sqrt{\frac{\widehat{\Sigma}_{P,\omega,kk}}{N}}\right],
\end{equation}
where $z_{1-\alpha/2}$ is the $(1-\alpha/2)$ quantile of the standard normal distribution. These are coordinatewise, pointwise intervals at a fixed parameter; they do not provide simultaneous projected-stationarity certificates or inference for an adaptively selected optimizer.

\begin{remark}
Theorem~\ref{thm:CLT} and \eqref{eq:wald-ci} concern a fixed, non-data-dependent $\omega$. If $\omega$ is selected using the same observations, pointwise Wald validity does not follow from the uniform estimation-error bounds of Appendix~\ref{sec:appendix_uniform}. A simple construction selects $\omega$ on independent data and performs inference on a fresh held-out sample, treating the selected value as conditionally fixed and verifying the inference conditions there. More general post-selection inference requires additional simultaneous or selective-inference arguments not developed here.
\end{remark}

\section{Additional Experiment Details}
\label{app:experiment-details}
\raggedbottom

This appendix gives the data-generating processes (DGPs), implementation details, and numerical outputs behind Section~\ref{sec:experiments}. We first give the matched optimization experiments, then retain the supplementary fixed-gradient, inference, regularization, and Gaussian root-estimation designs. Throughout, ``Plug-in'' (PI) denotes the direct plug-in gradient and OBiGrad denotes its orthogonal counterpart. The original inference DGPs, evaluation points, and basis learners are preserved; KBO and KBO+OBiGrad are added on the same datasets and folds.

\paragraph{Computation.}
All experiments were run locally on a CPU-only Apple MacBook Pro with an Apple M4 chip, 10 CPU cores, 24GB memory, macOS 15.0, and arm64 Darwin kernel. No external GPU, cluster, or remote compute was used. All scripts use fixed random seeds and write configuration files alongside the results.

\paragraph{Reporting conventions.}
Gradient errors are Euclidean RMSEs of $\widehat{\nabla\mathcal F}^{\mathrm{OBi}}(\omega)$ relative to the unregularized population target $\nabla\mathcal F(\omega)$. Wald coverage is coordinate-wise and averaged across gradient coordinates, using the confidence intervals \eqref{eq:wald-ci}. KBO results are reported relative to $\nabla\mathcal F(\omega)$, even though fixed-$\lambda$ KBO targets a regularized gradient $\nabla\mathcal F_\lambda(\omega)$. In KBO tables, ``KBO total'' is $\|\widehat{\nabla\mathcal F}^{\mathrm{KBO}}_\lambda(\omega)-\nabla\mathcal F(\omega)\|$, ``KBO estimation'' is $\|\widehat{\nabla\mathcal F}^{\mathrm{KBO}}_\lambda(\omega)-\nabla\mathcal F_\lambda(\omega)\|$, and ``Reg.\ bias'' is $\|\nabla\mathcal F_\lambda(\omega)-\nabla\mathcal F(\omega)\|$, so that by the triangle inequality
\begin{equation*}
    \|\widehat{\nabla\mathcal F}^{\mathrm{KBO}}_\lambda(\omega)-\nabla\mathcal F(\omega)\|\leq\|\widehat{\nabla\mathcal F}^{\mathrm{KBO}}_\lambda(\omega)-\nabla\mathcal F_\lambda(\omega)\|+\|\nabla\mathcal F_\lambda(\omega)-\nabla\mathcal F(\omega)\|.
\end{equation*}
Monte Carlo error bars are $95\%$ half-widths across complete simulated datasets, using $1.96$ times the Monte Carlo standard error unless a simultaneous adjustment is stated. Coordinates within a replication are not treated as independent replications. Gradient RMSE uses the Euclidean norm; ``Coord.\ RMSE'' in the inference tables is the Euclidean RMSE divided by $\sqrt d$. Studentized quantiles are descriptive pooled-coordinate diagnostics. Nuisance sensitivity norms use the vector $L^2$ norm, so the reported product is a bound diagnostic rather than the observed signed bias.

The fixed-gradient, Wald, and Gaussian Bellman-root summaries pool the original-seed rerun and an independent rerun with the master seed increased by $104729$. This gives $600$ IV and $400$ Bellman fixed-gradient replications, $1000$ IV and $400$ Bellman Wald replications, and $400$ Bellman-root replications per $N$. The matched legacy KBO and scalar-IV-root tables use their original-seed reruns ($300$ IV, $200$ Bellman, and $300$ scalar-root replications); separate independent-seed outputs are retained with the reproducibility files. No old finite-sample ``population proxy'' is used as ground truth: kernel population targets are computed by converged quadrature, conditional on the fixed RFF map where applicable. The displayed KBO decomposition is a triangle bound, not an additive decomposition of RMSE.

\subsection{Population optimization with matched nuisance learning}
\label{app:optimization-details}

The three experiments in Section~\ref{sec:experiments} examine stationary-point estimation, population stationarity of a returned optimization iterate, and robustness to nonzero conditional residuals. They use the same simulated datasets for the three analyses. The DGPs and nuisance learners below differ from the supplementary Gaussian, basis-regression designs in Appendices~\ref{app:iv-experiments} and~\ref{app:fitted-q-experiments}.

\paragraph{IV DGP.}
Let $X\sim\mathrm{Unif}[-1,1]$, $e_1(X)=\sqrt3X$, and $e_2(X)=\sqrt5(3X^2-1)/2$. Define
\begin{equation*}
    j^\star(X)=e_1(X)+\tfrac12 e_2(X),\qquad
    r(X)=\tfrac12e_1(X)-e_2(X),\qquad
    Z=j^\star(X)+U,
\end{equation*}
and $Y=0.75Z+\delta r(X)+0.25U+\varepsilon$. Independently of $X$ and each other, $U$ and $\varepsilon$ are centered uniform variables with standard deviations $1$ and $0.25$. We take $g_\omega(Z)=\omega Z$. Thus $h^\star_\omega=\omega j^\star$, $m^\star=0.75j^\star+\delta r$, and $\mathbb E[j^\star r]=0$. The population gradient is $1.25(\omega-0.75)$ for both $\delta=0$ and $\delta=0.25$. The first-stage $R^2$ is $1.25/2.25\simeq0.556$. The case $\delta=0$ is a triangular IV model with endogenous treatment; $\delta=0.25$ is an approximate conditional-moment problem, since exact conditional fit no longer holds. The learner receives $(X,Z,Y)$, without the polynomial basis, conditional means, or population coefficients.

\paragraph{Fitted $Q$-regression DGP.}
Draw $S\sim\mathrm{Unif}[-\sqrt3,\sqrt3]$ and $A\sim\mathrm{Bernoulli}(0.5)$ independently, and set
\begin{equation*}
    S'=0.7S+0.5A+\xi,\qquad
    R=\sin S+0.5A+0.25SA+\varepsilon_R,\qquad
    D=0.8(\sin S',\cos S')^\top.
\end{equation*}
Here $\xi$, $\varepsilon_R$, and $\varepsilon_Y$ are independent centered uniform noises with standard deviations $0.2$, $0.1$, and $0.1$. With $X=(S,A)$, $g_\omega=R+D^\top\omega$, and $\omega^\star=(0.65,-0.45)^\top$, generate $Y=R+D^\top\omega^\star+\delta r(X)+\varepsilon_Y$. Writing $j^\star=\mathbb E[D\mid X]$, the residual is
\begin{equation*}
    r(X)=\frac{q(X)-j^\star(X)^\top c}{\|q-j^{\star\top}c\|_{\LL}},\quad
    q(S,A)=\sin(2S)+0.5(A-0.5),\quad
    c=\{\mathbb E[j^\star j^{\star\top}]\}^{-1}\mathbb E[j^\star q].
\end{equation*}
This construction preserves $\mathbb E[j^\star r]=0$ and hence the optimizer. Population projection coefficients are used by the simulator and evaluator only. The conditional sensitivity is $0.8\operatorname{sinc}(\sqrt3\,0.2)(\sin\mu,\cos\mu)^\top$, where $\mu=0.7S+0.5A$ and $\operatorname{sinc}(u)=\sin(u)/u$. Population integrals use $160$-node Gauss--Legendre quadrature over the state and exact summation over actions. This is one Bellman-regression problem, with bounded state and noise distributions, rather than a complete fitted $Q$-iteration algorithm.

\paragraph{Matched nuisance fits.}
Both learners use two-fold cross-fitting on the same $N$ complete observations. Training-fold means and standard deviations standardize the covariates. NW uses a Gaussian kernel and three-fold training-only cross-validation over bandwidths $\{0.05,0.075,0.1,0.15,0.225,0.3375,0.5,0.75,1\}$. KBO uses an exact Gaussian kernel whose bandwidth is the median positive pairwise distance among at most $256$ evenly indexed standardized training covariates, and the penalty grid
\begin{equation*}
    \Lambda=\{10^{-j},3\cdot10^{-j}:j=2,\ldots,7\}\cup\{10^{-1}\}.
\end{equation*}
Cross-validation minimizes prediction error for the observed inner targets: $Z$ in IV and $(R,D)$ in fitted $Q$-regression, with scales computed on training data. It does not use $Y$, population gradients, or the optimizer. All chosen hyperparameters are fixed in $\omega$ and shared within a learner pair. The outcome regression uses the same selected operator and training observations. No nuisance is refitted during outer optimization.

\paragraph{KBO specialization and orthogonal counterpart.}
For the affine inner signal $g_\omega=b+D^\top\omega$, the implemented kernel system is
\begin{equation*}
    [\hat b,\hat J,\hat m]
    =K_{\mathrm{eval},\mathrm{tr}}(K_{\mathrm{tr},\mathrm{tr}}+n_{\mathrm{tr}}\lambda I)^{-1}
    [b_{\mathrm{tr}},D_{\mathrm{tr}},Y_{\mathrm{tr}}].
\end{equation*}
Consequently $\hat h_\omega=\hat b+\hat J\omega$ and $\partial_\omega\hat h_\omega=\hat J$ exactly. The existing KBO gradient in this specialization is $P_{\mathrm{eval}}[\hat j(\hat h_\omega-Y)]$, the derivative of the empirical outer squared loss. Its matched counterpart is
\begin{equation*}
    P_{\mathrm{eval}}\!\left[\hat j(g_\omega-Y)+(D-\hat j)(\hat h_\omega-\hat m)\right].
\end{equation*}
We implement explicit sensitivities. The same score is obtained by differentiating $P_{\mathrm{eval}}[\tfrac12(Y-\hat h_\omega)^2+(g_\omega-\hat h_\omega)(\hat h_\omega-\hat m)]$, holding the regression operator fixed. The additional cost is one outcome right-hand side and its predictions using the existing factorization; no extra observations are required. This construction does not equate a general KBO adjoint with $m-h$. Numerical tests verify native-gradient equivalence, sensitivity and corrected-objective derivatives, zero-correction parity, and oracle controls.

Conditional on training data, the plug-in score targets the gradient of the fitted outer objective. For deterministic population fits at penalty $\lambda$, its target is $\nabla\mathcal F_\lambda$; the orthogonal target differs from $\nabla\mathcal F$ by
\begin{equation*}
    \mathbb E\!\left[(j_\lambda-j^\star)
    \{(m_\lambda-m^\star)-(h_\lambda-h^\star)\}\right].
\end{equation*}
Thus fixed regularization can bias both methods. The mixed identity was checked against population integration for the saved main fits, with maximum discrepancy $6.2\cdot10^{-16}$. For RFF fits in the supplementary experiments, the corresponding system is $(\Phi_{\mathrm{tr}}^\top\Phi_{\mathrm{tr}}+n_{\mathrm{tr}}\lambda I)^{-1}\Phi_{\mathrm{tr}}^\top$ followed by evaluation features; all equivalences remain valid for a fixed feature realization.

\paragraph{Protocol and Monte Carlo comparisons.}
The four settings are IV with $\delta=0,0.25$ (cases $0,1$), and fitted $Q$-regression with $\delta=0,0.25$ (cases $2,3$). After three exploratory pilots, we fixed the confirmation protocol and ran master seeds $2026092511$ and $2026092512$, each with $400$ replications per case and $N\in\{200,400,800,1600\}$. Pooling gives $800$ replications per cell and $12{,}800$ datasets. Tables~\ref{tab:optim-case0}--\ref{tab:optim-case3} report the complete primary grid. Parentheses give pointwise Monte Carlo $95\%$ half-widths; RMSE uncertainty uses the delta method on squared errors. The paired comparisons in Tables~\ref{tab:optim-paired0}--\ref{tab:optim-paired3} use differences on the same dataset. The declared family contains $64$ contrasts: two endpoints, two learners, four cases, and four sample sizes. Approximate simultaneous intervals use the normal critical value $3.3594$ (Bonferroni correction). All pooled lower limits are positive. The separate seed suites agree in direction; their adjusted intervals are positive for all $64$ root comparisons and $63$ of $64$ stationarity comparisons across the two suites. The confirmation replication counts were not adapted to the observed differences.

\paragraph{Experiment 1: stationary-point estimation.}
We average the two held-out gradient fields before solving the affine estimating equation. Roots are evaluated against the unique unregularized population optimizer. Across the primary grid, root RMSE reductions range from $35.5\%$ to $69.8\%$ for NW and from $5.2\%$ to $34.9\%$ for KBO. There were no singular-root failures, non-positive-definite fitted quadratic objectives, or roots outside $[-2,2]^d$ for the primary cross-validated methods. The finite-sample oracle estimator uses true nuisances and the same evaluation observations; it is distinct from exact population-gradient descent.

\paragraph{Experiment 2: population optimization.}
Projected descent starts at zero, uses $\eta=0.1$, and runs for $1000$ updates in $\Omega=[-2,2]^d$. The fitted field remains fixed. At each budget $T$, we select the first minimizer
\begin{equation*}
    \hat t_T=\min\argmin_{0\leq t<T}\|G_\eta(\omega_t;\widehat{\nabla\mathcal F}(\omega_t))\|.
\end{equation*}
We report $\|G_\eta(\omega_{\hat t_T})\|$, population excess objective, and distance to $\omega^\star$. Selection is feasible and does not use the smallest population norm along the trajectory. Indeed, the proof of \Cref{cor:gradient-descent-convergence} bounds the minimum estimated mapping; nonexpansiveness gives, with $\Delta=\sup_\Omega\|\widehat{\nabla\mathcal F}-\nabla\mathcal F\|$,
\begin{equation*}
    \|G_\eta(\omega_{\hat t_T})\|\leq
    2\sqrt{\frac{\mathcal F(\omega_0)-F_{\inf}}{\eta T}}+3\Delta.
\end{equation*}
For the two-fold average, the uniform field error is at most the average of the two foldwise uniform errors; this step requires no independence between foldwise fields. These observations connect the reported statistic to the theory, without claiming that prediction cross-validation verifies all of its nuisance-rate and entropy assumptions. The main affine population objectives are strongly convex, and late-iteration and root errors are closely related. Correction reduces the final statistical stationarity error, not necessarily the number of iterations needed for a fixed estimated objective. No primary trajectories failed or ended on the boundary.

\paragraph{Experiment 3: nonzero residual and regularization.}
With $\delta=0$, a common linear population smoother satisfies $m_\lambda=h_{\omega^\star,\lambda}$, which makes the corrected population root exact despite shrinkage. This does not make its entire gradient field exact. The $\delta=0.25$ cases remove this alignment and retain improvement under the primary cross-validated protocol (Figure~\ref{fig:optim-residual}). Figure~\ref{fig:optim-penalties} reports the full penalty sweep at $N=800$, and Figure~\ref{fig:optim-decreasing} retains $\lambda_N=0.05N^{-0.6}$. The decreasing-penalty root RMSE is lower after correction in all $16$ case/sample-size cells. These are descriptive secondary comparisons. The full grid contains a near-reversal for residual fitted $Q$-regression at $N=800$, $\lambda=10^{-7}$: KBO RMSE is $0.0263857$, versus $0.0263907$ after correction. We do not claim uniform superiority over penalties.

A deterministic finite-kernel calculation isolates the bias mechanism. For $k(x,x')=e_1(x)e_1(x')+0.25e_2(x)e_2(x')$, let $s_1=(1+\lambda)^{-1}$, $s_2=0.25/(0.25+\lambda)$, and $w_i=2s_i-s_i^2$. Then
\begin{equation*}
    \omega_{\mathrm{PI},\lambda}=
    \frac{0.75(s_1+0.25s_2)+0.5\delta(s_1-s_2)}{s_1^2+0.25s_2^2},\qquad
    \omega_{\mathrm{OBi},\lambda}=0.75+\frac{0.5\delta(w_1-w_2)}{w_1+0.25w_2}.
\end{equation*}
For nonzero $\delta$, root displacement is first order in $\lambda$ for PI and second order for OBiGrad (Figure~\ref{fig:optim-population-control}). This is an analytic mechanism illustration, not a fitted convergence rate for the empirical experiments.

\paragraph{Additional checks and limitations.}
A nonlinear fitted $Q$-regression diagnostic uses $v(\omega)=(\omega_1,\omega_2+\sin\omega_1)$ and transforms each fitted gradient by $Dv(\omega)^\top$. The reparameterized population Hessian has a negative eigenvalue of approximately $-1.072$ at $(\pi/2,2)$; the common step satisfies $\eta L\leq0.237$. Replaying the $6400$ Bellman datasets, correction improves all $16$ matched returned-stationarity comparisons and all $16$ excess-objective comparisons under a separate approximate simultaneous family of $32$ endpoints (Figure~\ref{fig:optim-nonlinear}). The map is invertible, so this is not a multiple-basin benchmark or independent replication.

A Gaussian sensitivity study uses $2400$ fresh datasets at $N=200,800$. Its separate adjusted comparisons support $15$ of $16$ matched root and stationarity improvements; residual Bellman KBO at $N=800$ has a $2.1\%$ root-RMSE reduction with an interval including zero. The Gaussian design is outside the bounded-data setting and is not pooled with the primary study. Exploratory higher-noise settings did not uniformly favor correction; all pilot configurations and outcomes are retained in the supplementary result files. Pointwise intervals support reduced uniform gradient error in $24$ of $32$ primary comparisons, with eight unresolved. Better root accuracy does not establish universally better field estimation or verify the general asymptotic exponents.

\begin{figure}[H]\centering\includegraphics[width=\linewidth]{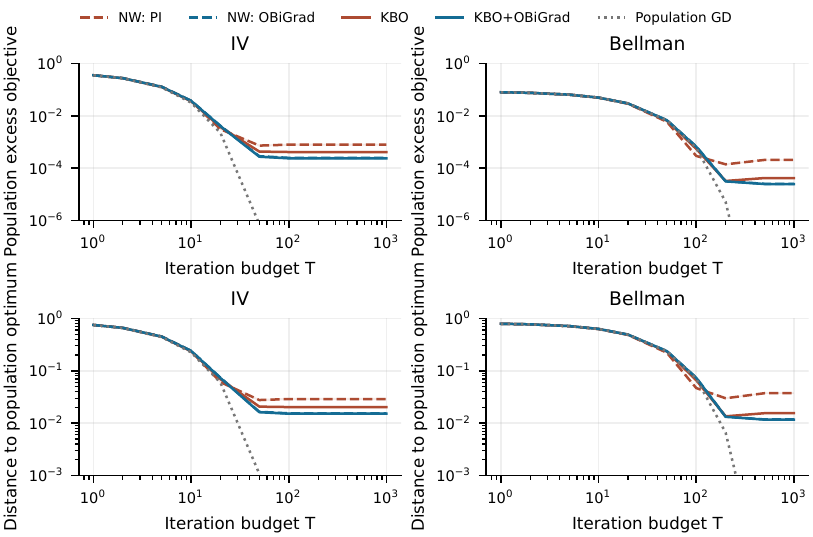}
\caption{Population excess objective (top) and distance to the population optimizer (bottom) at the selected iterate, with $N=400$ and $\delta=0$. Bands give pointwise Monte Carlo $95\%$ intervals. Population GD eventually falls below the plotted range.}\label{fig:optim-objective-distance}\end{figure}

\begin{table}[H]\centering\small\setlength{\tabcolsep}{3pt}
\caption{IV, $\delta=0$: root RMSE and mean returned population mapping norm. Parentheses are pointwise $95\%$ Monte Carlo half-widths; $800$ replications per $N$.}
\label{tab:optim-case0}\begin{tabular}{rlrrrr}
\toprule
$N$ & Learner & PI root RMSE & OBi root RMSE & PI mapping & OBi mapping \\
\midrule
200 & NW & 0.0511 (0.0023) & 0.0313 (0.0020) & 0.0516 (0.0026) & 0.0297 (0.0018) \\
200 & KBO & 0.0444 (0.0026) & 0.0317 (0.0021) & 0.0430 (0.0024) & 0.0299 (0.0018) \\
400 & NW & 0.0355 (0.0015) & 0.0199 (0.0011) & 0.0362 (0.0018) & 0.0193 (0.0011) \\
400 & KBO & 0.0257 (0.0013) & 0.0194 (0.0010) & 0.0254 (0.0014) & 0.0190 (0.0011) \\
800 & NW & 0.0251 (0.0010) & 0.0124 (0.0006) & 0.0260 (0.0012) & 0.0122 (0.0007) \\
800 & KBO & 0.0153 (0.0007) & 0.0123 (0.0006) & 0.0153 (0.0008) & 0.0122 (0.0007) \\
1600 & NW & 0.0200 (0.0007) & 0.0085 (0.0004) & 0.0213 (0.0009) & 0.0085 (0.0004) \\
1600 & KBO & 0.0096 (0.0005) & 0.0084 (0.0004) & 0.0096 (0.0005) & 0.0084 (0.0004) \\
\bottomrule
\end{tabular}
\end{table}
\begin{table}[H]\centering\small\setlength{\tabcolsep}{3pt}
\caption{IV, $\delta=0.25$: root RMSE and returned population mapping norm, with the same protocol as Table~\ref{tab:optim-case0}.}
\label{tab:optim-case1}\begin{tabular}{rlrrrr}
\toprule
$N$ & Learner & PI root RMSE & OBi root RMSE & PI mapping & OBi mapping \\
\midrule
200 & NW & 0.0666 (0.0030) & 0.0375 (0.0020) & 0.0676 (0.0034) & 0.0363 (0.0021) \\
200 & KBO & 0.0539 (0.0029) & 0.0393 (0.0022) & 0.0531 (0.0029) & 0.0378 (0.0022) \\
400 & NW & 0.0464 (0.0020) & 0.0239 (0.0012) & 0.0476 (0.0023) & 0.0236 (0.0013) \\
400 & KBO & 0.0309 (0.0015) & 0.0241 (0.0012) & 0.0308 (0.0016) & 0.0239 (0.0013) \\
800 & NW & 0.0354 (0.0014) & 0.0168 (0.0008) & 0.0375 (0.0016) & 0.0168 (0.0009) \\
800 & KBO & 0.0205 (0.0010) & 0.0169 (0.0009) & 0.0202 (0.0011) & 0.0168 (0.0009) \\
1600 & NW & 0.0273 (0.0010) & 0.0116 (0.0005) & 0.0288 (0.0013) & 0.0117 (0.0006) \\
1600 & KBO & 0.0129 (0.0006) & 0.0115 (0.0005) & 0.0127 (0.0007) & 0.0116 (0.0006) \\
\bottomrule
\end{tabular}
\end{table}
\begin{table}[H]\centering\small\setlength{\tabcolsep}{3pt}
\caption{Fitted $Q$-regression, $\delta=0$: root RMSE and returned population mapping norm.}
\label{tab:optim-case2}\begin{tabular}{rlrrrr}
\toprule
$N$ & Learner & PI root RMSE & OBi root RMSE & PI mapping & OBi mapping \\
\midrule
200 & NW & 0.0620 (0.0024) & 0.0188 (0.0007) & 0.0136 (0.0005) & 0.0049 (0.0002) \\
200 & KBO & 0.0287 (0.0012) & 0.0187 (0.0007) & 0.0068 (0.0003) & 0.0049 (0.0002) \\
400 & NW & 0.0412 (0.0012) & 0.0132 (0.0004) & 0.0092 (0.0003) & 0.0035 (0.0001) \\
400 & KBO & 0.0177 (0.0007) & 0.0131 (0.0004) & 0.0043 (0.0002) & 0.0035 (0.0001) \\
800 & NW & 0.0283 (0.0009) & 0.0094 (0.0003) & 0.0063 (0.0002) & 0.0025 (0.0001) \\
800 & KBO & 0.0117 (0.0005) & 0.0093 (0.0003) & 0.0029 (0.0001) & 0.0025 (0.0001) \\
1600 & NW & 0.0211 (0.0005) & 0.0067 (0.0002) & 0.0048 (0.0001) & 0.0018 (0.0001) \\
1600 & KBO & 0.0077 (0.0003) & 0.0067 (0.0002) & 0.0019 (0.0001) & 0.0018 (0.0001) \\
\bottomrule
\end{tabular}
\end{table}
\begin{table}[H]\centering\small\setlength{\tabcolsep}{3pt}
\caption{Fitted $Q$-regression, $\delta=0.25$: root RMSE and returned population mapping norm.}
\label{tab:optim-case3}\begin{tabular}{rlrrrr}
\toprule
$N$ & Learner & PI root RMSE & OBi root RMSE & PI mapping & OBi mapping \\
\midrule
200 & NW & 0.0847 (0.0034) & 0.0516 (0.0020) & 0.0191 (0.0007) & 0.0129 (0.0005) \\
200 & KBO & 0.0593 (0.0025) & 0.0515 (0.0020) & 0.0143 (0.0005) & 0.0129 (0.0005) \\
400 & NW & 0.0590 (0.0022) & 0.0380 (0.0014) & 0.0135 (0.0005) & 0.0097 (0.0003) \\
400 & KBO & 0.0412 (0.0016) & 0.0379 (0.0014) & 0.0102 (0.0004) & 0.0096 (0.0003) \\
800 & NW & 0.0416 (0.0016) & 0.0265 (0.0010) & 0.0095 (0.0003) & 0.0067 (0.0002) \\
800 & KBO & 0.0284 (0.0012) & 0.0264 (0.0010) & 0.0070 (0.0003) & 0.0066 (0.0002) \\
1600 & NW & 0.0299 (0.0010) & 0.0185 (0.0007) & 0.0068 (0.0002) & 0.0047 (0.0002) \\
1600 & KBO & 0.0195 (0.0007) & 0.0185 (0.0007) & 0.0049 (0.0002) & 0.0047 (0.0002) \\
\bottomrule
\end{tabular}
\end{table}

\begin{figure}[H]\centering\includegraphics[width=\linewidth]{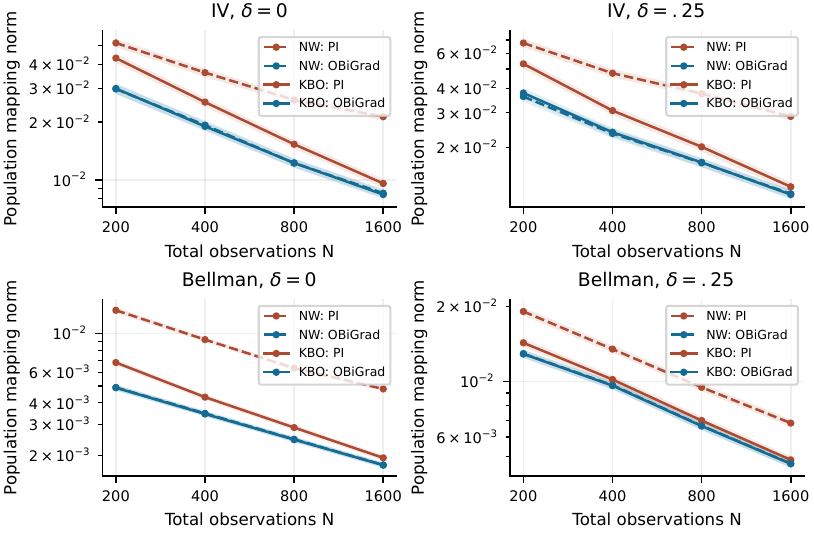}
\caption{Returned population stationarity after $1000$ updates across the complete primary grid. Bands give pointwise Monte Carlo $95\%$ intervals. Residual cases are shown alongside aligned cases for comparison.}\label{fig:optim-residual}\end{figure}
\begin{figure}[H]\centering\includegraphics[width=\linewidth]{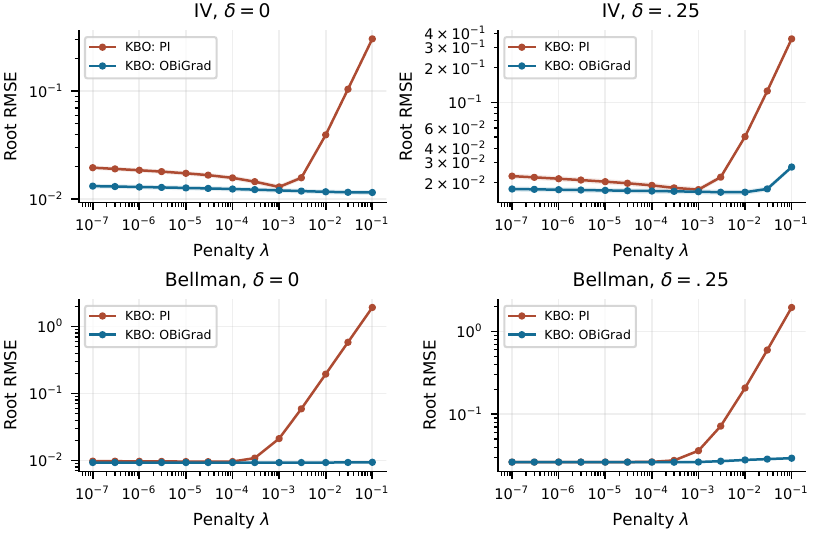}
\caption{Matched KBO root RMSE over the complete penalty grid, at $N=800$. Both methods use identical nuisance fits at each penalty. These curves are secondary diagnostics; the main comparison uses training-only cross-validation.}\label{fig:optim-penalties}\end{figure}
\begin{figure}[H]\centering\includegraphics[width=\linewidth]{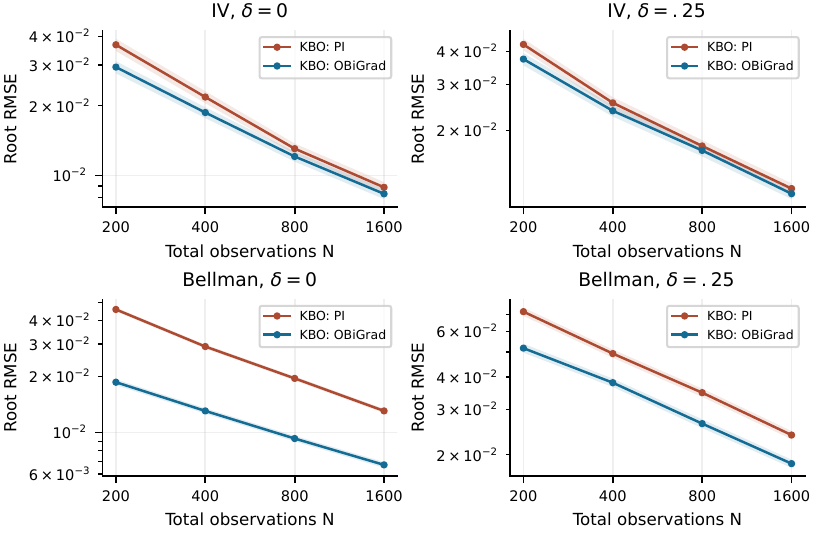}
\caption{Matched KBO root RMSE with decreasing penalty $\lambda_N=0.05N^{-0.6}$. Bands are pointwise Monte Carlo $95\%$ intervals.}\label{fig:optim-decreasing}\end{figure}
\begin{figure}[H]\centering\includegraphics[width=.9\linewidth]{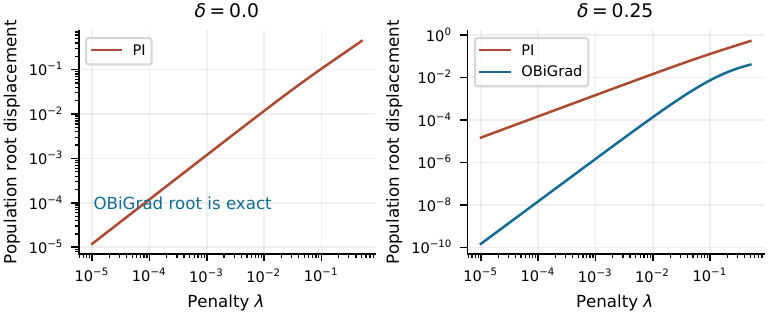}
\caption{Analytic finite-kernel control: population root displacement under regularization. This deterministic calculation illustrates the first-order and product-order bias mechanisms.}\label{fig:optim-population-control}\end{figure}
\begin{figure}[H]\centering\includegraphics[width=\linewidth]{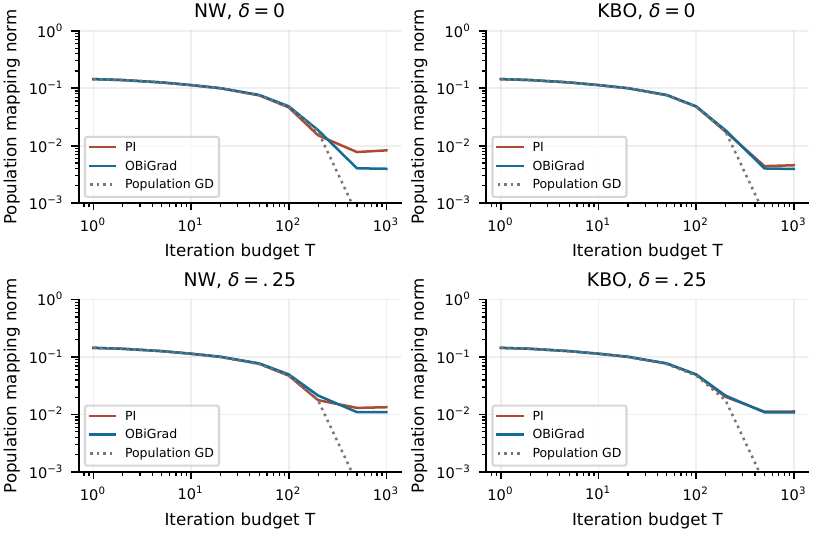}
\caption{Nonlinear Bellman reparameterization: returned population mapping norm versus iteration budget, reusing the main Bellman samples with $N=400$. Population truth is used only for evaluation.}\label{fig:optim-nonlinear}\end{figure}

\begin{table}[H]\centering\scriptsize\setlength{\tabcolsep}{3pt}
\caption{Paired primary contrasts, IV with $\delta=0$. Positive differences favor correction. Intervals adjust over all $64$ root-MSE and stationarity contrasts.}\label{tab:optim-paired0}\begin{tabular}{rrlrrrr}
\toprule
Case & $N$ & Learner & $\Delta$ MSE & Adjusted interval & $\Delta$ mapping & Adjusted interval \\
\midrule
0 & 200 & NW & 0.001626 & [0.001232, 0.002021] & 0.02192 & [0.01706, 0.02678] \\
0 & 200 & KBO & 0.000970 & [0.000662, 0.001279] & 0.01317 & [0.00963, 0.01671] \\
0 & 400 & NW & 0.000865 & [0.000678, 0.001052] & 0.01690 & [0.01343, 0.02038] \\
0 & 400 & KBO & 0.000280 & [0.000197, 0.000363] & 0.00641 & [0.00448, 0.00834] \\
0 & 800 & NW & 0.000476 & [0.000384, 0.000568] & 0.01381 & [0.01137, 0.01626] \\
0 & 800 & KBO & 0.000081 & [0.000053, 0.000109] & 0.00307 & [0.00201, 0.00414] \\
0 & 1600 & NW & 0.000328 & [0.000275, 0.000381] & 0.01280 & [0.01101, 0.01459] \\
0 & 1600 & KBO & 0.000022 & [0.000011, 0.000033] & 0.00120 & [0.00058, 0.00181] \\
\bottomrule
\end{tabular}
\end{table}
\begin{table}[H]\centering\scriptsize\setlength{\tabcolsep}{3pt}
\caption{Paired primary contrasts, IV with $\delta=0.25$; same adjustment as Table~\ref{tab:optim-paired0}.}\label{tab:optim-paired1}\begin{tabular}{rrlrrrr}
\toprule
Case & $N$ & Learner & $\Delta$ MSE & Adjusted interval & $\Delta$ mapping & Adjusted interval \\
\midrule
1 & 200 & NW & 0.003034 & [0.002378, 0.003690] & 0.03137 & [0.02540, 0.03733] \\
1 & 200 & KBO & 0.001359 & [0.000928, 0.001791] & 0.01528 & [0.01109, 0.01947] \\
1 & 400 & NW & 0.001581 & [0.001270, 0.001892] & 0.02401 & [0.01976, 0.02826] \\
1 & 400 & KBO & 0.000375 & [0.000247, 0.000503] & 0.00686 & [0.00456, 0.00915] \\
1 & 800 & NW & 0.000973 & [0.000808, 0.001138] & 0.02073 & [0.01768, 0.02378] \\
1 & 800 & KBO & 0.000132 & [0.000082, 0.000181] & 0.00335 & [0.00204, 0.00466] \\
1 & 1600 & NW & 0.000610 & [0.000515, 0.000705] & 0.01708 & [0.01478, 0.01939] \\
1 & 1600 & KBO & 0.000033 & [0.000017, 0.000050] & 0.00110 & [0.00039, 0.00181] \\
\bottomrule
\end{tabular}
\end{table}
\begin{table}[H]\centering\scriptsize\setlength{\tabcolsep}{3pt}
\caption{Paired primary contrasts, fitted $Q$-regression with $\delta=0$.}\label{tab:optim-paired2}\begin{tabular}{rrlrrrr}
\toprule
Case & $N$ & Learner & $\Delta$ MSE & Adjusted interval & $\Delta$ mapping & Adjusted interval \\
\midrule
2 & 200 & NW & 0.003498 & [0.002987, 0.004009] & 0.00873 & [0.00790, 0.00955] \\
2 & 200 & KBO & 0.000476 & [0.000366, 0.000587] & 0.00193 & [0.00154, 0.00233] \\
2 & 400 & NW & 0.001521 & [0.001346, 0.001696] & 0.00576 & [0.00526, 0.00626] \\
2 & 400 & KBO & 0.000142 & [0.000105, 0.000179] & 0.00086 & [0.00064, 0.00108] \\
2 & 800 & NW & 0.000711 & [0.000626, 0.000797] & 0.00386 & [0.00353, 0.00419] \\
2 & 800 & KBO & 0.000049 & [0.000034, 0.000064] & 0.00042 & [0.00029, 0.00055] \\
2 & 1600 & NW & 0.000399 & [0.000364, 0.000434] & 0.00305 & [0.00284, 0.00326] \\
2 & 1600 & KBO & 0.000014 & [0.000009, 0.000019] & 0.00018 & [0.00011, 0.00025] \\
\bottomrule
\end{tabular}
\end{table}
\begin{table}[H]\centering\scriptsize\setlength{\tabcolsep}{3pt}
\caption{Paired primary contrasts, fitted $Q$-regression with $\delta=0.25$.}\label{tab:optim-paired3}\begin{tabular}{rrlrrrr}
\toprule
Case & $N$ & Learner & $\Delta$ MSE & Adjusted interval & $\Delta$ mapping & Adjusted interval \\
\midrule
3 & 200 & NW & 0.004507 & [0.003655, 0.005360] & 0.00616 & [0.00509, 0.00723] \\
3 & 200 & KBO & 0.000871 & [0.000588, 0.001154] & 0.00140 & [0.00095, 0.00185] \\
3 & 400 & NW & 0.002035 & [0.001661, 0.002410] & 0.00382 & [0.00311, 0.00453] \\
3 & 400 & KBO & 0.000255 & [0.000145, 0.000365] & 0.00056 & [0.00028, 0.00083] \\
3 & 800 & NW & 0.001033 & [0.000838, 0.001228] & 0.00283 & [0.00233, 0.00333] \\
3 & 800 & KBO & 0.000113 & [0.000066, 0.000160] & 0.00034 & [0.00019, 0.00050] \\
3 & 1600 & NW & 0.000549 & [0.000459, 0.000640] & 0.00212 & [0.00177, 0.00247] \\
3 & 1600 & KBO & 0.000038 & [0.000020, 0.000057] & 0.00017 & [0.00008, 0.00027] \\
\bottomrule
\end{tabular}
\end{table}

\subsection{Instrumental-variable experiments}
\label{app:iv-experiments}

These supplementary Gaussian designs (including Gaussian noise and, in some fitted-$Q$ cases, unbounded features) fall outside the bounded-data conditions of \Cref{assum:eif-integrability}. They are robustness illustrations rather than literal instantiations of the finite-sample theory.

\paragraph{IV DGP.}
The IV experiments use a two-stage conditional-projection design, following the IV example of Section~\ref{sec:problem}. Here $X$ is the instrument, $Z$ is the treatment, and the lower-level problem projects $g_\omega(Z)$ onto functions of $X$. We draw $X\sim\mathcal{N}(0,I_p)$ with $p=3$, set
\begin{equation*}
    Z=2\sum_{j=1}^p X_j+\eta,\qquad\eta\sim\mathcal{N}(0,\sigma_Z^2),\qquad\sigma_Z^2=0.1,
\end{equation*}
and define $\phi_\ell(z)=\sin(z+\ell)$, $\ell=1,\ldots,d$, with $d=4$ and $g_\omega(Z)=\omega^\top\phi(Z)$. The structural parameter is $\omega^\star=(1,2,3,4)^\top/\sqrt{30}$. For the gradient estimation and inference experiments, $Y=\omega^{\star\top}\phi(Z)+\varepsilon_Y$ with $\varepsilon_Y\sim\mathcal{N}(0,0.25^2)$, which isolates gradient estimation and calibration. For the KBO regularization experiment, we use $Y=\omega^{\star\top}\phi(Z)+0.5\eta$, which preserves $\mathbb{E}_P[Y\mid X]$ and hence $\nabla\mathcal F(\omega)$, but introduces correlation between the outcome noise and $Z$. We evaluate the population gradient at the fixed non-stationary point
\begin{equation*}
    \omega_0=\omega^\star+0.35\frac{(1,1/3,-1/3,-1)^\top}{\|(1,1/3,-1/3,-1)\|},
\end{equation*}
chosen away from $\omega^\star$ to measure gradient-estimation error at a non-stationary point.

The following analytic expressions are used only to compute the ground-truth population target $\nabla\mathcal F(\omega)$ for evaluation; they are not used by the estimator:
\begin{equation*}
    j^\star_\ell(X)=\mathbb{E}_P[\phi_\ell(Z)\mid X]=\exp(-\sigma_Z^2/2)\sin\!\left(2\sum_{j=1}^p X_j+\ell\right),
\end{equation*}
with $h^\star_\omega(X)=j^\star(X)^\top\omega$ and $m^\star_P(X)=j^\star(X)^\top\omega^\star$. Hence
\begin{equation*}
    \nabla\mathcal F(\omega)=A(\omega-\omega^\star),\qquad A_{k\ell}=\frac{\exp(-\sigma_Z^2)}{2}\left\{\cos(k-\ell)-\exp(-8p)\cos(k+\ell)\right\},
\end{equation*}
where the term $\exp(-8p)$ is negligible for $p=3$. Feasible IV nuisance learners use ridge regression on Fourier features of $\sum_j X_j$, with frequencies $1,\ldots,8$, intercept included, and ridge parameter $10^{-6}$. Two-fold cross-fitting is used for the fixed-gradient, inference, and root experiments. The KBO experiment reports both matched two-fold estimates and the original one-direction split.

\paragraph{IV Experiment 1: fixed-gradient estimation.}
This experiment estimates $\nabla\mathcal F(\omega_0)$ for $N\in\{200,400,800,1600,3200\}$ over $600$ replications, comparing the plug-in hypergradient, OBiGrad, and the oracle DR benchmark. The nuisance product proxy
\begin{equation*}
    \|\hat j_\omega-j^\star_{\omega,P}\|_{\LL^d}\left(\|\hat h_\omega-h^\star_{\omega,P}\|_{\LL}+\|\hat m-m^\star_P\|_{\LL}\right)
\end{equation*}
is reported as the product-rate diagnostic corresponding to the exact mixed-bias bound in \eqref{eq:mixed-bias}. Tables~\ref{tab:generated-iv-gradient}--\ref{tab:generated-iv-nuisance} and Figure~\ref{fig:app-iv-gradient-rmse} show that OBiGrad improves over the plug-in hypergradient at small $N$ and approaches the oracle DR benchmark as the nuisance product decreases with $N$.

\begin{table}[H]
\centering
\small
\setlength{\tabcolsep}{5pt}
\caption{Fixed-$\omega_0$ IV gradient estimation. Parentheses report Monte Carlo $95\%$ error bars for RMSE.}
\label{tab:generated-iv-gradient}
\begin{tabular}{lrrrrr}
\toprule
$N$ & PI & OBiGrad & Oracle & OBi coverage & Product \\
\midrule
200 & 0.0449 (0.0027) & 0.0399 (0.0027) & 0.0364 (0.0016) & 0.947 (0.012) & 0.0517 \\
400 & 0.0257 (0.0010) & 0.0253 (0.0010) & 0.0245 (0.0010) & 0.959 (0.011) & 0.0217 \\
800 & 0.0194 (0.0008) & 0.0192 (0.0008) & 0.0191 (0.0008) & 0.956 (0.010) & 0.0103 \\
1600 & 0.0137 (0.0005) & 0.0136 (0.0005) & 0.0135 (0.0005) & 0.949 (0.011) & 0.0050 \\
3200 & 0.0098 (0.0004) & 0.0098 (0.0004) & 0.0098 (0.0004) & 0.941 (0.013) & 0.0025 \\
\bottomrule
\end{tabular}

\end{table}

\begin{table}[H]
\centering
\small
\setlength{\tabcolsep}{5pt}
\caption{Nuisance-learning diagnostics for the fixed-$\omega_0$ IV experiment.}
\label{tab:generated-iv-nuisance}
\begin{tabular}{lrrrr}
\toprule
$N$ & $\|\hat h-h^\star\|$ & $\|\hat j-j^\star\|$ & $\|\hat m-m^\star\|$ & Product \\
\midrule
200 & 0.0884 & 0.2053 & 0.1573 & 0.0517 \\
400 & 0.0579 & 0.1329 & 0.1039 & 0.0217 \\
800 & 0.0398 & 0.0917 & 0.0719 & 0.0103 \\
1600 & 0.0276 & 0.0643 & 0.0495 & 0.0050 \\
3200 & 0.0197 & 0.0454 & 0.0354 & 0.0025 \\
\bottomrule
\end{tabular}

\end{table}

\begin{figure}[H]
    \centering
    \includegraphics[width=.58\linewidth]{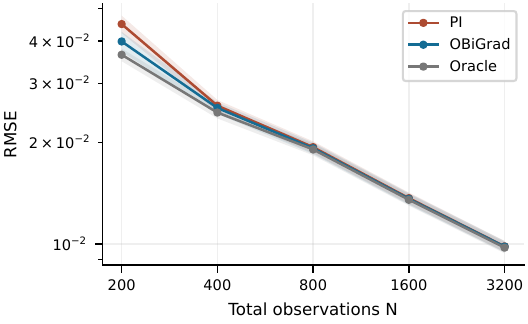}
    \caption{IV fixed-gradient estimation. OBiGrad tracks the oracle DR benchmark and improves over PI at smaller sample sizes.}
    \label{fig:app-iv-gradient-rmse}
\end{figure}

\paragraph{IV Experiment 2: Wald calibration.}
This experiment uses the same vector sine-IV design and forms coordinate-wise $95\%$ Wald confidence intervals \eqref{eq:wald-ci} using the empirical covariance estimator. Results are averaged over $1000$ replications. Tables~\ref{tab:generated-iv-wald}--\ref{tab:generated-iv-studentized} and Figures~\ref{fig:app-iv-wald-curves}--\ref{fig:app-iv-studentized} report coverage, interval length, RMSE, oracle coverage (using the true nuisances $\eta^\star_\omega$), and studentized QQ diagnostics. OBiGrad coverage is approximately nominal across sample sizes, and the QQ diagnostic at $N=3200$ is consistent with a Gaussian approximation. The implemented PI diagnostic intervals are longer and generally conservative, rather than systematically shorter and under-covering. These are marginal intervals for gradient coordinates at a fixed parameter, not confidence regions for a fitted stationary point. The four sine features span only two functions, so no nonsingular four-dimensional joint-inference claim is made.

\begin{table}[H]
\centering
\small
\setlength{\tabcolsep}{5pt}
\caption{OBiGrad Wald calibration on the IV design.}
\label{tab:generated-iv-wald}
\begin{tabular}{lrrrrr}
\toprule
$N$ & Coverage & Length & Coord. RMSE & Product & Oracle cov. \\
\midrule
200 & 0.942 (0.010) & 0.0763 (0.0003) & 0.0203 (0.0009) & 0.0521 & 0.950 (0.009) \\
400 & 0.953 (0.009) & 0.0531 (0.0001) & 0.0132 (0.0004) & 0.0218 & 0.959 (0.008) \\
800 & 0.954 (0.008) & 0.0373 (4.2e-05) & 0.0094 (0.0003) & 0.0103 & 0.953 (0.008) \\
1600 & 0.952 (0.008) & 0.0263 (2.1e-05) & 0.0067 (0.0002) & 0.0050 & 0.952 (0.008) \\
3200 & 0.948 (0.009) & 0.0186 (1e-05) & 0.0048 (0.0002) & 0.0025 & 0.948 (0.010) \\
\bottomrule
\end{tabular}

\end{table}

\begin{table}[H]
\centering
\small
\setlength{\tabcolsep}{5pt}
\caption{Studentized-error diagnostics for the vector sine-IV inference experiment.}
\label{tab:generated-iv-studentized}
\begin{tabular}{lrrrrrr}
\toprule
$N$ & Mean & SD & 2.5\% & Median & 97.5\% & PI coverage \\
\midrule
200 & 0.015 & 1.028 & -1.983 & 0.011 & 2.045 & 0.940 (0.010) \\
400 & 0.005 & 0.978 & -1.921 & 0.001 & 1.958 & 0.966 (0.007) \\
800 & 0.005 & 0.989 & -1.880 & -0.018 & 1.980 & 0.965 (0.007) \\
1600 & 0.027 & 0.999 & -1.940 & 0.032 & 1.945 & 0.967 (0.007) \\
3200 & -0.022 & 1.011 & -2.015 & -0.031 & 1.934 & 0.963 (0.008) \\
\bottomrule
\end{tabular}

\end{table}

\begin{figure}[H]
    \centering
    \includegraphics[width=.48\linewidth]{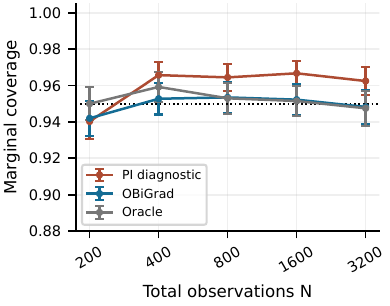}
    \hfill
    \includegraphics[width=.48\linewidth]{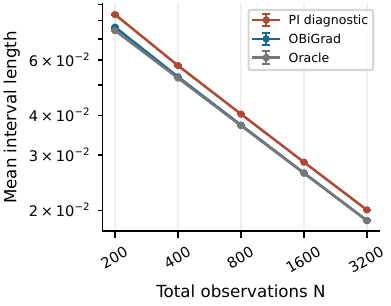}
    \caption{IV Wald diagnostics. Left: coordinate-wise coverage of nominal $95\%$ intervals. Right: average interval length.}
    \label{fig:app-iv-wald-curves}
\end{figure}

\begin{figure}[H]
    \centering
    \includegraphics[width=.48\linewidth]{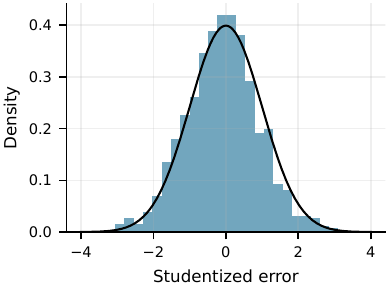}
    \hfill
    \includegraphics[width=.48\linewidth]{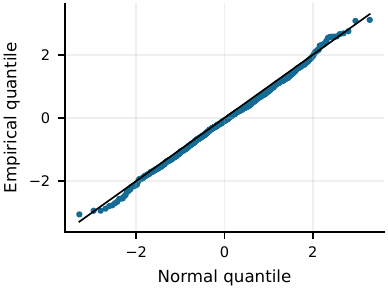}
    \caption{IV studentized OBiGrad errors for coordinate $0$ at $N=3200$.}
    \label{fig:app-iv-studentized}
\end{figure}

\paragraph{Matched kernel Wald control.}
We also evaluated KBO and KBO+OBiGrad using the same IV inference samples and folds, the original Gaussian kernel on $X\in\mathbb R^3$ with bandwidth $0.5$, and either $\lambda=0.01$ or $\lambda_N=0.05N^{-0.6}$. Both retain substantial calibration error in this design, so we do not use them as evidence of valid inference. Fixed regularization need not satisfy the nuisance-accuracy conditions of Theorem~\ref{thm:CLT}, and the decreasing schedule does not establish these conditions at the displayed sample sizes. The complete coverage, length, and score diagnostics are retained in the supplementary result files; the Bellman extension below reports the same diagnostic protocol.

\paragraph{IV Experiment 3: KBO regularization bias.}
This experiment uses $600$ training and $600$ evaluation observations per direction ($N=1200$ total), $300$ replications, and the original Gaussian kernel $\exp(-\|x-x'\|^2/(2\sigma^2))$ with $\sigma=0.5$. We retain the nine-penalty grid $\{10^{-5},3\cdot10^{-5},10^{-4},3\cdot10^{-4},10^{-3},3\cdot10^{-3},10^{-2},3\cdot10^{-2},10^{-1}\}$. Tables~\ref{tab:generated-iv-kbo-gradient}--\ref{tab:generated-iv-kbo-gradient-appendix} compare the existing basis pair and the matched kernel pair using two-fold averages. Quadrature gives a distinct population target for each gradient construction. Small penalties can increase finite-sample error even as population bias decreases: at $\lambda=10^{-5}$, KBO RMSE is $0.0768$, versus $0.0846$ for KBO+OBiGrad. Correction is beneficial at larger penalties, but does not uniformly improve this design. Figure~\ref{fig:app-iv-kbo} displays total error and population bias. The original one-direction protocol is retained in Table~\ref{tab:iv-native-kbo}; within each protocol all methods use the same evaluation observations.

\begin{table}[H]
\centering
\scriptsize
\setlength{\tabcolsep}{3pt}
\caption{Matched two-fold IV gradient RMSE at $N=1200$. Parentheses report Monte Carlo $95\%$ half-widths.}
\label{tab:generated-iv-kbo-gradient}
\begin{tabular}{lrrrrr}
\toprule
$\lambda$ & PI & OBiGrad & KBO & KBO+OBiGrad & Oracle \\
\midrule
1e-05 & 0.0137 (0.0008) & 0.0134 (0.0008) & 0.0768 (0.0015) & 0.0846 (0.0019) & 0.0133 (0.0008) \\
3e-05 & 0.0137 (0.0008) & 0.0134 (0.0008) & 0.0530 (0.0014) & 0.0621 (0.0015) & 0.0133 (0.0008) \\
0.0001 & 0.0137 (0.0008) & 0.0134 (0.0008) & 0.0478 (0.0013) & 0.0508 (0.0014) & 0.0133 (0.0008) \\
0.0003 & 0.0137 (0.0008) & 0.0134 (0.0008) & 0.0608 (0.0013) & 0.0488 (0.0013) & 0.0133 (0.0008) \\
0.001 & 0.0137 (0.0008) & 0.0134 (0.0008) & 0.0993 (0.0011) & 0.0586 (0.0012) & 0.0133 (0.0008) \\
0.003 & 0.0137 (0.0008) & 0.0134 (0.0008) & 0.1589 (0.0010) & 0.0895 (0.0012) & 0.0133 (0.0008) \\
0.01 & 0.0137 (0.0008) & 0.0134 (0.0008) & 0.2379 (0.0009) & 0.1617 (0.0010) & 0.0133 (0.0008) \\
0.03 & 0.0137 (0.0008) & 0.0134 (0.0008) & 0.3024 (0.0006) & 0.2492 (0.0008) & 0.0133 (0.0008) \\
0.1 & 0.0137 (0.0008) & 0.0134 (0.0008) & 0.3500 (0.0004) & 0.3256 (0.0004) & 0.0133 (0.0008) \\
\bottomrule
\end{tabular}

\end{table}

\begin{table}[H]
\centering
\small
\setlength{\tabcolsep}{5pt}
\caption{IV kernel population-bias norms and estimation RMSE around each method's own population target.}
\label{tab:generated-iv-kbo-gradient-appendix}
\begin{tabular}{lrrrr}
\toprule
$\lambda$ & KBO bias & KBO estimation & OBi bias & OBi estimation \\
\midrule
1e-05 & 0.0034 & 0.0778 (0.0015) & 0.0005 & 0.0842 (0.0019) \\
3e-05 & 0.0080 & 0.0530 (0.0013) & 0.0013 & 0.0608 (0.0015) \\
0.0001 & 0.0196 & 0.0406 (0.0012) & 0.0039 & 0.0471 (0.0014) \\
0.0003 & 0.0418 & 0.0339 (0.0012) & 0.0103 & 0.0391 (0.0013) \\
0.001 & 0.0875 & 0.0271 (0.0011) & 0.0288 & 0.0314 (0.0012) \\
0.003 & 0.1517 & 0.0209 (0.0009) & 0.0682 & 0.0242 (0.0011) \\
0.01 & 0.2346 & 0.0142 (0.0007) & 0.1497 & 0.0161 (0.0009) \\
0.03 & 0.3012 & 0.0087 (0.0005) & 0.2444 & 0.0096 (0.0006) \\
0.1 & 0.3498 & 0.0043 (0.0003) & 0.3248 & 0.0046 (0.0003) \\
\bottomrule
\end{tabular}

\end{table}

\begin{figure}[H]
    \centering
\includegraphics[width=\linewidth]{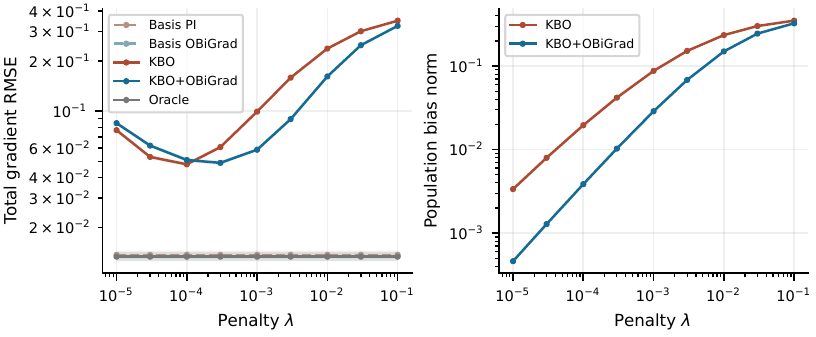}
    \caption{IV kernel diagnostics: matched total gradient RMSE (left) and population-bias norms (right).}
    \label{fig:app-iv-kbo}
\end{figure}

\begin{table}[H]\centering\scriptsize\setlength{\tabcolsep}{3pt}
\caption{IV: original one-direction evaluation, with matched observations for all five estimators. Parentheses give Monte Carlo $95\%$ half-widths.}
\label{tab:iv-native-kbo}\begin{tabular}{lrrrrr}
\toprule
$\lambda$ & PI & OBiGrad & KBO & KBO+OBiGrad & Oracle \\
\midrule
1e-05 & 0.0257 (0.0018) & 0.0191 (0.0011) & 0.0801 (0.0023) & 0.0862 (0.0025) & 0.0191 (0.0011) \\
3e-05 & 0.0257 (0.0018) & 0.0191 (0.0011) & 0.0567 (0.0022) & 0.0634 (0.0020) & 0.0191 (0.0011) \\
0.0001 & 0.0257 (0.0018) & 0.0191 (0.0011) & 0.0510 (0.0024) & 0.0521 (0.0018) & 0.0191 (0.0011) \\
0.0003 & 0.0257 (0.0018) & 0.0191 (0.0011) & 0.0624 (0.0024) & 0.0499 (0.0017) & 0.0191 (0.0011) \\
0.001 & 0.0257 (0.0018) & 0.0191 (0.0011) & 0.0995 (0.0021) & 0.0591 (0.0016) & 0.0191 (0.0011) \\
0.003 & 0.0257 (0.0018) & 0.0191 (0.0011) & 0.1584 (0.0018) & 0.0894 (0.0014) & 0.0191 (0.0011) \\
0.01 & 0.0257 (0.0018) & 0.0191 (0.0011) & 0.2373 (0.0013) & 0.1615 (0.0011) & 0.0191 (0.0011) \\
0.03 & 0.0257 (0.0018) & 0.0191 (0.0011) & 0.3021 (0.0009) & 0.2491 (0.0008) & 0.0191 (0.0011) \\
0.1 & 0.0257 (0.0018) & 0.0191 (0.0011) & 0.3498 (0.0004) & 0.3256 (0.0004) & 0.0191 (0.0011) \\
\bottomrule
\end{tabular}
\end{table}

\paragraph{IV Experiment 4: root estimation.}
The root experiment uses a scalar IV design. We draw $X\sim\mathcal{N}(0,I_p)$ with $p=3$, set $Z=2\sum_{j=1}^p X_j+\eta$ with $\eta\sim\mathcal{N}(0,0.1)$, and generate
\begin{equation*}
    Y=\omega^\star Z+0.5\eta+\varepsilon_Y,\qquad\varepsilon_Y\sim\mathcal{N}(0,0.1^2),\qquad\omega^\star=2.
\end{equation*}
Here $j^\star(X)=\mathbb{E}_P[Z\mid X]=2\sum_j X_j$, $m^\star_P(X)=\mathbb{E}_P[Y\mid X]=\omega^\star j^\star(X)$, and $\nabla\mathcal F(\omega)=4p(\omega-\omega^\star)$, so the population stationary point is exactly $\omega^\star=2$. We compare PI, OBiGrad, the oracle DR benchmark, and matched KBO/KBO+OBiGrad with either $\lambda=10^{-2}$ or $\lambda_N=0.05N^{-0.6}$. Every method now uses the same total $N$ observations and two-fold evaluation, with $300$ replications per $N$; the KBO bandwidth is $2$. Population roots use converged quadrature. At $\lambda=0.01$, the KBO population root is $2.23507$, while the matched corrected population root is $2$ because the outcome and inner conditional means align at the true root. Tables~\ref{tab:generated-iv-root-rmse}--\ref{tab:generated-iv-root-kbo-population} and Figures~\ref{fig:app-iv-root}--\ref{fig:app-iv-root-kbo} show that OBiGrad tracks the oracle DR benchmark, while fixed-$\lambda$ KBO remains biased toward its regularized stationary point.

\begin{table}[H]
\centering
\scriptsize
\setlength{\tabcolsep}{3pt}
\caption{Root-estimation RMSE. Parentheses report Monte Carlo $95\%$ error bars.}
\label{tab:generated-iv-root-rmse}
\begin{tabular}{lrrrrrrr}
\toprule
$N$ & PI & OBiGrad & Oracle & KBO fixed & OBi fixed & KBO decay & OBi decay \\
\midrule
100 & 0.0107 (0.0012) & 0.0056 (0.0005) & 0.0056 (0.0005) & 0.3146 (0.0065) & 0.0059 (0.0005) & 0.1903 (0.0064) & 0.0060 (0.0006) \\
200 & 0.0056 (0.0004) & 0.0038 (0.0003) & 0.0038 (0.0003) & 0.2730 (0.0038) & 0.0039 (0.0003) & 0.1265 (0.0035) & 0.0040 (0.0003) \\
400 & 0.0035 (0.0003) & 0.0029 (0.0002) & 0.0029 (0.0002) & 0.2572 (0.0023) & 0.0030 (0.0002) & 0.0891 (0.0019) & 0.0030 (0.0002) \\
800 & 0.0023 (0.0002) & 0.0019 (0.0001) & 0.0019 (0.0001) & 0.2452 (0.0013) & 0.0020 (0.0001) & 0.0599 (0.0009) & 0.0020 (0.0001) \\
1600 & 0.0015 (0.0001) & 0.0014 (0.0001) & 0.0014 (0.0001) & 0.2401 (0.0008) & 0.0014 (0.0001) & 0.0407 (0.0006) & 0.0014 (0.0001) \\
\bottomrule
\end{tabular}

\end{table}

\begin{table}[H]
\centering
\small
\setlength{\tabcolsep}{5pt}
\caption{KBO population roots in the scalar root experiment.}
\label{tab:generated-iv-root-kbo-population}
\begin{tabular}{lrrrr}
\toprule
$\lambda$ & KBO root & KBO bias & OBi root & OBi bias \\
\midrule
0.000598 & 2.03267 & 0.03267 & 2.00000 & 0.00000 \\
0.000906 & 2.04458 & 0.04458 & 2.00000 & 0.00000 \\
0.001373 & 2.05964 & 0.05964 & 2.00000 & 0.00000 \\
0.002081 & 2.07836 & 0.07836 & 2.00000 & 0.00000 \\
0.003155 & 2.10220 & 0.10220 & 2.00000 & 0.00000 \\
0.010000 & 2.23507 & 0.23507 & 2.00000 & 0.00000 \\
\bottomrule
\end{tabular}

\end{table}

\begin{figure}[H]
    \centering
    \includegraphics[width=.8\linewidth]{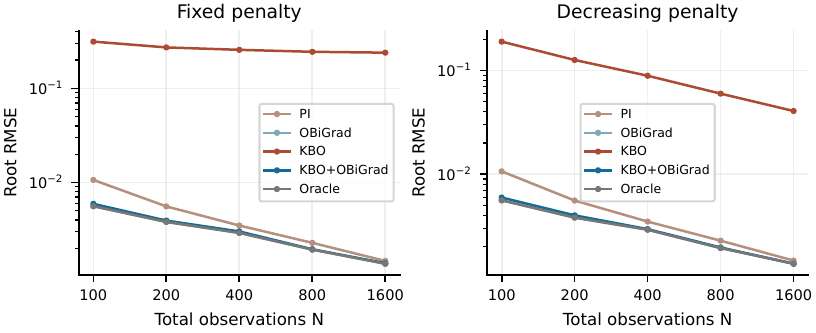}
    \hfill
    \caption{IV root estimation RMSE.}
    \label{fig:app-iv-root}
\end{figure}

\begin{figure}[H]
    \centering
    \includegraphics[width=.8\linewidth]{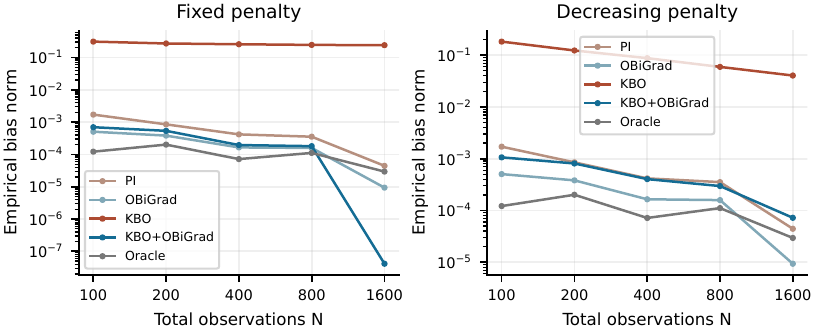}
    \caption{IV empirical root-bias norms with matched total sample sizes, under fixed and decreasing penalties.}
    \label{fig:app-iv-root-kbo}
\end{figure}

\subsection{\texorpdfstring{Fitted $Q$-regression experiments}{Fitted Q-regression experiments}}
\label{app:fitted-q-experiments}

These supplementary Gaussian designs (including Gaussian noise and, in some fitted-$Q$ cases, unbounded features) fall outside the bounded-data conditions of \Cref{assum:eif-integrability}. They are robustness illustrations rather than literal instantiations of the finite-sample theory.

\paragraph{Fitted $Q$-regression DGP.}
The fitted $Q$-regression experiments use a projected Bellman-backup design. Let $X=(S,A)$ and $Z=(S,A,R,S')$. The state is $S\sim\mathcal{N}(0,1)$, the action is binary $A\sim\mathrm{Bernoulli}(0.5)$, and the next state follows
\begin{equation*}
    S'=\rho S+\tau A+\xi,\qquad\xi\sim\mathcal{N}(0,\sigma_S^2),
\end{equation*}
with $\rho=0.7$, $\tau=0.5$, and $\sigma_S=0.2$. The reward is
\begin{equation*}
    R=\sin(S)+0.5A+0.25SA+\varepsilon_R,\qquad\varepsilon_R\sim\mathcal{N}(0,0.1^2).
\end{equation*}
For a continuation-value model $V_\omega(s')=\omega^\top\phi(s')$, the Bellman target is $g_\omega(Z)=R+\gamma V_\omega(S')$, and the lower-level loss is the quadratic Bellman-regression loss
\begin{equation*}
    \frac{1}{2}\mathbb{E}_P\bigl[(h(S,A)-R-\gamma V_\omega(S'))^2\bigr],
\end{equation*}
whose population solution is the projected Bellman backup
\begin{equation*}
    h^\star_{\omega,P}(S,A)=\mathbb{E}_P[R+\gamma V_\omega(S')\mid S,A],
\end{equation*}
as in the FQE example of Section~\ref{sec:problem}.

For the gradient-estimation and KBO experiments, we use
\begin{equation*}
    \phi(s')=(\sin s',\cos s',s',(s')^2)^\top,\qquad\omega^\star=(0.55,-0.35,0.25,0.15)^\top,
\end{equation*}
$\gamma=0.8$, and
\begin{equation*}
    Y=R+\gamma V_{\omega^\star}(S')+\varepsilon_Y,\qquad\varepsilon_Y\sim\mathcal{N}(0,0.25^2).
\end{equation*}
The evaluation point is
\begin{equation*}
    \omega_0=\omega^\star+0.35\frac{(1,-0.5,0.35,-0.25)^\top}{\|(1,-0.5,0.35,-0.25)\|}.
\end{equation*}
For the Wald experiment, we use a state-dependent behavior policy $P(A=1\mid S)=\sigma(0.5S)$ and $\gamma=0.9$, set $\omega^\star=(1,-0.6,0.35,0.15)^\top/\|(1,-0.6,0.35,0.15)\|$, and evaluate at $\omega_0=\omega^\star+0.35(0.6,0.2,-0.5,0.3)^\top/\|(0.6,0.2,-0.5,0.3)\|$. For the root experiment, we use the two-dimensional feature map $\phi(s')=(\sin s',\cos s')^\top$ with $\omega^\star=(0.65,-0.45)^\top$, which gives a well-conditioned population stationary point; the outcome-noise standard deviation is $0.1$ in this root experiment.

The analytic nuisances are available because $S'\mid S,A$ is Gaussian. Writing $\mu(S,A)=\rho S+\tau A$,
\begin{equation*}
    \mathbb{E}_P[\sin(S')\mid S,A]=e^{-\sigma_S^2/2}\sin\{\mu(S,A)\},\qquad\mathbb{E}_P[\cos(S')\mid S,A]=e^{-\sigma_S^2/2}\cos\{\mu(S,A)\},
\end{equation*}
and $\mathbb{E}_P[S'\mid S,A]=\mu(S,A)$, $\mathbb{E}_P[(S')^2\mid S,A]=\mu(S,A)^2+\sigma_S^2$. Hence
\begin{equation*}
    j^\star_{\omega,P}(S,A)=\mathbb{E}_P[\gamma\phi(S')\mid S,A],
\end{equation*}
\begin{equation*}
    h^\star_{\omega,P}(S,A)=\mathbb{E}_P[R\mid S,A]+j^\star_{\omega,P}(S,A)^\top\omega,\qquad m^\star_P(S,A)=\mathbb{E}_P[R\mid S,A]+j^\star_{\omega,P}(S,A)^\top\omega^\star.
\end{equation*}
The default feasible nuisance learner uses a 38-column basis of the raw observed covariates $(S,A)$, without centering or standardization:
\begin{align*}
 b(S,A)=\bigl(&1,S,A,SA,S^2,AS^2,S^3,AS^3,\\
 &\{\sin(fS),\cos(fS),A\sin(fS),A\cos(fS)\}_{f\in\{0.5,0.75,1,1.5,2\}},\\
 &\{e^{-(S-c)^2/(2\cdot0.8^2)},Ae^{-(S-c)^2/(2\cdot0.8^2)}\}_{c\in\{-2,-1,0,1,2\}}\bigr)^\top.
\end{align*}
On each training fold, with design matrix $B$ and matrix $T$ of observed regression targets, the coefficients solve $(B^\top B+\alpha\,\mathrm{diag}(0,1,\ldots,1))\widehat\beta=B^\top T$. The intercept is unpenalized; the loss uses a sum of squares, so $\alpha$ is not multiplied by the fold size. The fixed penalties are $\alpha=0.5$ for fixed-gradient estimation and the basis learner in the regularization comparison, $\alpha=0.05$ for Wald calibration, and $\alpha=1$ for root estimation, with no tuning of these basis-ridge penalties. Oracle conditional expectations and transition parameters are used only for evaluation and for the oracle DR benchmark.

\paragraph{Fitted $Q$-regression Experiment 1: fixed-gradient estimation.}
This experiment estimates $\nabla\mathcal F(\omega_0)$ for $N\in\{200,400,800,1600,3200\}$ over $400$ replications, comparing the plug-in hypergradient, OBiGrad, and the oracle DR benchmark. The nuisance product proxy $\|\hat j_\omega-j^\star_{\omega,P}\|_{\LL^d}(\|\hat h_\omega-h^\star_{\omega,P}\|_{\LL}+\|\hat m-m^\star_P\|_{\LL})$ is the product-rate diagnostic corresponding to the exact mixed-bias bound in \eqref{eq:mixed-bias}. Tables~\ref{tab:generated-bellman-b1-gradient}--\ref{tab:generated-bellman-b1-nuisance} and Figure~\ref{fig:app-bellman-gradient-rmse} show that OBiGrad substantially improves over the plug-in hypergradient at small $N$ and approaches the oracle DR benchmark as $N$ grows.

\begin{table}[H]
\centering
\small
\setlength{\tabcolsep}{5pt}
\caption{Fitted $Q$-regression gradient estimation. Parentheses report Monte Carlo $95\%$ error bars.}
\label{tab:generated-bellman-b1-gradient}
\begin{tabular}{lrrrrr}
\toprule
$N$ & PI & OBiGrad & Oracle & OBi coverage & Product \\
\midrule
200 & 0.0812 (0.0112) & 0.0458 (0.0045) & 0.0321 (0.0015) & 0.941 (0.013) & 0.0832 \\
400 & 0.0477 (0.0111) & 0.0258 (0.0020) & 0.0218 (0.0010) & 0.951 (0.013) & 0.0363 \\
800 & 0.0215 (0.0022) & 0.0167 (0.0007) & 0.0160 (0.0006) & 0.948 (0.013) & 0.0146 \\
1600 & 0.0126 (0.0006) & 0.0115 (0.0005) & 0.0114 (0.0005) & 0.953 (0.012) & 0.0069 \\
3200 & 0.0081 (0.0003) & 0.0079 (0.0003) & 0.0078 (0.0003) & 0.954 (0.012) & 0.0035 \\
\bottomrule
\end{tabular}

\end{table}

\begin{table}[H]
\centering
\small
\setlength{\tabcolsep}{5pt}
\caption{Nuisance-learning diagnostics for fitted $Q$-regression gradient estimation.}
\label{tab:generated-bellman-b1-nuisance}
\begin{tabular}{lrrrr}
\toprule
$N$ & $\|\hat h-h^\star\|$ & $\|\hat j-j^\star\|$ & $\|\hat m-m^\star\|$ & Product \\
\midrule
200 & 0.1265 & 0.2519 & 0.1707 & 0.0832 \\
400 & 0.0746 & 0.1668 & 0.1120 & 0.0363 \\
800 & 0.0461 & 0.1160 & 0.0715 & 0.0146 \\
1600 & 0.0302 & 0.0833 & 0.0494 & 0.0069 \\
3200 & 0.0212 & 0.0620 & 0.0339 & 0.0035 \\
\bottomrule
\end{tabular}

\end{table}

\begin{figure}[H]
    \centering
    \includegraphics[width=.58\linewidth]{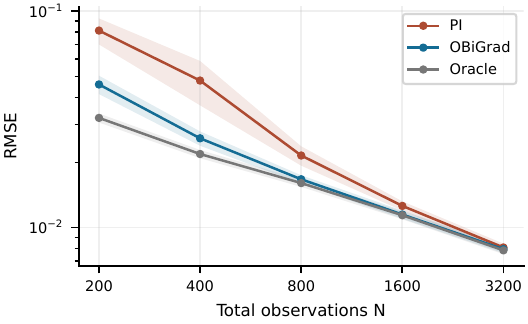}
    \caption{Fitted $Q$-regression fixed-gradient estimation. OBiGrad improves over PI at small sample sizes and approaches the oracle DR benchmark as $N$ grows.}
    \label{fig:app-bellman-gradient-rmse}
\end{figure}

\paragraph{Fitted $Q$-regression Experiment 2: Wald calibration.}
This experiment forms coordinate-wise $95\%$ Wald confidence intervals \eqref{eq:wald-ci} using the empirical covariance estimator. Results are averaged over $400$ replications. Tables~\ref{tab:generated-bellman-b2-wald}--\ref{tab:generated-bellman-b2-studentized} and Figures~\ref{fig:app-bellman-wald-curves}--\ref{fig:app-bellman-studentized} report coverage, interval length, RMSE, studentized quantiles, and QQ diagnostics. OBiGrad intervals are shorter than the PI diagnostic intervals, with broadly near-nominal coverage but visible finite-sample departures. At $N=3200$, pooled coverage is $0.9338$ with Monte Carlo half-width $0.0154$, versus $0.9806$ for PI. The oracle coverage is $0.9381$, so the departures cannot all be assigned to nuisance learning. The QQ plot describes the displayed coordinate; it does not validate inference after adaptive optimization.

\begin{table}[H]
\centering
\small
\setlength{\tabcolsep}{5pt}
\caption{Fitted $Q$-regression Wald calibration.}
\label{tab:generated-bellman-b2-wald}
\begin{tabular}{lrrrrr}
\toprule
$N$ & Coverage & Length & Coord. RMSE & Product & Oracle cov. \\
\midrule
200 & 0.943 (0.013) & 0.1171 (0.0129) & 0.0584 (0.0124) & 0.2105 & 0.945 (0.013) \\
400 & 0.946 (0.013) & 0.0653 (0.0069) & 0.0330 (0.0123) & 0.0885 & 0.950 (0.013) \\
800 & 0.952 (0.012) & 0.0380 (0.0011) & 0.0113 (0.0009) & 0.0360 & 0.947 (0.013) \\
1600 & 0.943 (0.013) & 0.0255 (0.0004) & 0.0076 (0.0004) & 0.0155 & 0.948 (0.012) \\
3200 & 0.934 (0.015) & 0.0175 (0.0001) & 0.0049 (0.0002) & 0.0068 & 0.938 (0.014) \\
\bottomrule
\end{tabular}

\end{table}

\begin{table}[H]
\centering
\small
\setlength{\tabcolsep}{5pt}
\caption{Studentized-error diagnostics for fitted $Q$-regression Wald intervals.}
\label{tab:generated-bellman-b2-studentized}
\begin{tabular}{lrrrrrr}
\toprule
$N$ & Mean & SD & 2.5\% & Median & 97.5\% & PI coverage \\
\midrule
200 & -0.005 & 1.117 & -2.079 & 0.034 & 1.928 & 0.961 (0.012) \\
400 & -0.031 & 1.077 & -2.022 & -0.032 & 1.944 & 0.972 (0.010) \\
800 & 0.011 & 1.016 & -1.977 & 0.055 & 1.926 & 0.979 (0.008) \\
1600 & 0.018 & 1.037 & -1.962 & 0.017 & 2.058 & 0.968 (0.010) \\
3200 & -0.078 & 1.040 & -2.174 & -0.082 & 1.961 & 0.981 (0.007) \\
\bottomrule
\end{tabular}

\end{table}

\begin{figure}[H]
    \centering
    \includegraphics[width=.48\linewidth]{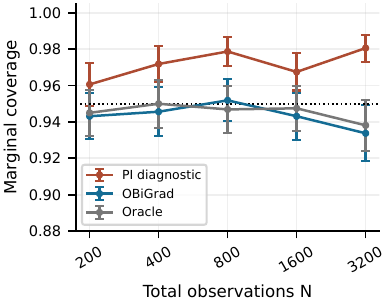}
    \hfill
    \includegraphics[width=.48\linewidth]{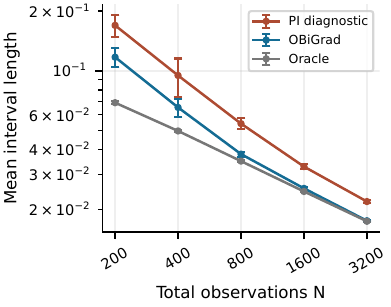}
    \caption{Fitted $Q$-regression Wald diagnostics. Left: coordinate-wise coverage of nominal $95\%$ intervals. Right: average interval length.}
    \label{fig:app-bellman-wald-curves}
\end{figure}

\begin{figure}[H]
    \centering
    \includegraphics[width=.48\linewidth]{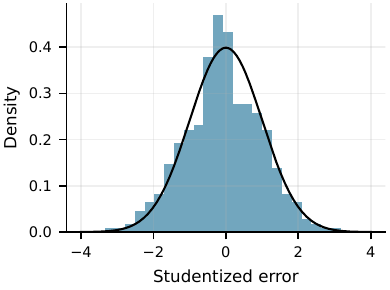}
    \hfill
    \includegraphics[width=.48\linewidth]{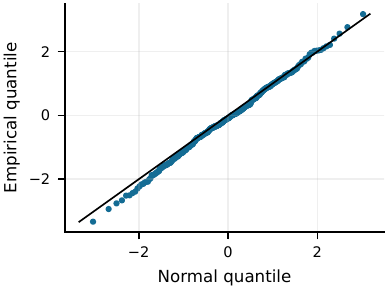}
    \caption{Fitted $Q$-regression studentized OBiGrad errors for coordinate $0$ at $N=3200$.}
    \label{fig:app-bellman-studentized}
\end{figure}

\paragraph{Matched kernel Wald extension.}
We add KBO and KBO+OBiGrad on exactly the same datasets, folds, evaluation parameter, and population target as the preceding Wald experiment. The two seed suites give $400$ replications per $N$. The kernel specification is the Bellman KBO map with $256$ RFFs and kernel $\exp(-0.35\|x-x'\|^2)$; each seed suite retains its own fixed map. We use fixed $\lambda=0.01$ and decreasing $\lambda_N=0.05N^{-0.6}$, without selecting penalties using coverage or population truth. Each coordinate interval is the score mean plus or minus $1.96$ times its empirical score standard deviation divided by $\sqrt N$. KBO uses its plug-in scores and KBO+OBiGrad its corrected scores. The KBO intervals are diagnostic: this score covariance alone is not a general variance formula accounting for learned plug-in nuisances. Orthogonal intervals also require sufficiently accurate nuisance fits for the asymptotic result to apply; a fixed penalty can leave bias.

Table~\ref{tab:b2-kernel-wald} and Figure~\ref{fig:b2-kernel-wald} show improved coverage after correction, especially with decreasing regularization. At $N=3200$, decreasing-penalty coverage is $0.8613$ for KBO and $0.9319$ for KBO+OBiGrad. With fixed $\lambda=0.01$, coverage is $0.0025$ and $0.8413$, respectively. The residual fixed-penalty undercoverage is consistent with a remaining product bias; correction does not eliminate that bias entirely.

\begin{table}[H]\centering\small\setlength{\tabcolsep}{3pt}
\caption{Fitted $Q$-regression kernel Wald extension: marginal coverage and mean interval length. Parentheses give $95\%$ Monte Carlo half-widths across complete datasets. ``OBi'' denotes KBO+OBiGrad.}\label{tab:b2-kernel-wald}\begin{tabular}{rlrrrr}
\toprule
$N$ & Penalty & KBO cov. & OBi cov. & KBO length & OBi length \\
\midrule
200 & Fixed & 0.293 (0.015) & 0.907 (0.016) & 0.0792 (0.0007) & 0.0601 (0.0007) \\
400 & Fixed & 0.195 (0.011) & 0.913 (0.016) & 0.0558 (0.0003) & 0.0431 (0.0003) \\
800 & Fixed & 0.128 (0.012) & 0.908 (0.015) & 0.0389 (0.0002) & 0.0306 (0.0002) \\
1600 & Fixed & 0.037 (0.009) & 0.879 (0.017) & 0.0276 (0.0001) & 0.0216 (0.0001) \\
3200 & Fixed & 0.003 (0.002) & 0.841 (0.018) & 0.0194 (4.7e-05) & 0.0153 (4.3e-05) \\
200 & Decreasing & 0.851 (0.020) & 0.924 (0.014) & 0.0788 (0.0008) & 0.0645 (0.0008) \\
400 & Decreasing & 0.863 (0.019) & 0.937 (0.015) & 0.0556 (0.0004) & 0.0469 (0.0004) \\
800 & Decreasing & 0.873 (0.018) & 0.941 (0.013) & 0.0389 (0.0002) & 0.0337 (0.0002) \\
1600 & Decreasing & 0.873 (0.017) & 0.941 (0.013) & 0.0277 (0.0001) & 0.0240 (0.0001) \\
3200 & Decreasing & 0.861 (0.018) & 0.932 (0.015) & 0.0196 (0.0001) & 0.0170 (0.0001) \\
\bottomrule
\end{tabular}
\end{table}
\begin{figure}[H]\centering\includegraphics[width=\linewidth]{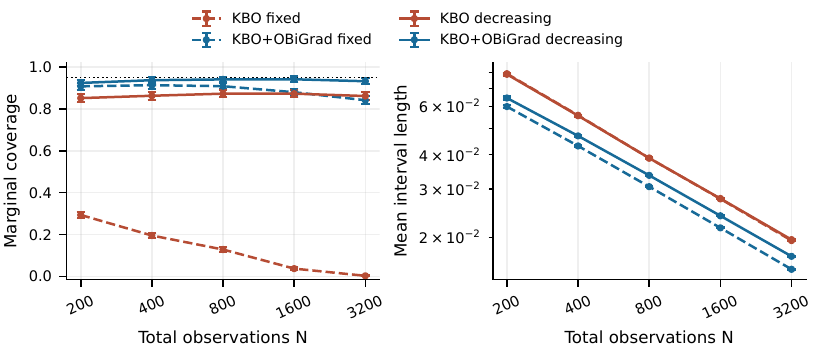}
\caption{Fitted $Q$-regression kernel Wald diagnostics: coverage (left) and interval length (right), for fixed and decreasing regularization. Both methods use identical kernel nuisance fits.}\label{fig:b2-kernel-wald}\end{figure}

\paragraph{Fitted $Q$-regression Experiment 3: root estimation.}
The root experiment uses the two-feature continuation model $V_\omega(S')=\omega^\top(\sin S',\cos S')^\top$ with $\omega^\star=(0.65,-0.45)^\top$. The outcome is $Y=R+\gamma V_{\omega^\star}(S')+\varepsilon_Y$, so the unregularized population stationary point is exactly $\omega^\star$. We compare the plug-in hypergradient, OBiGrad, the oracle DR benchmark, and the oracle plug-in root. The root equations are affine in $\omega$, so they are solved directly rather than by gradient descent. Tables~\ref{tab:generated-bellman-b3-root-rmse}--\ref{tab:generated-bellman-b3-root-appendix} and Figures~\ref{fig:app-bellman-root-rmse}--\ref{fig:app-bellman-root-bias} show that OBiGrad tracks the oracle DR root and improves over the plug-in hypergradient for moderate and large sample sizes.

\begin{table}[H]
\centering
\small
\setlength{\tabcolsep}{5pt}
\caption{Fitted $Q$-regression root-estimation RMSE. Parentheses report Monte Carlo $95\%$ error bars.}
\label{tab:generated-bellman-b3-root-rmse}
\begin{tabular}{lrrrr}
\toprule
$N$ & PI & OBiGrad & Oracle DR & Oracle PI \\
\midrule
200 & 0.0707 (0.0090) & 0.0248 (0.0028) & 0.0188 (0.0010) & 0.0346 (0.0018) \\
400 & 0.0316 (0.0108) & 0.0172 (0.0040) & 0.0133 (0.0007) & 0.0216 (0.0011) \\
800 & 0.0127 (0.0011) & 0.0100 (0.0005) & 0.0097 (0.0005) & 0.0165 (0.0008) \\
1600 & 0.0072 (0.0004) & 0.0066 (0.0003) & 0.0066 (0.0003) & 0.0110 (0.0006) \\
3200 & 0.0049 (0.0002) & 0.0048 (0.0002) & 0.0048 (0.0002) & 0.0082 (0.0004) \\
\bottomrule
\end{tabular}

\end{table}

\begin{table}[H]
\centering
\small
\setlength{\tabcolsep}{5pt}
\caption{Additional fitted $Q$-regression root-estimation diagnostics.}
\label{tab:generated-bellman-b3-root-appendix}
\begin{tabular}{llrrrr}
\toprule
$N$ & Method & Bias norm & Mean norm & Median norm & 90\% norm \\
\midrule
200 & PI & 0.0379 & 0.0500 & 0.0318 & 0.1112 \\
200 & OBiGrad & 0.0014 & 0.0203 & 0.0175 & 0.0351 \\
200 & Oracle DR & 0.0007 & 0.0164 & 0.0147 & 0.0292 \\
200 & Oracle PI & 0.0009 & 0.0305 & 0.0279 & 0.0534 \\
400 & PI & 0.0096 & 0.0204 & 0.0151 & 0.0375 \\
400 & OBiGrad & 0.0001 & 0.0131 & 0.0110 & 0.0231 \\
400 & Oracle DR & 0.0006 & 0.0117 & 0.0107 & 0.0201 \\
400 & Oracle PI & 0.0009 & 0.0190 & 0.0172 & 0.0349 \\
800 & PI & 0.0024 & 0.0106 & 0.0095 & 0.0180 \\
800 & OBiGrad & 0.0007 & 0.0088 & 0.0084 & 0.0150 \\
800 & Oracle DR & 0.0006 & 0.0086 & 0.0081 & 0.0146 \\
800 & Oracle PI & 0.0003 & 0.0144 & 0.0134 & 0.0249 \\
1600 & PI & 0.0003 & 0.0063 & 0.0059 & 0.0111 \\
1600 & OBiGrad & 0.0003 & 0.0059 & 0.0056 & 0.0101 \\
1600 & Oracle DR & 0.0003 & 0.0058 & 0.0055 & 0.0100 \\
1600 & Oracle PI & 0.0004 & 0.0096 & 0.0087 & 0.0171 \\
3200 & PI & 0.0004 & 0.0043 & 0.0041 & 0.0075 \\
3200 & OBiGrad & 0.0003 & 0.0042 & 0.0039 & 0.0074 \\
3200 & Oracle DR & 0.0003 & 0.0042 & 0.0040 & 0.0073 \\
3200 & Oracle PI & 0.0005 & 0.0073 & 0.0070 & 0.0125 \\
\bottomrule
\end{tabular}

\end{table}

\begin{figure}[H]
    \centering
    \includegraphics[width=.68\linewidth]{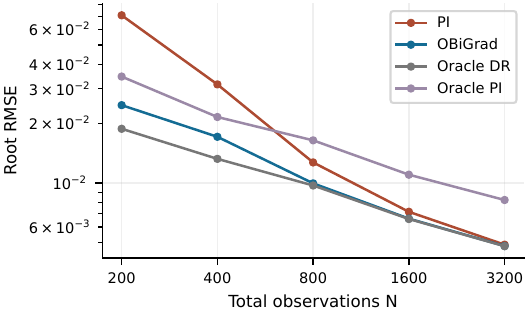}
    \caption{Fitted $Q$-regression root-estimation RMSE. OBiGrad tracks the oracle DR benchmark and improves over PI.}
    \label{fig:app-bellman-root-rmse}
\end{figure}

\begin{figure}[H]
    \centering
    \includegraphics[width=.68\linewidth]{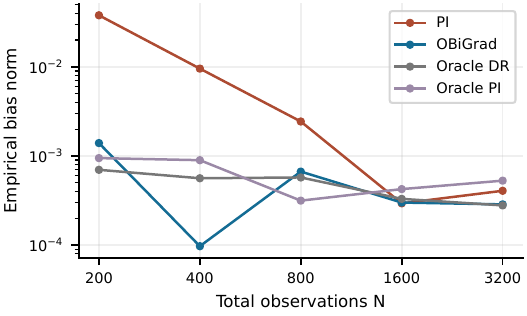}
    \caption{Fitted $Q$-regression root-estimation bias.}
    \label{fig:app-bellman-root-bias}
\end{figure}

\paragraph{Fitted $Q$-regression Experiment 4: KBO regularization bias.}
This experiment uses $N=1200$ total observations and $200$ replications. KBO uses a fixed $256$-feature RFF approximation to $k(x,x')=\exp(-0.35\|x-x'\|^2)$, with the same nine penalties as IV Experiment 3. The number $0.35$ is the kernel exponent coefficient, not the Gaussian bandwidth. Tables~\ref{tab:generated-bellman-b4-kbo}--\ref{tab:generated-bellman-b4-appendix} use matched two-fold evaluation for all methods and population targets obtained by quadrature for the saved feature map. The one-direction protocol is shown separately in Figure~\ref{fig:app-bellman-kbo-total} and Table~\ref{tab:bellman-native-kbo}. Orthogonalization reduces population bias and improves total error at many penalties, but comparisons at the smallest penalties are close. At $\lambda=10^{-4}$, matched KBO RMSE is $0.01353$ and existing basis OBiGrad RMSE is $0.01317$; in the independent seed/RFF suite they are $0.01359$ and $0.01364$. Thus a claim that the basis estimator dominates KBO at every penalty is not supported. The primary ablation compares each learner with its own orthogonal counterpart.

\begin{table}[H]
\centering
\scriptsize
\setlength{\tabcolsep}{3pt}
\caption{Matched two-fold fitted $Q$-regression gradient RMSE at $N=1200$. Parentheses report Monte Carlo $95\%$ half-widths.}
\label{tab:generated-bellman-b4-kbo}
\begin{tabular}{lrrrrr}
\toprule
$\lambda$ & PI & OBiGrad & KBO & KBO+OBiGrad & Oracle \\
\midrule
1e-05 & 0.0148 (0.0011) & 0.0132 (0.0008) & 0.0135 (0.0008) & 0.0133 (0.0009) & 0.0131 (0.0009) \\
3e-05 & 0.0148 (0.0011) & 0.0132 (0.0008) & 0.0134 (0.0008) & 0.0132 (0.0009) & 0.0131 (0.0009) \\
0.0001 & 0.0148 (0.0011) & 0.0132 (0.0008) & 0.0135 (0.0008) & 0.0131 (0.0008) & 0.0131 (0.0009) \\
0.0003 & 0.0148 (0.0011) & 0.0132 (0.0008) & 0.0142 (0.0009) & 0.0130 (0.0008) & 0.0131 (0.0009) \\
0.001 & 0.0148 (0.0011) & 0.0132 (0.0008) & 0.0187 (0.0010) & 0.0130 (0.0008) & 0.0131 (0.0009) \\
0.003 & 0.0148 (0.0011) & 0.0132 (0.0008) & 0.0345 (0.0011) & 0.0130 (0.0008) & 0.0131 (0.0009) \\
0.01 & 0.0148 (0.0011) & 0.0132 (0.0008) & 0.0811 (0.0011) & 0.0133 (0.0009) & 0.0131 (0.0009) \\
0.03 & 0.0148 (0.0011) & 0.0132 (0.0008) & 0.1681 (0.0011) & 0.0149 (0.0009) & 0.0131 (0.0009) \\
0.1 & 0.0148 (0.0011) & 0.0132 (0.0008) & 0.3017 (0.0013) & 0.0246 (0.0009) & 0.0131 (0.0009) \\
\bottomrule
\end{tabular}

\end{table}

\begin{table}[H]
\centering
\small
\setlength{\tabcolsep}{5pt}
\caption{Fitted $Q$-regression kernel population-bias norms and estimation RMSE around each method's own population target.}
\label{tab:generated-bellman-b4-appendix}
\begin{tabular}{lrrrr}
\toprule
$\lambda$ & KBO bias & KBO estimation & OBi bias & OBi estimation \\
\midrule
1e-05 & 0.0002 & 0.0135 (0.0008) & 2e-05 & 0.0133 (0.0009) \\
3e-05 & 0.0006 & 0.0134 (0.0008) & 0.0001 & 0.0132 (0.0009) \\
0.0001 & 0.0016 & 0.0134 (0.0008) & 0.0001 & 0.0131 (0.0008) \\
0.0003 & 0.0043 & 0.0133 (0.0008) & 0.0003 & 0.0130 (0.0008) \\
0.001 & 0.0126 & 0.0130 (0.0008) & 0.0007 & 0.0130 (0.0008) \\
0.003 & 0.0314 & 0.0125 (0.0008) & 0.0016 & 0.0129 (0.0008) \\
0.01 & 0.0795 & 0.0116 (0.0007) & 0.0041 & 0.0127 (0.0008) \\
0.03 & 0.1673 & 0.0111 (0.0007) & 0.0087 & 0.0121 (0.0008) \\
0.1 & 0.3015 & 0.0132 (0.0011) & 0.0219 & 0.0108 (0.0007) \\
\bottomrule
\end{tabular}

\end{table}

\begin{figure}[H]
    \centering
    \includegraphics[width=\linewidth]{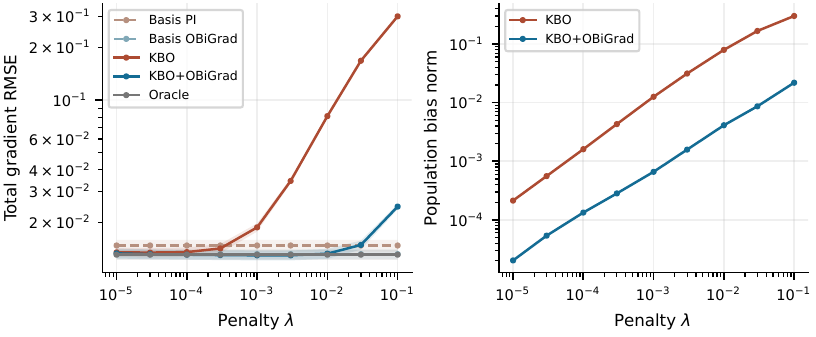}
    \caption{Fitted $Q$-regression kernel diagnostics: matched total gradient RMSE (left) and population-bias norms (right).}
    \label{fig:app-bellman-kbo}
\end{figure}

\begin{figure}[H]
    \centering
    \includegraphics[width=\linewidth]{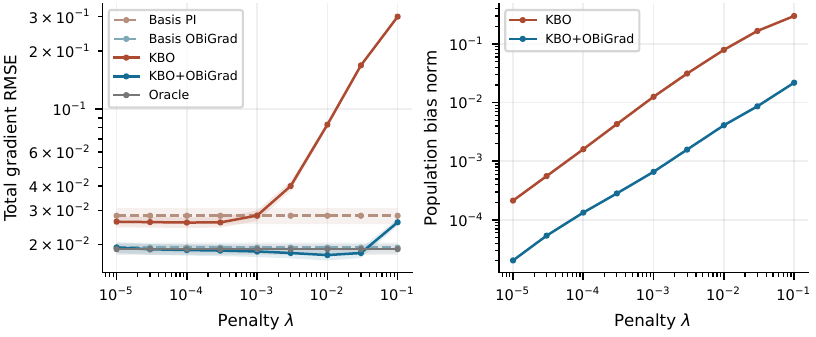}
    \caption{Fitted $Q$-regression under the original one-direction split: $600$ training and $600$ evaluation observations, shared by all methods.}
    \label{fig:app-bellman-kbo-total}
\end{figure}
\begin{table}[H]\centering\scriptsize\setlength{\tabcolsep}{3pt}
\caption{Fitted $Q$-regression: original one-direction evaluation, with matched observations for all five estimators. Parentheses give Monte Carlo $95\%$ half-widths.}
\label{tab:bellman-native-kbo}\begin{tabular}{lrrrrr}
\toprule
$\lambda$ & PI & OBiGrad & KBO & KBO+OBiGrad & Oracle \\
\midrule
1e-05 & 0.0281 (0.0027) & 0.0192 (0.0013) & 0.0262 (0.0017) & 0.0192 (0.0012) & 0.0189 (0.0012) \\
3e-05 & 0.0281 (0.0027) & 0.0192 (0.0013) & 0.0261 (0.0017) & 0.0189 (0.0012) & 0.0189 (0.0012) \\
0.0001 & 0.0281 (0.0027) & 0.0192 (0.0013) & 0.0259 (0.0016) & 0.0187 (0.0012) & 0.0189 (0.0012) \\
0.0003 & 0.0281 (0.0027) & 0.0192 (0.0013) & 0.0260 (0.0016) & 0.0186 (0.0012) & 0.0189 (0.0012) \\
0.001 & 0.0281 (0.0027) & 0.0192 (0.0013) & 0.0282 (0.0017) & 0.0184 (0.0012) & 0.0189 (0.0012) \\
0.003 & 0.0281 (0.0027) & 0.0192 (0.0013) & 0.0400 (0.0022) & 0.0180 (0.0011) & 0.0189 (0.0012) \\
0.01 & 0.0281 (0.0027) & 0.0192 (0.0013) & 0.0831 (0.0024) & 0.0176 (0.0011) & 0.0189 (0.0012) \\
0.03 & 0.0281 (0.0027) & 0.0192 (0.0013) & 0.1683 (0.0025) & 0.0180 (0.0010) & 0.0189 (0.0012) \\
0.1 & 0.0281 (0.0027) & 0.0192 (0.0013) & 0.3007 (0.0024) & 0.0260 (0.0010) & 0.0189 (0.0012) \\
\bottomrule
\end{tabular}
\end{table}
\end{document}